\documentclass{article}

\PassOptionsToPackage{numbers, compress}{natbib}
\usepackage[preprint]{neurips_2026}

\usepackage[utf8]{inputenc} 
\usepackage[T1]{fontenc}    
\usepackage{hyperref}       
\usepackage{url}            
\usepackage{booktabs}       
\usepackage{amsfonts}       
\usepackage{nicefrac}       
\usepackage{microtype}      
\usepackage{xcolor}         
\usepackage{algorithm}
\usepackage{algorithmicx}
\usepackage{algpseudocode}
\usepackage{amsmath}
\usepackage{amssymb}
\usepackage{amsthm}
\usepackage{hyperref}
\usepackage{graphicx}
\usepackage{wrapfig}
\usepackage{caption}
\usepackage[table]{xcolor}
\usepackage{etoc}
\etocdepthtag.toc{mtchapter}
\etocsettagdepth{mtchapter}{subsubsection}
\etocsettagdepth{mtappendix}{none}
\usepackage{enumitem}

\newtheorem{theorem}{Theorem}

\def\eg{\mbox{\textit{e.g.}}}

\title{Zeta: Dual Whitening for Matrix Optimization via Coordinate-Adaptive Preconditioning}

\author{
    Kaiwen Chen\textsuperscript{\rm 1}\thanks{Equal contribution.Email: kchen584272225@gmail.com, shuhaizhangshz@gmail.com}~~
    Shuhai Zhang\textsuperscript{\rm 1}\footnotemark[1]~~
    Zimo Liu\textsuperscript{\rm 2}~~
    Linxiao Li\textsuperscript{\rm 2}~~
    \textbf{Ying Sun}\textsuperscript{\rm 2}\\
    \textbf{Yuchen Li}\textsuperscript{\rm 2}~~
    \textbf{Yifan Zhang}\textsuperscript{\rm 2}~~
    \textbf{Bo Han}\textsuperscript{\rm 3}~~
    \textbf{Mingkui Tan}\textsuperscript{\rm 1}\footnotemark[2]~~
    \textbf{Qiuwu Chen}\textsuperscript{\rm 2}\thanks{Corresponding author. Email: chenqiuwu@aigcode.net, mingkuitan@scut.edu.cn}\\
    \textsuperscript{\rm 1} South China University of Technology,
    \textsuperscript{\rm 2} AIGCode,
    \textsuperscript{\rm 3} Hong Kong Baptist University
}

\begin{document}

\maketitle

\begin{abstract}

  Large-scale neural network training increasingly relies on matrix-aware optimizers that exploit the structure of weight parameters beyond element-wise adaptation. However, existing matrix-aware methods such as Muon have an underappreciated vulnerability: their core operation, Newton–Schulz iteration, depends critically on input conditioning, yet the raw momentum matrices exhibit severe coordinate-wise scale heterogeneity. In this paper, we first verify this scale heterogeneity through a chi-square uniformity test, showing that intra-matrix scale imbalance is prevalent across Transformer layers and that coordinate whitening effectively corrects it. Motivated by this finding, we propose \textit{Zeta}, a dual whitening optimizer that applies coordinate whitening and spectral whitening in a strictly ordered pipeline.   The ordering is not a tunable choice but follows from a mathematical dependency: coordinate whitening establishes the statistical isotropy that spectral whitening requires to function reliably. We further prove that this dual pipeline strictly reduces orthogonalization error relative to pure spectral methods by improving the condition number of the input. Empirically, Zeta matches or surpasses strong baselines across language modeling (0.6B to 8B parameters), mixture-of-experts architectures, and vision tasks, demonstrating that resolving scale imbalance before orthogonalization leads to faster convergence and better generalization. The source code is available at \url{https://github.com/AIGCodeOS/aigcode_zeta_optimizer}.

\end{abstract}

\section{Introduction}
\label{sec:intro}


Neural network training \cite{bottou2018optimization,goodfellow2016deep} relies heavily on optimization algorithms that determine how each parameter is updated. For example, adaptive optimizers\cite{duchi2011adaptive,shazeer2018adafactor,liu2019variance} and their variants (\eg, AdamW\cite{loshchilov2017decoupled}) operate coordinate by coordinate, using the running first and second moments of each parameter's gradient to rescale updates element-wise. This handles the wide range of gradient magnitudes\cite{glorot2010understanding} that can arise across different parts of the network \cite{zhang2020adaptive}. However, by treating weight matrices as flattened vectors, these methods ignore the spatial correlations and spectral properties that define the internal representations of deep neural networks, properties that are intrinsic to effective learning.

Several recent optimizers aim to capture the matrix structure that scalar methods discard. Muon \cite{jordan2024muon} is a leading example. It enforces orthogonal constraints on weight updates via \textit{spectral whitening}, a process that equalizes the singular values of the momentum matrix, preventing any single direction from dominating the update. This design has proven effective at scale, most notably in the training of the Kimi-1.1T model \cite{team2025kimi}. By operating directly on the matrix rather than on flattened entries, Muon preserves the geometry of the update that scalar methods cannot, leading to more stable training.

Despite these strengths, Muon exhibits an underappreciated vulnerability. Its spectral operation, the Newton–Schulz iteration \cite{kovarik1970some,bjorck1971iterative,higham2008functions}, is a projection onto the Stiefel manifold \cite{zhang2026mousse, gonon2026insights}, whose accuracy is sensitive to input conditioning \cite{nakatsukasa2012backward, he2025root}. The iteration works well when the singular values are well-distributed. But raw momentum matrices in deep network training are often dominated by a few coordinates with disproportionately large magnitudes. When such a matrix enters the Newton–Schulz step, the large entries skew the singular spectrum and orthogonalization degrades \cite{cheng2026trasmuon}. The update direction no longer faithfully reflects the gradient geometry. In short, Muon applies a powerful geometric operation to a signal with uncontrolled coordinate-wise scales.

This reveals a deeper mismatch, not a matter of tuning. Spectral whitening demands a scale-normalized input, but the raw momentum offers no such guarantee. The missing ingredient is a step that homogenizes the coordinate-wise scales before entering the Newton–Schulz iteration. To verify that this scale heterogeneity is both real and correctable, we apply a chi-square uniformity test \cite{pearson1900x} to the row energies of representative weight matrices in a Transformer. The null hypothesis is that all blocks of the matrix share the same scale. As in Figure~\ref{fig:motivation_pvalue}, the $p$-values before any preprocessing are consistently small across layers, confirming that \textit{significant intra-matrix scale imbalance} exists. After coordinate whitening, where each entry is normalized by its running second moment, the $p$-values shift markedly toward 1, indicating that the imbalance has been largely removed.

\begin{figure*}[t]
\vspace{-15pt}
    \centering
    \includegraphics[width=0.98\textwidth]{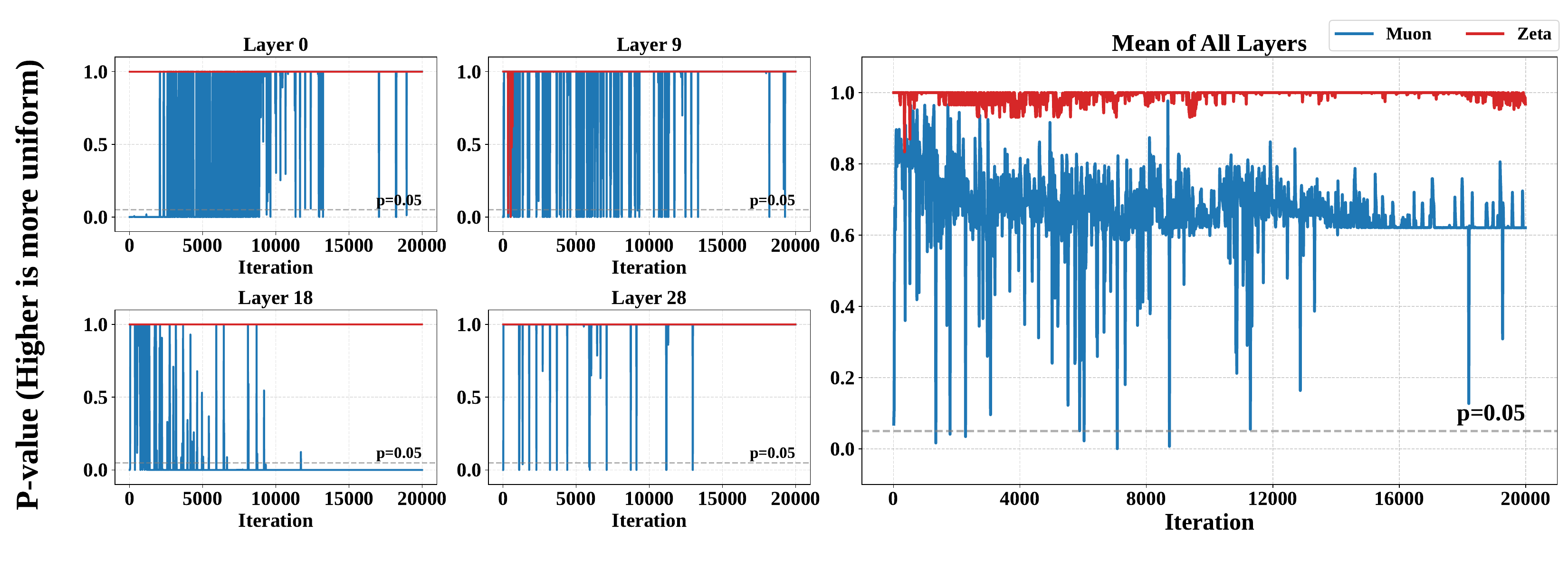}
    \vspace{-10pt}
    \caption{
    Illustration of intra-matrix scale heterogeneity in Muon. We apply a chi-square uniformity test to the momentum matrices of a Qwen3-0.6B model trained with Muon, where the null hypothesis is that all blocks share the same scale and $p$-values closer to 1 indicate stronger uniformity. The left panel shows representative layers from shallow, middle, and deep parts of the network, while the right panel reports the mean across all layers. The $p$-values before whitening are consistently small, confirming that the raw momentum matrices exhibit significant scale imbalance. After coordinate whitening, the $p$-values shift markedly toward 1, indicating that the imbalance can be corrected.
    }
    \label{fig:motivation_pvalue}
    \vspace{-1.0em}
\end{figure*}


Motivated by this finding, we propose \textbf{Zeta}, a \textit{dual whitening} optimizer. The core insight is straightforward: spectral whitening must be built upon statistical stabilization. Zeta implements this as a sequential pipeline of two operations. The first, \textbf{coordinate whitening}, normalizes each entry of the momentum by its running second moment, reducing scale disparity across coordinates. The second, \textbf{spectral whitening}, applies Newton–Schulz orthogonalization to the now well-conditioned matrix, equalizing its singular values to produce a balanced update direction. These two operations are complementary halves of a single requirement: coordinate whitening establishes the statistical isotropy that spectral whitening depends on, so the ordering is not a free choice. A rotation is only meaningful if the input has been freed from extreme scale distortion beforehand.

We further support this design with theoretical analysis. Specifically, we show that coordinate whitening acts as an approximate diagonal preconditioner under standard assumptions (Theorem \ref{thm:coord_whitening}), giving the first stage a principled interpretation rather than a heuristic step. We then prove that this preconditioning improves the condition number of the input, which reduces the orthogonalization error of the Newton–Schulz iteration relative to operating on raw momentum (Theorem \ref{thm:ns_orthogonality}). Finally, we confirm that the RMS scaling preserves the spectral direction of the update (Theorem \ref{thm:rms_scaling}), so the pipeline produces a well-conditioned update without distorting its geometry.

Empirically, we evaluate Zeta on language model pretraining across scales from 0.6B to 8B parameters, on mixture-of-experts architectures, and on vision tasks. Across these settings, Zeta consistently matches or surpasses AdamW \cite{loshchilov2017decoupled}, Muon \cite{jordan2024muon}, and AdaMuon \cite{si2025adamuon} in both convergence speed and downstream performance. These results show that resolving scale imbalance before orthogonalization leads to faster convergence and better generalization, without sacrificing the matrix-level geometric structure that makes spectral whitening valuable.
Our contributions are summarized as: 

\begin{itemize}[leftmargin=*]
\item We identify and empirically verify that spectral whitening methods like Muon implicitly require scale-normalized input, a condition that raw momentum matrices systematically violate. Through a chi-square uniformity test, we demonstrate that significant intra-matrix scale heterogeneity exists across transformer layers and that coordinate whitening effectively corrects it, revealing a previously overlooked prerequisite for stable spectral orthogonalization.
\item We propose Zeta, a dual whitening optimizer, which organizes coordinate whitening and spectral whitening into a strictly ordered pipeline. The ordering is not a tunable choice but follows from a mathematical dependency: coordinate whitening establishes the statistical isotropy that spectral whitening requires to function reliably. This design unifies what have been treated as competing optimization paradigms into complementary stages of a single requirement.
\item  We prove that the dual whitening pipeline strictly reduces orthogonalization error relative to pure spectral methods, by showing that the coordinate whitening stage improves the condition number of the input to the Newton–Schulz iteration. We evaluate Zeta on language modeling (0.6B to 8B parameters), mixture-of-experts architectures, and vision tasks, showing consistent improvements in convergence speed and downstream performance over AdamW, Muon, and AdaMuon.
\end{itemize}


\section{Related Work}
\label{sec:related_work}

Optimization algorithms in deep learning primarily follow two trajectories: coordinate-wise adaptive scaling and methods that exploit the structural correlations within parameter matrices.

\textbf{First-order Adaptive Optimizers}.
First-order methods remain central to large-scale training because of their efficiency. Beyond classical Stochastic Gradient Descent (SGD\cite{bottou2010large}), adaptive methods such as \textbf{Adam}\cite{kingma2014adam} and \textbf{AdamW}\cite{loshchilov2017decoupled} are standard for Transformers, using second-moment estimates to scale coordinates independently. \textbf{Lion}\cite{chen2023symbolic} offers a memory-efficient alternative based on the sign of momentum, while \textbf{Sophia}\cite{liu2023sophia} incorporates diagonal Hessian estimates to capture loss curvature. Despite their effectiveness for coordinate-wise variance, these methods treat parameters as flattened vectors and thus ignore the spatial and spectral structure of weight matrices.

\textbf{Matrix-structured and Second-order Optimizers}.
To capture the correlations between parameters, another line of research focuses on matrix-structured optimization. \textbf{Shampoo}\cite{gupta2018shampoo} and its variants like \textbf{SOAP}\cite{vyas2024soap} employ Kronecker-factored preconditioners to approximate the inverse of the full second-moment matrix, effectively whitening the gradient across rows and columns. However, these methods often involve expensive matrix square root or inversion operations. \textbf{Muon}\cite{jordan2024muon} provides a more efficient alternative by using Newton-Schulz iteration for spectral whitening, which forces the update direction to be an orthogonal matrix. Following this, hybrid approaches like \textbf{AdaMuon}\cite{si2025adamuon} attempt to combine adaptive scaling with spectral methods. In contrast, our Zeta positions itself at the intersection of these two trajectories. By coupling coordinate whitening with spectral orthogonalization in a unified hierarchical pipeline, Zeta addresses the gradient scale imbalance that remains a challenge for pure spectral methods, while preserving the structural benefits of matrix-based optimization.

\section{Proposed Method}
\label{sec:proposed}

The Muon optimizer \cite{jordan2024muon} enforces orthogonal constraints on weight updates via spectral whitening, a matrix-aware design that has proven highly effective at scale \cite{team2025kimi}. Despite this success, Muon suffers from a critical vulnerability. The Newton–Schulz iteration \cite{kovarik1970some,bjorck1971iterative,higham2008functions} it relies on is a projection onto the Stiefel manifold \cite{zhang2026mousse, gonon2026insights}, whose accuracy is sensitive to input conditioning \cite{nakatsukasa2012backward,he2025root}. When the momentum matrix is dominated by a few coordinates with large magnitudes, its singular spectrum becomes skewed and orthogonalization degrades \cite{cheng2026trasmuon}, producing distorted updates. In essence, Muon applies a powerful geometric operation to an insufficiently conditioned signal, without a mechanism to stabilize the input beforehand. This is not an implementation flaw but a structural gap: spectral whitening requires a well-conditioned input, yet the raw momentum provides no such guarantee.

To address the above issue, we propose \textbf{Zeta}, a \textit{dual} whitening optimizer. Our core insight is that spectral whitening should be built upon statistical stabilization: the coordinate-wise scales of the momentum matrix must be homogenized before it enters the Newton–Schulz iteration. This principle leads to a design where two complementary operations are applied sequentially: \textit{coordinate whitening}, which normalizes each entry to reduce its scale disparity across coordinates; \textit{spectral whitening}, which applies Newton–Schulz orthogonalization to equalize the singular values of the update. These two operations are not competing alternatives but sequential necessities. One operates in the coordinate basis to establish statistical isotropy, the other in the spectral basis to enforce geometric structure.

Crucially, Zeta is not a heuristic combination of existing techniques but is based on the dual whitening principle. We highlight three key aspects. 1) \textbf{The two operations are complementary}. Coordinate whitening reduces entry-wise scale disparity (Theorem \ref{thm:coord_whitening}); spectral whitening corrects directional imbalance among singular values. Neither alone covers both, and together they form a complete normalization of the update.
2) \textbf{The sequential ordering is a logical necessity, not a tunable choice}. Spectral whitening performs a rotation, whose fidelity depends on the input being scale-normalized. Reversing the sequence would feed an unnormalized matrix into the Newton–Schulz step, returning the instability Zeta removes. This ordering is provably necessary: Theorem \ref{thm:ns_orthogonality} shows that coordinate whitening strictly reduces the orthogonalization error of the spectral stage.
3) \textbf{The pipeline is unified, not modular}. Coordinate whitening does not merely preprocess input for spectral whitening. It provides the scale-normalized input that spectral whitening requires to function reliably. The two operations are not independent modules; they are two halves of one preconditioning requirement.


\subsection{Preliminaries and Motivations}
\label{subsec:muon_baseline}

Modern adaptive optimizers such as AdamW \cite{loshchilov2017decoupled} scale each parameter independently using its running second moment. This element-wise operation treats weight matrices as flat vectors and ignores their matrix structure. Muon \cite{jordan2024muon} departs from this paradigm by directly exploiting the fact that parameters like attention and feed-forward projections are matrices (Algorithm \ref{alg:muon}). Its core operation, spectral whitening, computes the orthogonal matrix closest to the current momentum $G_t \in \mathbb{R}^{m \times n}$:
\begin{equation}
    \text{Ortho}(G_t) = \arg\min_{O} \|O - G_t\|_F \quad \text{s.t. } O^\top O = I \text{ or } O O^\top = I.
\end{equation}
Rather than computing this projection via expensive SVD \cite{bernstein2024old}, Muon approximates it via the Newton–Schulz iteration \cite{kovarik1970some,bjorck1971iterative,higham2008functions}. Starting from a normalized matrix $X_0 {=} G_t / \|G_t\|_F$, it iterates by
\begin{equation}
    X_{k+1} = aX_k + b(X_k X_k^\top)X_k + c(X_k X_k^\top)^2 X_k,
    \label{eq:newton_schulz}
\end{equation}
where $a = 3.4445$, $b = -4.7750$, $c = 2.0315$. After a small number of steps (typically $K = 5$), the output $U_t \approx X_K$ is an approximate orthogonal matrix whose singular values are close to 1. This equalization prevents any single direction from dominating the update and promotes \textit{balanced} learning across \textit{all} intrinsic dimensions of the representation.

\textbf{Implicit assumption of coordinate-scale uniformity in Muon.}
The Newton–Schulz iteration is a projection onto the Stiefel manifold \cite{zhang2026mousse, gonon2026insights}, whose accuracy depends on the conditioning of its input \cite{nakatsukasa2012backward, he2025root}. When the singular values are well-distributed, the iteration converges rapidly to a faithful orthogonal approximation. This assumption, however, is often \textit{violated} in practice. In Transformer training, gradients within a single weight matrix often exhibit vast disparities across attention heads, feed-forward channels, and individual rows. When these imbalanced gradients are accumulated into the momentum $G_t$, the resulting matrix is dominated by a few coordinates with large magnitudes, its singular spectrum becomes skewed, and the orthogonalization degrades \cite{cheng2026trasmuon}. The structural benefit of spectral whitening is thus undermined by the coordinate-scale heterogeneity it never addresses.

\textbf{Empirical evidence of scale heterogeneity.}
To verify that this scale heterogeneity is real and can be corrected, we conduct a \textit{chi-square uniformity test} \cite{pearson1900x} on the row energies of representative weight matrices. We partition the matrix into equal-sized blocks and test the null hypothesis that all blocks share the same scale. Under the null, $p$-values close to 1 indicate scale uniformity. As shown in Figure \ref{fig:motivation_pvalue}, the $p$-values before any preprocessing are consistently small across layers, confirming significant intra-matrix scale imbalance. After applying coordinate whitening (element-wise normalization by the running second moment), the $p$-values shift markedly closer to 1, showing that the matrix has been brought substantially closer to the scale-uniform regime that the Newton–Schulz iteration requires.

\textbf{From diagnosis to design principle.}
This experiment reveals a failure mode in spectral whitening and points to its remedy: the Newton–Schulz iteration is not flawed in itself, but it operates on an input that violates its implicit precondition.  The missing ingredient is a preprocessing stage that homogenizes coordinate-wise scales before spectral whitening is applied. 
This insight directly motivates Zeta, where coordinate whitening and spectral whitening are arranged in sequence, the former providing the statistical condition the latter depends on to function reliably.

\begin{figure}[t]
    \centering
    \vspace{-0.8em}
    \begin{minipage}[t]{0.46\linewidth}
    \centering
    \footnotesize
    \begin{algorithm}[H]\small
        \caption{Muon Optimizer}
        \label{alg:muon}
        \begin{algorithmic}[1]
            \Require Learning rate $\eta$, momentum $\beta$, weight decay $\lambda$, Newton-Schulz steps $K$
            \Ensure Optimized matrix weight $W$
            \State Initialize $G_0 = 0$
            \While{$t < T_{max}$}
                \State $g_t \gets \nabla_W \mathcal{L}(W_t)$
                \State $G_t \gets g_t + \beta G_{t-1}$
                \State $U_t \gets \text{Newton-Schulz}(G_t, K)$
                \State $\eta' \gets 0.2 \eta \cdot \sqrt{\max(m,n)}$
                \State $W_t \gets (1 - \eta \lambda) W_t$
                \State $W_{t+1} \gets W_t - \eta' U_t$
                \State $t \gets t + 1$
            \EndWhile
            \State \Return $W_t$
        \end{algorithmic}
    \end{algorithm}
\end{minipage}\hspace{0.02\linewidth}
\begin{minipage}[t]{0.5\linewidth}
    \centering
    \footnotesize
    \begin{algorithm}[H]\small
        \caption{Zeta Optimizer}
        \label{alg:Zeta}
        \begin{algorithmic}[1]
            \Require Learning rate $\eta$, coefficients $\beta_1, \beta_2$, weight decay $\lambda$, Newton-Schulz steps $K$, $\varepsilon$
            \Ensure Optimized weight matrix $W$
            \State Initialize $M_0 \gets 0, V_0 \gets 0$
            \While{$t < T_{max}$}
                \State $g_t \gets \nabla_W \mathcal{L}(W_t)$
                \State $M_t \gets \beta_1 M_{t-1} + (1 - \beta_1) g_t$
                \State $V_t \gets \beta_2 V_{t-1} + (1 - \beta_2) g_t^2$
                \State $\tilde{G}_t \gets M_t / (\sqrt{V_t} + \varepsilon)$
                \State $U_t \gets \text{Newton-Schulz}(\tilde{G}_t, K)$
                \State $\Delta W_t \gets (0.2 \cdot \sqrt{m\times n}) / (\|U_t\|_F {+} \varepsilon) \cdot U_t$
                \State $W_{t+1} \gets (1 - \eta\lambda) W_t - \eta \cdot \Delta W_t$
                \State $t \gets t + 1$
            \EndWhile
            \State \Return $W_t$
        \end{algorithmic}
    \end{algorithm}
\end{minipage}
\vspace{-1em}
\end{figure}

\subsection{Zeta: Dual Whitening for Stable Orthogonal Updates}
\label{subsec:Zeta}

The principle established in Section \ref{subsec:muon_baseline} identifies a gap: spectral whitening demands a scale-normalized input that the raw momentum does not provide. We propose \textbf{Zeta} to close this gap with a preprocessing step that homogenizes coordinate-wise scales before entering Newton–Schulz iterations. Its core principle is that spectral whitening must be built upon statistical stabilization. Rather than feeding raw momentum into the Newton–Schulz iteration, Zeta first applies coordinate whitening to produce a statistically isotropic matrix, then performs spectral whitening on the result.

\textbf{Coordinate whitening.}
Given a matrix-parameter gradient $g_t \in \mathbb{R}^{m \times n}$ at step $t$, Zeta maintains two exponential moving averages, mirroring the standard moments in Adam \cite{loshchilov2017decoupled}:
\begin{equation}
    M_t = \beta_1 M_{t-1} + (1 - \beta_1) g_t, \qquad
    V_t = \beta_2 V_{t-1} + (1 - \beta_2) g_t^2,
\end{equation}
where all operations are element-wise. Instead of using $M_t$ directly, Zeta normalizes each entry by its running second moment:
\begin{equation}
    \tilde{G}_t = \frac{M_t}{\sqrt{V_t} + \varepsilon}.
    \label{eq:coord_whitening}
\end{equation}
Eqn. (\ref{eq:coord_whitening}) is the coordinate whitening step. It reduces the scale disparity across entries, bringing the matrix closer to the scale-uniform regime that spectral operators require. 
Theorem \ref{thm:coord_whitening} formalizes this operation as an approximate diagonal preconditioner: under a diagonal-covariance assumption, the element-wise second-moment estimator yields an approximation to the ideal whitening matrix, making this step a principled normalization rather than an ad-hoc heuristic.

\textbf{Spectral whitening.}
The homogenized matrix $\tilde{G}_t$ then enters the Newton–Schulz iteration (Eqn. (\ref{eq:newton_schulz})), which produces an orthogonal update direction:
\begin{equation}
    U_t = \text{Newton–Schulz}(\tilde{G}_t).
    \label{eq:spectral_whitening}
\end{equation}
Because $\tilde{G}_t$ is free of extreme entry-wise scale disparities, its singular spectrum is better conditioned than that of the raw momentum, and the Newton–Schulz iteration converges to a faithful orthogonal approximation. Theorem \ref{thm:ns_orthogonality} provides the formal guarantee: coordinate whitening reduces the condition number of the input, which decreases the orthogonalization error of the spectral stage. The two stages are thus complementary—the first reduces spectral anisotropy, the second exploits the improved spectrum to produce an update closer to the ideal orthogonal factor under the same iteration budget.

\textbf{Update construction.}
Zeta applies an RMS-normalized scaling to the orthogonal direction:
\begin{equation}
    \Delta W_t = \frac{0.2 \cdot \sqrt{\max(m, n)}}{\|U_t\|_F + \varepsilon} \cdot U_t.
    \label{eq:rms_scaling}
\end{equation}
Theorem~\ref{thm:rms_scaling} shows that this scaling preserves the spectral direction of the orthogonalized update while adjusting its global magnitude to a level compatible with standard optimizer configurations, synchronizing the update scale between the matrix branch and the rest of the network.

\textbf{Dual-path execution.}
For matrix-valued parameters (\eg, attention and feed-forward projections), Zeta applies the full coordinate-then-spectral pipeline described above. For non-matrix parameters such as biases and LayerNorm scales, the spectral whitening step is skipped, and Zeta falls back to an AdamW-style update. This dual-path strategy ensures that matrix parameters benefit from the dual whitening procedure while scalar parameters incur no unnecessary overhead.

\subsection{Theoretical Guarantees for Zeta}
\label{sec:theory}

In this section, we provide a concise theoretical interpretation of why dual whitening improves matrix optimization. We focus on three observations: first, the coordinate normalization induced by the second-moment estimator can be interpreted as an approximate whitening operation; second, this preprocessing improves the quality of the subsequent Newton-Schulz (NS) orthogonalization and makes the resulting update matrix closer to an ideal orthogonal matrix; third, the final RMS calibration enforces a controlled update scale while preserving the orthogonalized direction.

\begin{theorem}[Relation between second-moment and whitening]\label{thm:coord_whitening}
Let $M_t \in \mathbb{R}^{m \times n}$ be the momentum matrix at iteration $t$, with vectorized form $m_t = \mathrm{vec}(M_t)$. Assume its covariance is approximately diagonal, $\mathrm{Cov}(m_t) \approx \Sigma_t = \mathrm{diag}(\sigma_{t,1}^2, \dots, \sigma_{t,mn}^2)$. The ideal whitening transform is
\begin{equation}
    \hat{m}_t = \Sigma_t^{-1/2} m_t
\end{equation}
In Zeta, we instead use the second-moment estimator $V_t = \beta_2 V_{t-1} + (1-\beta_2) G_t^2$ and apply the coordinate whitening step in Eq.~\eqref{eq:coord_whitening}. If $V_t$ consistently estimates the diagonal second-order moment, then $\mathrm{vec}(\tilde{G}_t) \approx \mathrm{diag}(V_t)^{-1/2} \, \mathrm{vec}(M_t)$ is a diagonal approximation to $\Sigma_t^{-1/2} m_t$. Hence, Zeta implements approximate coordinate whitening.
\end{theorem}

Theorem~\ref{thm:coord_whitening} shows that Zeta replaces exact whitening, which would require the inverse square root of the full covariance matrix, with an efficient diagonal approximation based on the element-wise second moment. Thus, the coordinate-normalization stage is a principled approximation to whitening rather than an Adam-style heuristic. It also differs from sign-based preprocessing: second-moment normalization remains tied to a stable second-order statistic, whereas a sign operator mainly discards magnitude information; see Appendix~\ref{app:adam_vs_sign} for details.

\begin{theorem}[Dual Preconditioning Improves Spectral Orthogonalization]\label{thm:ns_orthogonality}
Let $G \in \mathbb{R}^{m \times n}$ be the input momentum matrix that would be fed to Newton-Schulz iteration in a pure spectral method, and let $D$ be the diagonal second-moment matrix estimated by the coordinate preconditioner. Consistent with Eq.~\eqref{eq:coord_whitening}, let $\tilde{G}$ denote the coordinate-whitened version of $G$. Let $NS_K(\cdot)$ denote the output of $K$ steps of Newton-Schulz iteration. Assume that the diagonal preconditioner reduces the condition number, i.e., $\kappa(\tilde{G}) \leq \kappa(G)$, where $\kappa(\cdot)$ is the $2$-norm condition number given by the ratio of the largest to the smallest non-zero singular value. Then, for any fixed number of iterations $K$,
\begin{equation}
    \|NS_K(\tilde{G})^\top NS_K(\tilde{G}) - I\|_F
    \leq
    \|NS_K(G)^\top NS_K(G) - I\|_F.
\end{equation}
\end{theorem}

Theorem~\ref{thm:ns_orthogonality} shows that coordinate whitening first improves the conditioning of the matrix passed to Newton-Schulz iteration, and benefits from this better-balanced input. Thus, the two stages are complementary: the first reduces spectral anisotropy, and the second uses the improved spectrum to produce an update that is closer to the ideal orthogonal factor under the same iteration budget.

\begin{theorem}[RMS-Compatible Update Scaling]\label{thm:rms_scaling}
Let $U_t \in \mathbb{R}^{m\times n}$ denote the orthogonalized matrix produced by the Newton-Schulz step. Define the scalar
\begin{equation}
    s_t = \frac{0.2}{\mathrm{RMS}(U_t)} = \frac{0.2\sqrt{mn}}{\|U_t\|_F},
\end{equation}
where the target RMS value $0.2$ follows the scaling rule used in prior Muon-style updates \cite{liu2025muon}. Then the rescaled update $s_t U_t$ preserves the direction of $U_t$ while matching its overall magnitude to the RMS level used by the Adam-style branch.
\end{theorem}

Theorem~\ref{thm:rms_scaling} establishes that, after coordinate whitening and spectral orthogonalization determine a well-conditioned update direction, the final RMS calibration aligns its global scale with the Adam-compatible branch without altering that direction. Thus, the matrix branch is not only geometrically structured but also magnitude-compatible with the rest of the optimizer, ensuring that Zeta couples directional improvement with explicit scale synchronization.

\section{Experiments}
\label{sec:experiments}



\textbf{Models and Datasets.} 
We evaluate Zeta on both dense and mixture-of-experts (MoE) language models spanning multiple scales. The dense models include GPT2-Large\cite{radford2019language} and three Qwen3-based\cite{yang2025qwen3} models with 0.6B, 1.7B, and 8B parameters. We further include a Qwen3-based MoE model with 1.3B total parameters and 0.6B activated parameters per forward pass. For GPT2-Large and Qwen3-0.6B, we use OpenWebText\footnote{\url{http://Skylion007.github.io/OpenWebTextCorpus}}as the pretraining corpus. For the larger Qwen3 models, including Qwen3-1.7B, Qwen3-8B, and Qwen3-MoE, we use the DCLM-baseline\cite{li2024datacomp} dataset, which provides high-quality filtered web text for scalable training.

\textbf{Compared Optimizers.} 
All reported pretraining results use the same four optimizers: AdamW\cite{loshchilov2017decoupled}, Muon\cite{jordan2024muon}, AdaMuon\cite{si2025adamuon}, and Zeta. We keep the optimizer set fixed across GPT2-Large, Qwen3-0.6B, Qwen3-1.7B, Qwen3-8B, and Qwen3-MoE so that differences in convergence and final performance can be attributed to the optimizer itself rather than to changes in the comparison protocol. 

\textbf{Training Configuration.} 
Unless otherwise specified, all Qwen-family models are trained with sequence length 4096, global batch size 256, cosine learning-rate decay, and a warmup ratio of 1\% of total training steps. GPT2-Large follows the same optimization schedule but uses sequence length 1024 and global batch size 480. For learning-rate selection, we tune and use the best learning rate for each optimizer on the Qwen3-0.6B model, while adopting a unified learning rate for the larger-scale experiments to keep the comparison protocol consistent: Qwen3-1.7B and Qwen3-8B use $9\times10^{-4}$, whereas GPT2-Large and Qwen3-MoE use $3\times10^{-4}$. Weight decay 0.1 is applied across all experiments. For optimizer-specific hyperparameters, AdamW uses $\beta_1 = 0.9$ and $\beta_2 = 0.95$, while Muon and AdaMuon both use $\beta = 0.95$ with Newton-Schulz iterations $K = 5$. Zeta uses $\beta_1 = 0.95$, $\beta_2 = 0.99$, and Newton-Schulz iterations $K = 5$.

\subsection{Comparisons on Model Training}
\label{subsec:main_results}

\textbf{Convergence Efficiency on Dense Models.} 
Figures \ref{fig:dense_results_06b_group}, \ref{fig:main_1.7b_dashboard}, and \ref{fig:main_8b_dashboard} present the training loss curves and performance speedup for the 0.6B, gpt2-large, 1.7B, and 8B dense models. Across all scales, Zeta consistently reduces loss faster than both Muon and AdamW. 
While Muon benefits from spectral whitening, its convergence is slower in the early phase, when the momentum matrix is most imbalanced and the Newton–Schulz iteration faces a poorly conditioned input. Zeta avoids this difficulty by applying coordinate whitening first, yielding a sharper and more sustained loss decrease.

Using AdamW as the baseline (1.00$\times$ speedup) to reach the same loss level, Zeta achieves a \textbf{1.67$\times$} speedup on the \textbf{1.7B model}, outperforming Muon (1.52$\times$) and AdaMuon (1.61$\times$), and a \textbf{1.25$\times$} speedup on the \textbf{8B model}, compared with 1.19$\times$ for Muon.
These gains do not incur a large runtime penalty. As shown in Figures \ref{fig:main_1.7b_dashboard} and \ref{fig:main_8b_dashboard}, Zeta's time per iteration remains close to AdamW (25.47s vs. 24.46s for the 8B model on 16 GPUs), suggesting that dual whitening achieves faster convergence at nearly the same per-step cost as standard adaptive methods.

\begin{figure}[t]
    \centering
    \begin{minipage}{0.48\textwidth}
        \centering
        \includegraphics[width=\textwidth]{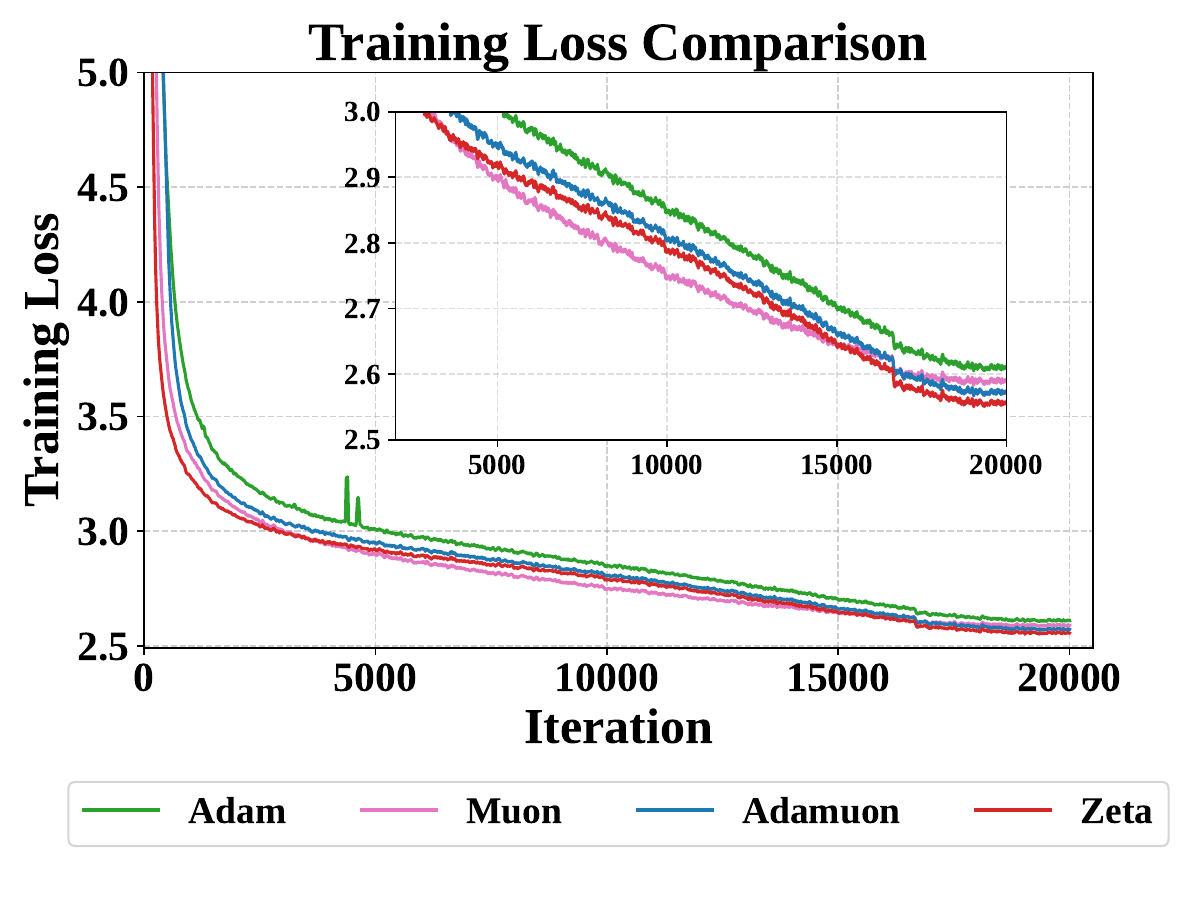}
    \end{minipage}
    \hfill 
    \begin{minipage}{0.48\textwidth}
        \centering
        \includegraphics[width=\textwidth]{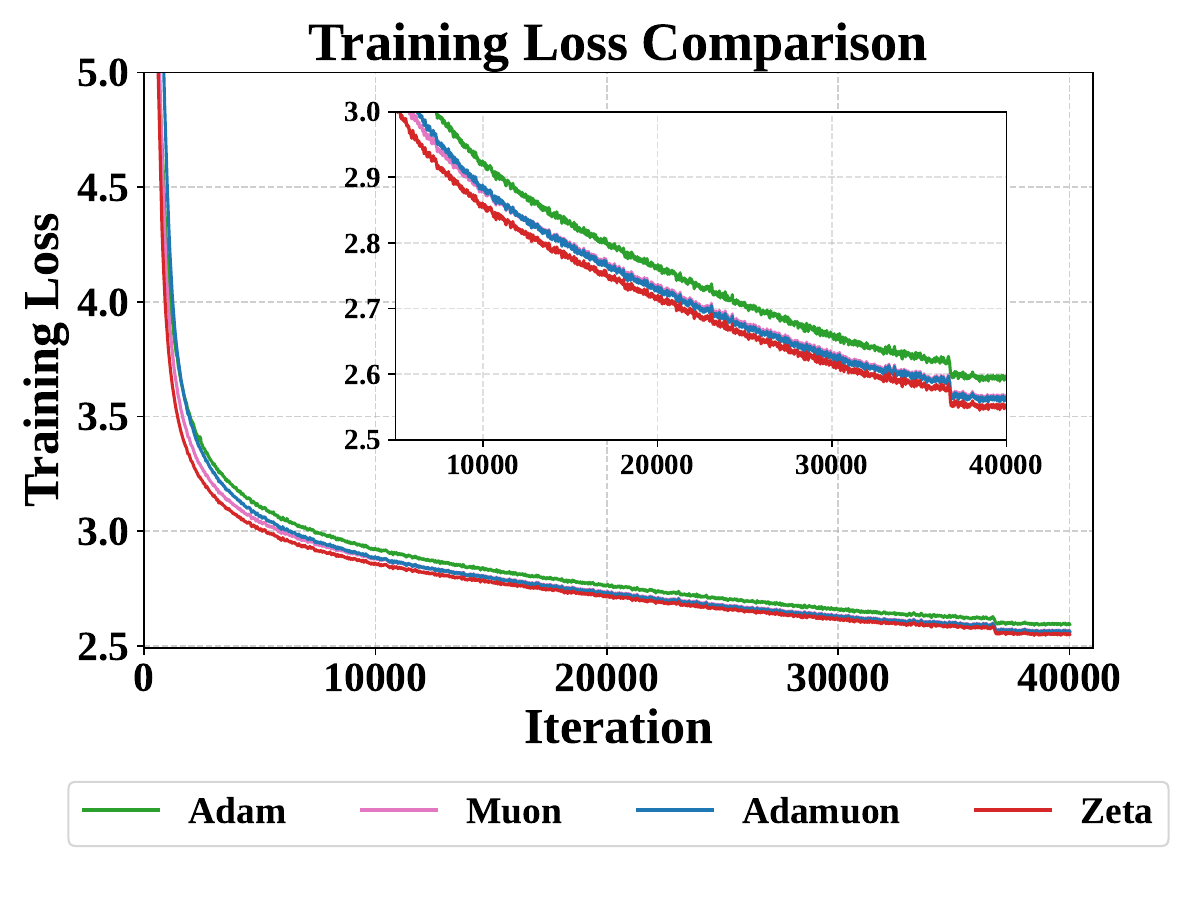} 
    \end{minipage}
    
    \caption{Training loss curves for Qwen3-0.6B (left) and GPT-2 Large (right).}
    \label{fig:dense_results_06b_group}
    \vspace{-1em}
\end{figure}

\begin{figure}[t]
\vspace{-15pt}
    \centering
    \begin{minipage}{0.55\textwidth}
        \centering
        \includegraphics[width=\textwidth]{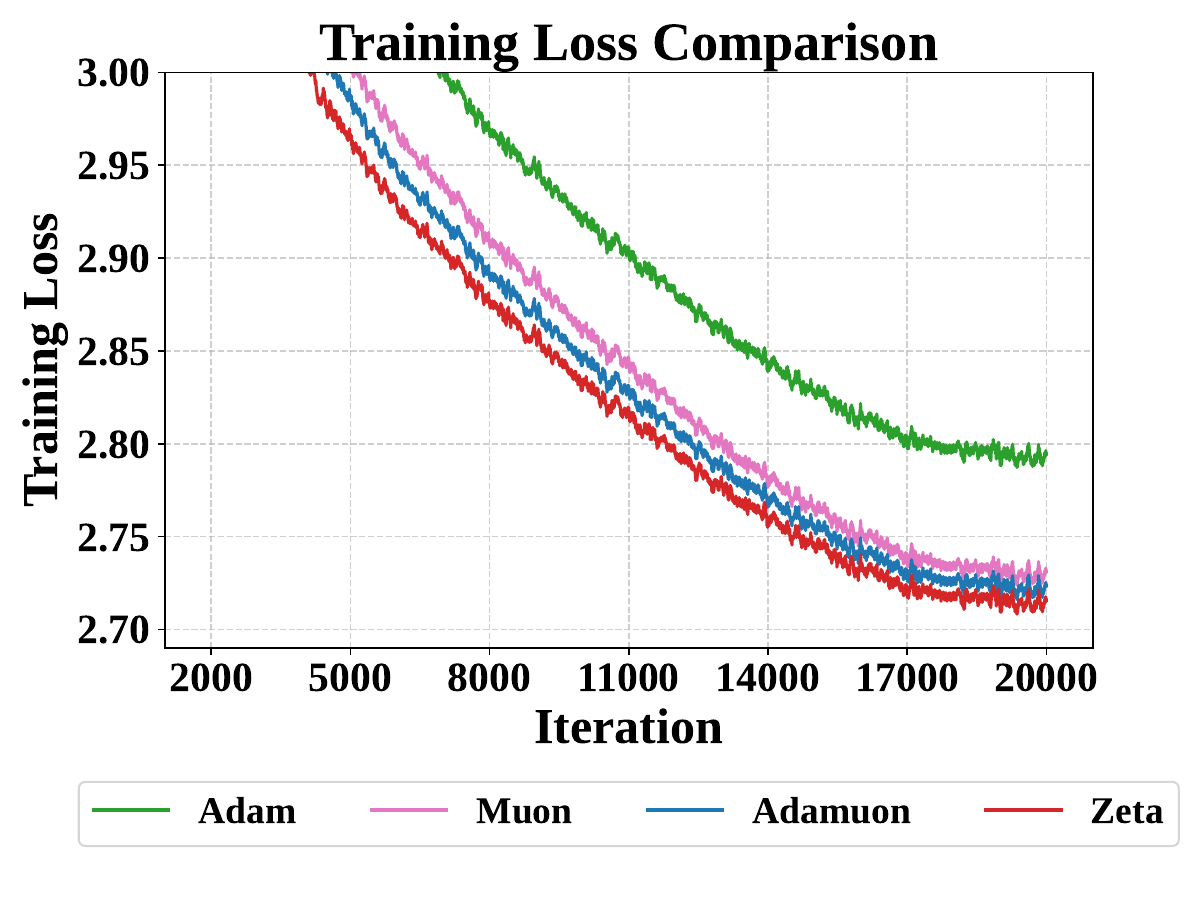}
    \end{minipage}
    \hfill 
    \begin{minipage}{0.42\textwidth}
        \small 
        \raggedright 
        
        \textbf{Speedup} (base on iteration, AdamW: 1.00x) \\
        \rule{\linewidth}{0.5pt} \\ 
        Muon: \hfill 1.52x \\
        AdaMuon: \hfill 1.61x \\
        Zeta: \hfill \textbf{1.67x} \\
        \vspace{0.6em}

        \textbf{Speedup} (base on time, AdamW: 1.00x) \\
        \rule{\linewidth}{0.5pt} \\ 
        Muon: \hfill 1.52x \\
        AdaMuon: \hfill 1.60x \\
        Zeta: \hfill \textbf{1.64x} \\
        \vspace{0.6em}
        
        \textbf{Time per Iteration} (wall-clock on 8 GPU) \\
        \rule{\linewidth}{0.5pt} \\
        AdamW: \hfill 15.3s \\
        Muon: \hfill 15.35s \\
        AdaMuon: \hfill 15.44s \\
        Zeta: \hfill 15.59s \\
        \vspace{0.6em}
        
    \end{minipage}

    \vspace{-5pt}
    \caption{Training loss curves for Qwen3-1.7B (left) and speedup metrics (right).}
    \label{fig:main_1.7b_dashboard}
    \vspace{-0.2em}
\end{figure}

\begin{figure}[t]
    \centering
    \begin{minipage}{0.55\textwidth}
        \centering
        \includegraphics[width=\textwidth]{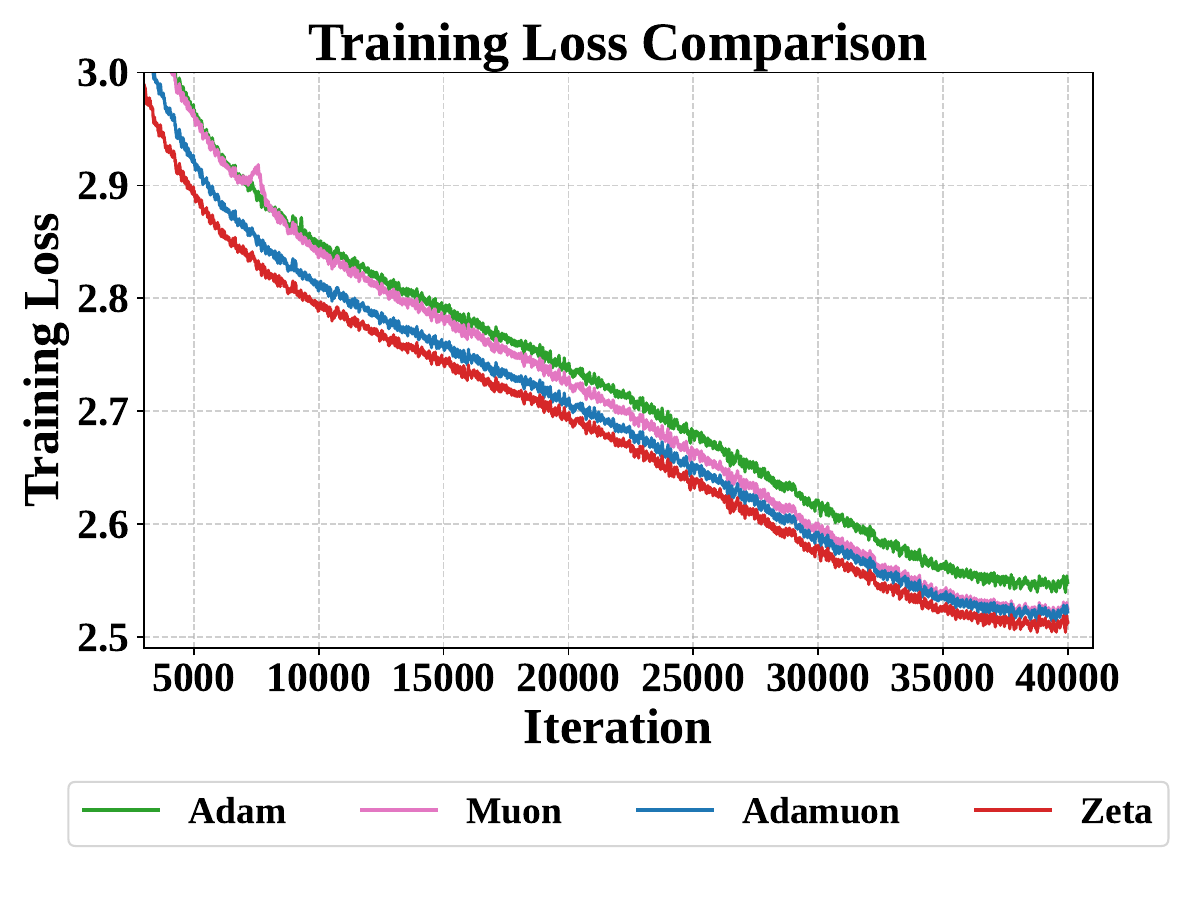}
    \end{minipage}
    \hfill 
    \begin{minipage}{0.42\textwidth}
        \small 
        \raggedright 
        
        \textbf{Speedup} (base on iteration, AdamW: 1.00x) \\
        \rule{\linewidth}{0.5pt} \\ 
        Muon: \hfill 1.19x \\
        AdaMuon: \hfill 1.21x \\
        Zeta: \hfill \textbf{1.25x} \\
        \vspace{0.6em}

        \textbf{Speedup} (base on time, AdamW: 1.00x) \\
        \rule{\linewidth}{0.5pt} \\ 
        Muon: \hfill 1.10x \\
        AdaMuon: \hfill 1.13x \\
        Zeta: \hfill \textbf{1.20x} \\
        \vspace{0.6em}
        
        \textbf{Time per Iteration} (wall-clock on 16 GPU) \\
        \rule{\linewidth}{0.5pt} \\
        AdamW: \hfill 24.46s \\
        Muon: \hfill 26.19s \\
        AdaMuon: \hfill 26.26s \\
        Zeta: \hfill 25.47s \\
        \vspace{0.6em}
        
    \end{minipage}

    \vspace{-7pt}
    \caption{Training loss curves for Qwen3-8B (left) and speedup metrics (right).}
    \label{fig:main_8b_dashboard}
    \vspace{-0.7em}
\end{figure}

\textbf{Performance on MoE Architectures.}
MoE training introduces unique optimization challenges because the sparse expert activation combined with shared dense blocks leads to substantial gradient heterogeneity across parameters.
Figure~\ref{fig:moe_results} shows that, under this more challenging setting, Zeta maintains a clear advantage. On the Qwen3-1.3B-A0.6B model, Zeta converges to a lower loss with less fluctuation than both Muon and AdaMuon, and achieves the strongest iteration-based speedup (1.47$\times$ over AdamW vs. 1.18$\times$ for Muon and 1.30$\times$ for AdaMuon) while keeping per-step time close to AdamW (11.93s vs. 10.87s). This suggests that dual whitening is particularly helpful under the dense–sparse statistical imbalance typical of MoE training.

\begin{figure}[t]
\vspace{-15pt}
    \centering
    \begin{minipage}{0.55\textwidth}
        \centering
        \includegraphics[width=\textwidth]{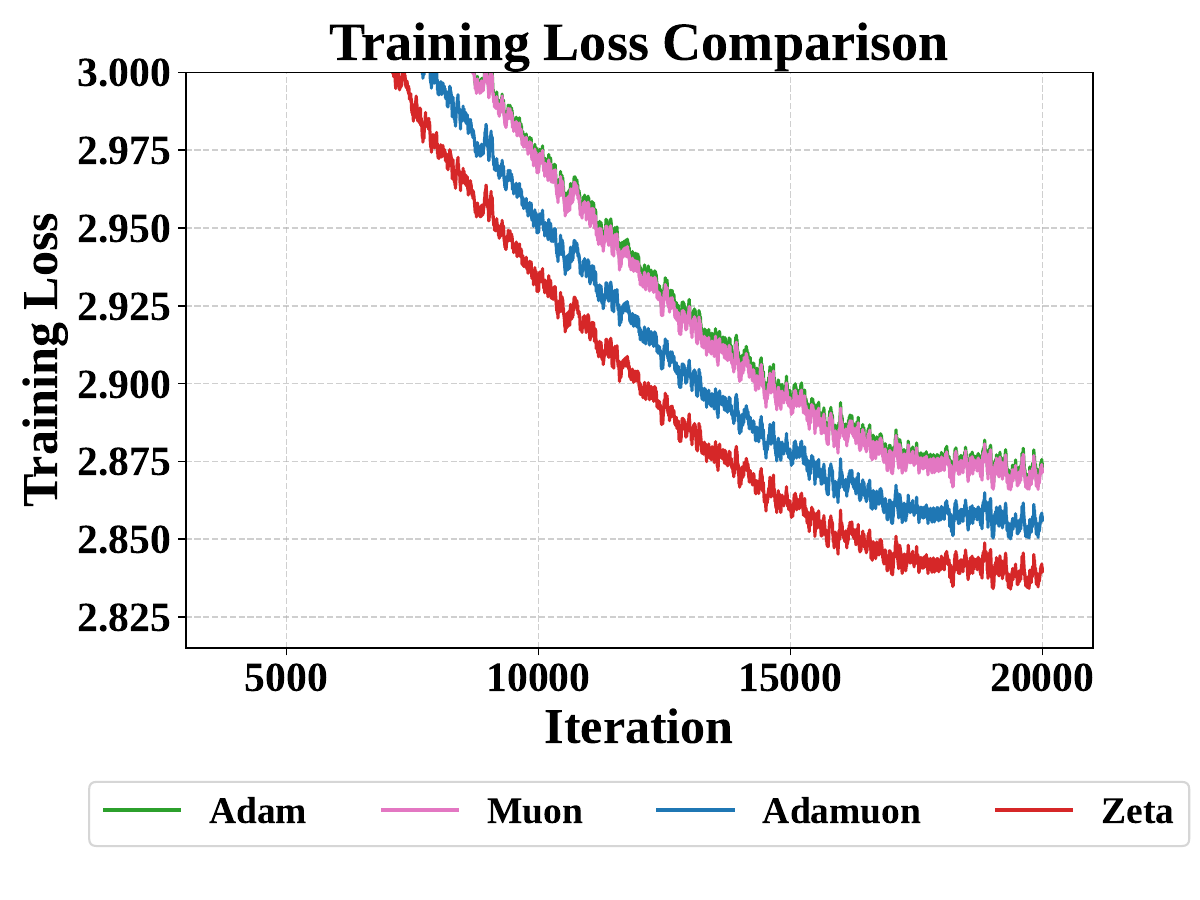}
    \end{minipage}
    \hfill 
    \begin{minipage}{0.42\textwidth}
        \small 
        \raggedright 
        
        \textbf{Speedup} (base on iteration, AdamW: 1.00x) \\
        \rule{\linewidth}{0.5pt} \\ 
        Muon: \hfill 1.18x \\
        AdaMuon: \hfill 1.30x \\
        Zeta: \hfill \textbf{1.47x} \\
        \vspace{0.6em}

        \textbf{Speedup} (base on time, AdamW: 1.00x) \\
        \rule{\linewidth}{0.5pt} \\ 
        Muon: \hfill 1.17x \\
        AdaMuon: \hfill 1.19x \\
        Zeta: \hfill \textbf{1.32x} \\
        \vspace{0.6em}
        
        \textbf{Time per Iteration} (wall-clock on 8 GPU) \\
        \rule{\linewidth}{0.5pt} \\
        AdamW: \hfill 10.87s \\
        Muon: \hfill 11.00s \\
        AdaMuon: \hfill 12.10s \\
        Zeta: \hfill 11.93s \\
        \vspace{0.6em}
        
    \end{minipage}

    \vspace{-5pt}
    \caption{Training loss curves for Qwen3-1.3B-A0.6B (left) and speedup metrics (right).}
    \label{fig:moe_results}
    \vspace{-1em}
\end{figure}


\subsection{Comparisons on Language Understanding Tasks}
\label{subsec:downstream}

We evaluate the trained models on 12 downstream benchmarks spanning commonsense reasoning, world knowledge, mathematical problem-solving, Chinese language understanding, and long-context retrieval. As shown in Table~\ref{tab:downstream_results}, Zeta achieves the best average accuracy on both model scales. On Qwen3-8B, Zeta reaches 26.29\%, outperforming AdamW, Muon, and AdaMuon by $12.26\%\!\uparrow$, $1.00\%\!\uparrow$, and $5.08\%\!\uparrow$. On Qwen3-1.7B, Zeta also ranks first at 23.69\%, with gains of $1.76\%\!\uparrow$ over AdamW, $4.22\%\!\uparrow$ over Muon, and $7.83\%\!\uparrow$ over AdaMuon. These results indicate that the dual-whitening design benefits downstream generalization across model scales and task types.

\begin{table*}[t]
    \centering
    \caption{Language understanding tasks results (Accuracy \%) on Qwen3-8B and Qwen3-1.7B.}
    \vspace{-5pt}
    \label{tab:downstream_results}
    \setlength{\tabcolsep}{4pt}
    \footnotesize
    \resizebox{\textwidth}{!}{
        \begin{tabular}{lcccc c cccc}
            \toprule
            & \multicolumn{4}{c}{\textbf{Qwen3-8B}} & & \multicolumn{4}{c}{\textbf{Qwen3-1.7B}} \\
            \cmidrule(lr){2-5} \cmidrule(lr){7-10}
            \textbf{Task} & \textbf{AdamW} & \textbf{Muon} & \textbf{AdaMuon} & \cellcolor{blue!10}\textbf{Zeta} & & \textbf{AdamW} & \textbf{Muon} & \textbf{AdaMuon} & \cellcolor{blue!10}\textbf{Zeta} \\
            \midrule
            HellaSwag\cite{zellers2019hellaswag} & 24.79 & 25.54 & 24.80 & \cellcolor{blue!10}\textbf{26.12} & & 24.98 & \textbf{25.31} & 25.13 & \cellcolor{blue!10}24.62 \\
            MMLU\cite{li2024cmmlu} & 25.48 & \textbf{25.88} & 24.96 & \cellcolor{blue!10}25.08 & & 23.66 & \textbf{24.63} & 23.09 & \cellcolor{blue!10}23.63 \\
            NIAH\cite{li2024needlebench} & 16.70 & 17.70 & 18.90 & \cellcolor{blue!10}\textbf{20.50} & & 6.50 & \textbf{7.65} & 6.35 & \cellcolor{blue!10}7.15 \\
            PIQA\cite{bisk2020piqa} & 49.13 & 48.37 & \textbf{49.18} & \cellcolor{blue!10}48.97 & & 48.86 & 49.13 & 47.77 & \cellcolor{blue!10}\textbf{50.65} \\
            ARC\_c\cite{clark2018think} & 24.75 & \textbf{26.44} & 25.42 & \cellcolor{blue!10}26.10 & & 24.75 & 24.07 & \textbf{25.08} & \cellcolor{blue!10}\textbf{25.08} \\
            C-Eval\cite{huang2023c} & 15.48 & 20.26 & 21.73 & \cellcolor{blue!10}\textbf{21.77} & & \textbf{21.81} & 21.16 & 9.87 & \cellcolor{blue!10}20.51 \\
            TriviaQA\cite{joshi2017triviaqa} & 34.43 & 37.11 & 36.70 & \cellcolor{blue!10}\textbf{37.48} & & 19.81 & 23.22 & 23.46 & \cellcolor{blue!10}\textbf{24.96} \\
            OBQA\cite{mihaylov2018can} & 25.60 & \textbf{29.80} & 24.80 & \cellcolor{blue!10}26.00 & & 25.60 & 19.80 & 25.60 & \cellcolor{blue!10}\textbf{28.00} \\
            WinoG\cite{sakaguchi2021winogrande} & 50.99 & 48.78 & \textbf{51.14} & \cellcolor{blue!10}49.41 & & 48.86 & 48.93 & 49.01 & \cellcolor{blue!10}\textbf{49.25} \\
            CHID\cite{zheng2019chid} & 6.49 & 13.54 & \textbf{14.29} & \cellcolor{blue!10}13.64 & & \textbf{14.29} & \textbf{14.29} & 13.79 & \cellcolor{blue!10}9.90 \\
            CMMLU\cite{li2024cmmlu} & 5.40 & 16.34 & 5.86 & \cellcolor{blue!10}\textbf{17.54} & & 18.75 & 12.74 & 12.49 & \cellcolor{blue!10}\textbf{18.91} \\
            GSM8k\cite{cobbe2021training} & 1.82 & 2.65 & 2.43 & \cellcolor{blue!10}\textbf{2.81} & & 1.52 & 1.82 & \textbf{1.97} & \cellcolor{blue!10}1.67 \\
            \rowcolor{red!10}
            \textbf{AVG} & 23.42 & 26.03 & 25.02 & \textbf{26.29} & & 23.28 & 22.73 & 21.97 & \textbf{23.69} \\
            \bottomrule
        \end{tabular}
    }
    \vspace{-1em}
\end{table*}


\subsection{Comparisons on Vision Tasks}
\label{subsec:vision}

\begin{table}[!h]
    \centering
    \begin{minipage}{0.45\textwidth}
        \centering
        \captionof{table}{Top-1 validation accuracy (\%) on CIFAR-100 for vision classification models.}
        \label{tab:vision_results}
        \vspace{0.4ex}
        \resizebox{0.88\textwidth}{!}{
        \begin{tabular}{lcc}
            \toprule
            \textbf{Optimizer} & \textbf{ViT-Tiny} & \textbf{ViT-Base} \\
            \midrule
            AdamW   & 57.02 & 57.98 \\
            Muon    & 64.68 & 64.53 \\
            AdaMuon & 63.96 & 62.03 \\
            \rowcolor{red!10}
            \textbf{Zeta (Ours)} & \textbf{64.98} & \textbf{65.34} \\
            \bottomrule
        \end{tabular}}
    \end{minipage}
    \hfill
    \begin{minipage}{0.48\textwidth}
        \centering
        \captionof{table}{Effect of final training loss to $\beta_1$ and $\beta_2$ on the Qwen3-0.6B dense model.}
        \label{tab:beta_sensitivity}
        \vspace{0.4ex}
        \resizebox{0.80\textwidth}{!}{
        \begin{tabular}{l|ccc}
            \toprule
            $\beta_1$ / $\beta_2$ & 0.9 & 0.95 & 0.99 \\ \midrule
            0.9  & 2.571 & 2.570 & \textbf{2.564} \\
            0.95 & 2.565 & 2.565 & \textbf{2.564} \\
            0.99 & 2.593 & 2.580 & 2.579 \\ \bottomrule
        \end{tabular}}
    \end{minipage}
    \vspace{-1em}
\end{table}


We evaluate ViT-Tiny and ViT-Base on CIFAR-100 \cite{krizhevsky2009learning} using the Vision Transformer (ViT) \cite{dosovitskiy2020image} architecture. All models are trained from scratch for 50 epochs with weight decay set to 0.05, while the peak learning rate and batch size are tuned separately for each optimizer. As shown in Table~\ref{tab:vision_results}, Zeta consistently achieves the highest Top-1 validation accuracy on both scales, reaching 64.98\% on ViT-Tiny and 65.34\% on ViT-Base. On ViT-Tiny, it outperforms AdamW, Muon, and AdaMuon by $13.96\%\!\uparrow$, $0.46\%\!\uparrow$, and $1.59\%\!\uparrow$, respectively. On ViT-Base, the gains are $12.69\%\!\uparrow$ over AdamW, $1.26\%\!\uparrow$ over Muon, and $5.34\%\!\uparrow$ over AdaMuon. These results indicate that the dual-whitening design generalizes well across vision architectures of different capacities.

\subsection{Effect of Momentum Coefficients}
\label{subsec:sensitivity}


We examine the effect of the momentum coefficients $\beta_1$ and $\beta_2$, since Zeta inherits the first-moment parameter \(\beta_1\) from Muon and the second-moment parameter \(\beta_2\) from AdamW. A grid search over \(\beta_1, \beta_2 \in \{0.9, 0.95, 0.99\}\) on the Qwen3-0.6B model yields the results in Table~\ref{tab:beta_sensitivity}. The final training loss varies by less than \(1.2\%\) across the entire grid, and the best value is obtained at both \((0.9, 0.99)\) and \((0.95, 0.99)\). This indicates that Zeta is robust to momentum hyperparameters and works well under standard settings, requiring little hyperparameter tuning.

\section{Conclusion}

We show raw momentum systematically violates the scale-normalization required by spectral whitening. Thus, we propose Zeta, a dual optimizer sequencing coordinate before spectral whitening. Zeta provably reduces orthogonalization error via improved conditioning, and matches or surpasses baselines across language modeling, mixture-of-experts, and vision tasks.



{
    \bibliographystyle{unsrt}
    \bibliography{refs}
}


\newpage
\appendix

\begin{leftline}
	{
		\LARGE{\textsc{Appendix}}
	}
\end{leftline}

\etocdepthtag.toc{mtappendix}
\etocsettagdepth{mtchapter}{none}
\etocsettagdepth{mtappendix}{subsection}

{
    \hypersetup{linkcolor=black}
        \footnotesize\tableofcontents
}

\newpage

\section{Theoretical Analysis}
\label{sec:proofs}

\subsection{Proof of Theorem~\ref{thm:coord_whitening}}

\textbf{Theorem~\ref{thm:coord_whitening}.}
\emph{Let $M_t \in \mathbb{R}^{m \times n}$ be the momentum matrix at iteration $t$, with vectorized form $m_t = \mathrm{vec}(M_t)$. Assume $m_t$ is approximately centered, $\mathbb{E}[m_t] \approx 0$, and its covariance is approximately diagonal, $\mathrm{Cov}(m_t) \approx \Sigma_t = \mathrm{diag}(\sigma_{t,1}^2, \dots, \sigma_{t,mn}^2)$. The ideal whitening transform is
\[
    \hat{m}_t = \Sigma_t^{-1/2} m_t.
\]
In Zeta, we instead use the element-wise second-moment estimator $V_t = \beta_2 V_{t-1} + (1-\beta_2) G_t^2$ and define
\[
    \tilde{G}_t = \frac{M_t}{\sqrt{V_t}+\varepsilon}.
\]
If $V_t$ consistently estimates the diagonal second-order moment, then, under the centering assumption, this is also a consistent estimate of the diagonal covariance. Therefore, $\mathrm{vec}(\tilde{G}_t) \approx \mathrm{diag}(V_t)^{-1/2} \, \mathrm{vec}(M_t)$ is a diagonal approximation to $\Sigma_t^{-1/2} m_t$. Hence, Zeta implements approximate coordinate whitening.}

\begin{proof}
The ideal whitening transform for the vectorized momentum $m_t$ is
\[
    \hat{m}_t = \Sigma_t^{-1/2} m_t
\]
where $\Sigma_t = \mathrm{Cov}(m_t)$. Since $\mathbb{E}[m_t] \approx 0$, the covariance agrees with the second-order moment up to a negligible centering error, i.e.,
\[
    \Sigma_t = \mathbb{E}[m_t m_t^\top] - \mathbb{E}[m_t]\mathbb{E}[m_t]^\top
    \approx
    \mathbb{E}[m_t m_t^\top].
\]
Under the diagonal-covariance assumption,
\[
    \Sigma_t \approx \mathrm{diag}(\sigma_{t,1}^2, \dots, \sigma_{t,mn}^2)
\]
so its inverse square root reduces to coordinate-wise scaling:
\[
    \Sigma_t^{-1/2} \approx \mathrm{diag}(\sigma_{t,1}^{-1}, \dots, \sigma_{t,mn}^{-1})
\]
Therefore, ideal whitening acts independently on each coordinate.

In Zeta, the second-moment accumulator is updated element-wise as
\[
    V_t = \beta_2 V_{t-1} + (1-\beta_2) G_t^2
\]
If $V_t$ consistently estimates the diagonal second-order moment, then after vectorization,
\[
    \mathrm{diag}(V_t) \approx \mathrm{diag}(\mathbb{E}[m_{t,1}^2], \dots, \mathbb{E}[m_{t,mn}^2])
    \approx \mathrm{diag}(\sigma_{t,1}^2, \dots, \sigma_{t,mn}^2)
\]
Hence,
\[
    \mathrm{diag}(V_t)^{-1/2}
    \approx
    \mathrm{diag}(\sigma_{t,1}^{-1}, \dots, \sigma_{t,mn}^{-1})
    \approx
    \Sigma_t^{-1/2}
\]
Applying this diagonal preconditioner to $m_t$ gives
\[
    \mathrm{vec}(\tilde{G}_t)
    = \mathrm{diag}(V_t)^{-1/2} \, \mathrm{vec}(M_t)
\]
up to the perturbation introduced by the numerical stabilizer $\varepsilon$. Thus, the Zeta update coincides with the ideal whitening transform up to the approximation that off-diagonal covariances are ignored and the second moment is estimated from running statistics. Consequently, Zeta implements approximate coordinate whitening.
\end{proof}

\subsection{Proof of Theorem~\ref{thm:ns_orthogonality}}

\textbf{Theorem~\ref{thm:ns_orthogonality}.}
\emph{Let $X \in \mathbb{R}^{m \times n}$ be the input to Newton-Schulz orthogonalization, and let $Q_K(X)$ be the output after $K$ iterations. Suppose $\tilde{X}$ is the whitened version of $X$, with a more concentrated singular-value spectrum and smaller condition number. Then, for fixed $K$, the orthogonalization error satisfies
\[
    \|Q_K(\tilde{X})^\top Q_K(\tilde{X}) - I\|_F
    \leq
    \|Q_K(X)^\top Q_K(X) - I\|_F
\]
up to higher-order truncation terms. Thus, whitening makes the post-NS update closer to an ideal orthogonal matrix.}

\begin{proof}
We make explicit the scalar spectral dynamics induced by the Newton-Schulz iteration used in Zeta. Let
\[
    X_0 = \frac{G}{\|G\|_F}
\]
be the normalized matrix fed into the NS step, and write its compact singular value decomposition as
\[
    X_0 = U \Sigma_0 V^\top,
    \qquad
    \Sigma_0 = \mathrm{diag}(\sigma_{0,1},\dots,\sigma_{0,d}),
    \qquad d = \min(m,n).
\]
Because $\|X_0\|_F^2 = 1$, we have
\[
    \sum_{i=1}^d \sigma_{0,i}^2 = 1,
\]
and in particular each nonzero singular value lies in $(0,1]$. The NS iteration used in practice has the form
\[
    X_{k+1} = a X_k + b (X_k X_k^\top) X_k + c (X_k X_k^\top)^2 X_k,
\]
with coefficients $a=3.4445$, $b=-4.7750$, and $c=2.0315$. Since this iteration is equivariant under orthogonal changes of basis, the left and right singular vectors remain unchanged and only the singular values evolve. Hence
\[
    X_k = U \Sigma_k V^\top,
    \qquad
    \Sigma_k = \mathrm{diag}(\sigma_{k,1},\dots,\sigma_{k,d}),
\]
where each singular value satisfies the scalar recurrence
\[
    \sigma_{k+1,i} = f(\sigma_{k,i}),
    \qquad
    f(x) = a x + b x^3 + c x^5.
\]
Let $f^{(K)}$ denote the $K$-fold composition of $f$. Then after $K$ iterations,
\[
    \sigma_{K,i} = f^{(K)}(\sigma_{0,i}).
\]
Therefore the orthogonality defect can be written exactly as
\begin{align}
    \|X_K^\top X_K - I\|_F^2
    &= \sum_{i=1}^d \bigl(\sigma_{K,i}^2 - 1\bigr)^2 \\
    &= \sum_{i=1}^d \Bigl(\bigl[f^{(K)}(\sigma_{0,i})\bigr]^2 - 1\Bigr)^2.
    \label{eq:ns_loss_spectral}
\end{align}
Define the single-mode error function
\[
    g(x) = \Bigl(\bigl[f^{(K)}(x)\bigr]^2 - 1\Bigr)^2.
\]
Then Eq.~\eqref{eq:ns_loss_spectral} becomes
\[
    \|X_K^\top X_K - I\|_F^2 = \sum_{i=1}^d g(\sigma_{0,i}).
\]
This shows that the NS orthogonality error is completely determined by the initial singular-value distribution.

We next examine the behavior of $g(x)$ near the small-singular-value regime that dominates the conditioning difficulty. As $x \to 0^+$,
\[
    f(x) = a x + O(x^3),
\]
so repeated composition gives
\[
    f^{(K)}(x) = a^K x + O(x^3).
\]
Substituting into $g$ yields
\[
    g(x)
    = \bigl(a^{2K} x^2 - 1\bigr)^2 + O(x^4)
    = 1 - 2 a^{2K} x^2 + O(x^4).
\]
Hence $g(x) \to 1$ as $x \to 0^+$: very small initial singular values contribute nearly the maximal possible error. This makes the orthogonality defect especially sensitive to poorly conditioned spectra with tiny trailing singular values.

Now compare two inputs, $G$ and its whitened version $\tilde{G}$, both normalized to unit Frobenius norm before entering NS. Let their initial singular values be $\{\sigma_{0,i}\}_{i=1}^d$ and $\{\tilde{\sigma}_{0,i}\}_{i=1}^d$, respectively. By assumption,
\[
    \kappa(\tilde{G}) \le \kappa(G),
\]
and the whitening step makes the spectrum of $\tilde{G}$ more concentrated. Under the common-energy constraint
\[
    \sum_{i=1}^d \sigma_{0,i}^2 = \sum_{i=1}^d \tilde{\sigma}_{0,i}^2 = 1,
\]
a smaller condition number means, in particular, that the smallest singular values of $\tilde{G}$ are lifted away from zero and the spectrum is less dispersed. On the NS convergence interval relevant here, the function $g(x)$ is decreasing in $x$: increasing a small singular value reduces its contribution to the orthogonality defect. Consequently, replacing $\{\sigma_{0,i}\}$ by the more balanced set $\{\tilde{\sigma}_{0,i}\}$ decreases the sum of mode-wise errors, up to the finite-iteration truncation terms already absorbed in the scalar approximation. Therefore,
\[
    \sum_{i=1}^d g(\tilde{\sigma}_{0,i})
    \le
    \sum_{i=1}^d g(\sigma_{0,i}),
\]
which is equivalent to
\[
    \|NS_K(\tilde{G})^\top NS_K(\tilde{G}) - I\|_F^2
    \le
    \|NS_K(G)^\top NS_K(G) - I\|_F^2.
\]
Taking square roots gives
\[
    \|NS_K(\tilde{G})^\top NS_K(\tilde{G}) - I\|_F
    \le
    \|NS_K(G)^\top NS_K(G) - I\|_F,
\]
up to the higher-order truncation terms from the finite-$K$ polynomial approximation. Thus, whitening improves the conditioning of the NS input and makes the post-iteration matrix closer to an ideal orthogonal factor.
\end{proof}

\subsection{Proof of Theorem~\ref{thm:rms_scaling}}

\textbf{Theorem~\ref{thm:rms_scaling}.}
\emph{Let $U_t \in \mathbb{R}^{m \times n}$ be the orthogonalized matrix update produced by Newton-Schulz iteration, and define
\[
    s_t = \frac{0.2}{\mathrm{RMS}(U_t)} = \frac{0.2\sqrt{mn}}{\|U_t\|_F}.
\]
Let $\Delta_t = s_t U_t$. Then $\mathrm{RMS}(\Delta_t)=0.2$, and $\Delta_t$ has the same direction as $U_t$ up to a positive scalar factor.}

\begin{proof}
By definition of RMS for a matrix,
\[
    \mathrm{RMS}(U_t) = \sqrt{\frac{1}{mn}\|U_t\|_F^2} = \frac{\|U_t\|_F}{\sqrt{mn}}
\]
Substituting this identity into the definition of $s_t$ gives
\[
    s_t = \frac{0.2}{\mathrm{RMS}(U_t)} = \frac{0.2\sqrt{mn}}{\|U_t\|_F}
\]
Therefore,
\[
    \mathrm{RMS}(\Delta_t)
    = \sqrt{\frac{1}{mn}\|s_t U_t\|_F^2}
    = s_t \frac{\|U_t\|_F}{\sqrt{mn}}
    = 0.2
\]
where we use the positive homogeneity of the Frobenius norm. Moreover, since $s_t>0$, the scaled update $\Delta_t=s_tU_t$ is a positive scalar multiple of $U_t$, so it lies on the same ray in matrix space and preserves the orthogonalized direction. Thus the post-NS normalization fixes the update RMS while leaving the update direction unchanged.
\end{proof}

\subsection{Statistical Analysis of the Motivation Experiment}

In the motivation experiment of Figure~\ref{fig:motivation_pvalue}, we quantify whether the row-wise energy of a matrix is uniformly distributed before and after coordinate whitening. The goal is to measure the degree of intra-matrix anisotropy from a statistical perspective and to verify whether whitening makes different rows more balanced in scale.

Let $G \in \mathbb{R}^{m \times n}$ denote the matrix under analysis, and let $g_i \in \mathbb{R}^n$ be its $i$-th row. We first compute the row norm
\[
    r_i = \|g_i\|_2
\]
and define the corresponding row energy as
\[
    E_i = r_i^2 = \|g_i\|_2^2, \qquad i = 1,2,\dots,m
\]

We then perform a chi-square uniformity test on the set of row energies. The null hypothesis is
\[
    H_0: E_1 = E_2 = \cdots = E_m
\]
which states that all rows have the same energy and therefore the matrix is statistically uniform across rows.

To make the test invariant to the absolute scale of the matrix, we normalize each row energy by the average row energy,
\[
    O_i = \frac{E_i}{\bar{E}},
    \qquad
    \bar{E} = \frac{1}{m} \sum_{i=1}^{m} E_i
\]
By construction, the normalized observations satisfy
\[
    \frac{1}{m} \sum_{i=1}^{m} O_i = 1,
\]
so under the null hypothesis each row has expected value $1$.

The chi-square statistic is therefore computed as
\[
    X^2 = \sum_{i=1}^{m} \frac{(O_i - 1)^2}{1}
    = \sum_{i=1}^{m} (O_i - 1)^2
\]
with degrees of freedom
\[
    df = m - 1
\]
The associated $p$-value is the upper-tail probability of the chi-square distribution,
\[
    p = \mathbb{P}(\chi^2_{df} \geq X^2)
      = Q\left(\frac{df}{2}, \frac{X^2}{2}\right)
\]
where $Q(a,x)$ denotes the regularized upper incomplete gamma function.

This test has a clear optimization interpretation. When the $p$-value is close to $1$, the row energies are nearly uniform and we cannot reject the hypothesis that the matrix is balanced across rows. In contrast, when the $p$-value approaches $0$, some rows have substantially larger or smaller energy than others, indicating pronounced anisotropy in the gradient distribution. In our motivation experiment, coordinate whitening consistently shifts the $p$-value upward, showing that whitening reduces row-wise energy disparity and makes the matrix more compatible with stable Newton-Schulz orthogonalization.

\subsection{Storage and Time Complexity Comparisons}

We next analyze the per-parameter-state complexity and per-step computational complexity of Zeta relative to AdamW, Muon, and AdaMuon. Consider a matrix parameter $W \in \mathbb{R}^{m\times n}$ and let $r=\min(m,n)$. We treat element-wise additions, multiplications, and square roots as $\mathcal{O}(mn)$ operations. For Newton-Schulz orthogonalization with $K$ iterations, the dominant cost comes from matrix multiplications such as $XX^\top X$ or its transposed counterpart, whose per-iteration cost is $\mathcal{O}(mnr)$. Hence the total orthogonalization cost is $\mathcal{O}(K m n r)$.

Under this notation, AdamW maintains first- and second-moment states for each coordinate, so its optimizer-state storage is $\mathcal{O}(2mn)$ and its per-step update cost is $\mathcal{O}(mn)$. Muon maintains only a momentum matrix before applying Newton-Schulz orthogonalization, so its storage is reduced to $\mathcal{O}(mn)$, while its per-step complexity becomes
\[
    \mathcal{O}(mn) + \mathcal{O}(K m n r) = \mathcal{O}(K m n r)
\]
AdaMuon augments Muon with Adam-style adaptive scaling, so it stores both momentum and second-moment statistics, giving storage $\mathcal{O}(2mn)$, and its per-step complexity is
\[
    \mathcal{O}(mn) + \mathcal{O}(K m n r) = \mathcal{O}(K m n r)
\]

Zeta has the same asymptotic storage as AdaMuon, because it also keeps a first-moment matrix $M_t$ and a second-moment matrix $V_t$, resulting in optimizer-state storage $\mathcal{O}(2mn)$. Its additional coordinate-whitening step,
\[
    \tilde{G}_t = \frac{M_t}{\sqrt{V_t}+\varepsilon}
\]
requires only element-wise operations and therefore costs $\mathcal{O}(mn)$. After whitening, Zeta performs the same $K$-step Newton-Schulz orthogonalization as Muon-type methods. Therefore the total per-step complexity of Zeta is
\[
    \mathcal{O}(mn) + \mathcal{O}(mn) + \mathcal{O}(K m n r)
    = \mathcal{O}(K m n r)
\]
This shows that Zeta does not increase the asymptotic time complexity relative to AdaMuon: its extra whitening stage changes only the lower-order element-wise cost, while the dominant term remains the spectral orthogonalization.

The comparison can be summarized as follows:
\begin{table}[t]
    \centering
    \caption{Comparison of optimizer-state storage complexity and per-step time complexity for a matrix parameter $W \in \mathbb{R}^{m \times n}$, where $r = \min(m,n)$.}
    \label{tab:complexity_comparison}
    \begin{tabular}{lcc}
        \toprule
        Optimizer & Storage Complexity & Per-step Time Complexity \\
        \midrule
        AdamW & $\mathcal{O}(2mn)$ & $\mathcal{O}(mn)$ \\
        Muon & $\mathcal{O}(mn)$ & $\mathcal{O}(K m n r)$ \\
        AdaMuon & $\mathcal{O}(2mn)$ & $\mathcal{O}(K m n r)$ \\
        Zeta & $\mathcal{O}(2mn)$ & $\mathcal{O}(K m n r)$ \\
        \bottomrule
    \end{tabular}
\end{table}
Thus, compared with AdamW, Zeta trades linear-time diagonal adaptation for a spectral orthogonalization cost dominated by Newton-Schulz iteration. Compared with Muon, Zeta incurs one additional second-moment state, increasing storage from $\mathcal{O}(mn)$ to $\mathcal{O}(2mn)$. Compared with AdaMuon, however, Zeta has the same asymptotic storage and time complexity, indicating that its empirical gains come from a better ordering of adaptive whitening and orthogonalization rather than from higher asymptotic computational cost.

\subsection{Second-Moment Whitening vs. Sign-Based Adaptation in AdaMuon}\label{app:adam_vs_sign}

Although Zeta and AdaMuon share the same asymptotic complexity, their adaptive preprocessing steps are driven by fundamentally different geometric principles. To formalize this distinction, we first consider the ideal whitening transform for a gradient vector $g = \nabla_\theta L$. 

In the full-matrix setting, Newton's method defines the parameter update via the inverse Hessian $H^{-1}g$. Under the standard information-geometric approximation, the Hessian is replaced by the Fisher Information Matrix (FIM), $F$:
\begin{equation}
    H \approx F = \mathbb{E}\left[g g^\top\right].
\end{equation}
Theoretical whitening amounts to preconditioning the gradient by the inverse square root of the curvature operator, transforming the gradient into an isotropic space:
\begin{equation}
    \tilde{g}_{\text{ideal}} = F^{-1/2} g = \left( \mathbb{E}\left[g g^\top\right] \right)^{-1/2} g.
\end{equation}
This ensures that the covariance of the whitened gradients is approximately the identity matrix, i.e., $\mathrm{Cov}(\tilde{g}_{\text{ideal}}) \approx I$.

\textbf{Zeta: Diagonal Second-Moment Whitening.}
Storing and inverting the full Fisher matrix $F$ is infeasible at scale. However, as established in Theorem~\ref{thm:coord_whitening}, Zeta's Adam-style coordinate normalization acts as a diagonal approximation to ideal whitening. Let $M_t$ and $V_t$ be the exponential moving averages (EMA) of the first and second moments of the gradient, respectively. The Zeta preconditioned gradient is given by:
\begin{equation}
    \tilde{G}_t^{\text{Zeta}} = \frac{M_t}{\sqrt{V_t} + \varepsilon}.
\end{equation}
Asymptotically, $V_t$ estimates the diagonal elements of the uncentered Fisher matrix. Let $D = \mathrm{diag}(F) = \mathrm{diag}(\mathbb{E}[g \odot g])$, where $\odot$ is the Hadamard product. Zeta implicitly applies a preconditioner matrix $P_{\text{Zeta}}$:
\begin{equation}
    P_{\text{Zeta}} \approx D^{-1/2} = \mathrm{diag}\left(\mathbb{E}\left[g^2\right]\right)^{-1/2}.
\end{equation}
Thus, the transformation can be expressed as:
\begin{equation}
    \tilde{g}_{\text{Zeta}} \approx P_{\text{Zeta}} \, g = \mathrm{diag}\left(\mathbb{E}\left[g^2\right]\right)^{-1/2} g.
\end{equation}
While $P_{\text{Zeta}}$ is a diagonal approximation of $F^{-1/2}$, it strictly preserves the core statistical objective of whitening: coordinates with variance $\sigma_i^2$ are scaled by $1/\sigma_i$. It aligns with local second-order geometry rather than acting as a mere heuristic.

\textbf{AdaMuon: Instantaneous Sign-Based Adaptation.}
By contrast, the $\mathrm{sign}(\cdot)$ transformation in AdaMuon yields a fundamentally different mathematical object. The operation $\tilde{g}_{\text{AdaMuon}} = \mathrm{sign}(g)$ can be rewritten as applying a data-dependent, instantaneous diagonal preconditioner $P_{\text{AdaMuon}}$:
\begin{equation}
    \tilde{g}_{\text{AdaMuon}} = P_{\text{AdaMuon}} \, g, \quad \text{where} \quad P_{\text{AdaMuon}} = \mathrm{diag}\left(|g|\right)^{-1}.
\end{equation}
Comparing $P_{\text{Zeta}}$ and $P_{\text{AdaMuon}}$ reveals the core divergence:
\begin{align}
    P_{\text{Zeta}} &= \mathrm{diag}\left(\mathbb{E}\left[g \odot g\right]\right)^{-1/2} \quad \text{(Stable Second-Order Statistic)}, \\
    P_{\text{AdaMuon}} &= \mathrm{diag}\left(g \odot g\right)^{-1/2} \quad \text{(Instantaneous Realization)}.
\end{align}
AdaMuon discards the expectation operator $\mathbb{E}[\cdot]$. It does not approximate the Fisher diagonal or satisfy Theorem~\ref{thm:coord_whitening}. Geometrically, while Zeta performs a Mahalanobis-type spatial rescaling, AdaMuon performs an $L_\infty$-style normalization that forces $\|\tilde{g}_{\text{AdaMuon}}\|_{\infty} = 1$, equalizing step magnitudes and collapsing relative curvature information.

\textbf{Impact on Newton-Schulz Orthogonalization.}
This geometric distinction is critical for the subsequent Newton-Schulz (NS) orthogonalization. Let $A$ be the input matrix to the NS iteration, and $\kappa(A)$ be its condition number. Theorem~\ref{thm:ns_orthogonality} requires the input to be reasonably well-conditioned ($\kappa(A)$ bounded) for NS to reliably converge to the ideal orthogonal factor $Q$. 

Because Zeta explicitly reduces coordinate-wise anisotropy via $P_{\text{Zeta}}$, it systematically improves the condition number of the gradient matrix $G$:
\begin{equation}
    \kappa\left(P_{\text{Zeta}} G\right) < \kappa(G).
\end{equation}
Conversely, AdaMuon suppresses extreme values but maps $G \mapsto \tilde{G} \in \{-1, 1\}^{n \times m}$, which does not inherently approximate a curvature-based whitening transform and provides weaker theoretical guarantees for $\kappa(P_{\text{AdaMuon}} G)$. Therefore, under identical asymptotic cost, Zeta's adaptive stage provides a mathematically principled preconditioner that naturally fulfills the conditioning prerequisites for downstream spectral orthogonalization.

\section{Details for Experiment Settings}
\label{sec:more exp details}

In this section, we provide the detailed training configurations for all pretraining models and the unified evaluation setup for downstream tasks.

\noindent\textbf{Hardware.} All pretraining runs and inference-time evaluations are conducted on a single node equipped with 16 Ascend 910C accelerators, each with 64GB device memory.

\subsection{Pretraining Configurations}

Unless otherwise specified, all large language model pretraining experiments use a sequence length of 4K. For the Qwen-family models, the learning-rate warm-up stage is set to 200 iterations for 20K-step runs and 400 iterations for 40K-step runs.

\textbf{Qwen3-0.6B.} We train Qwen3-0.6B on 8 GPUs with sequence length $4096$, global batch size $256$, tensor parallelism $\mathrm{tp}=1$, and pipeline parallelism $\mathrm{pp}=4$. The total number of training iterations is $20{,}000$, and the warm-up length is $200$ iterations.

\textbf{Qwen3-1.7B.} We train Qwen3-1.7B on 8 GPUs with sequence length $4096$, global batch size $256$, tensor parallelism $\mathrm{tp}=1$, and pipeline parallelism $\mathrm{pp}=4$. The total number of training iterations is $20{,}000$, and the warm-up length is $200$ iterations.

\textbf{Qwen3-8B.} We train Qwen3-8B on 16 GPUs with sequence length $4096$, global batch size $256$, tensor parallelism $\mathrm{tp}=1$, and pipeline parallelism $\mathrm{pp}=4$. The total number of training iterations is $40{,}000$, and the warm-up length is $400$ iterations.

\textbf{Qwen3-MoE-1.3B-A0.6B.} We train Qwen3-MoE on 8 GPUs with sequence length $4096$, global batch size $256$, tensor parallelism $\mathrm{tp}=1$, pipeline parallelism $\mathrm{pp}=1$, expert parallelism $\mathrm{ep}=1$, and context parallelism $\mathrm{cp}=1$. The total number of training iterations is $20{,}000$, and the warm-up length is $200$ iterations.

\textbf{GPT2-Large.} We train GPT2-Large on 8 GPUs with sequence length $1024$, global batch size $480$, tensor parallelism $\mathrm{tp}=1$, and pipeline parallelism $\mathrm{pp}=1$. The total number of training iterations is $40{,}000$, and the warm-up length is $800$ iterations.

\subsection{Downstream Evaluation Setup}

For downstream evaluation, we adopt EvalScope\footnote{\url{https://github.com/modelscope/evalscope}} as the unified evaluation platform with Opencompass\footnote{\url{https://github.com/open-compass/opencompass}} serving as its backend. Across all downstream tasks, we set the maximum number of newly generated tokens to $128$.

\section{More Experimental Results}

\subsection{Layer-wise Results for the Motivation Experiment}

In this subsection, we report the layer-wise results for the motivation experiment. Figure~\ref{fig:motivation_layers_part1}, Figure~\ref{fig:motivation_layers_part2}, and Figure~\ref{fig:motivation_layers_part3} visualize the $p$-value statistics for all 29 layers, so that the effect of coordinate whitening can be examined throughout the entire network depth rather than only through the aggregated mean result shown in the main text.

\begin{figure*}[!h]
    \centering
    \includegraphics[width=0.32\textwidth]{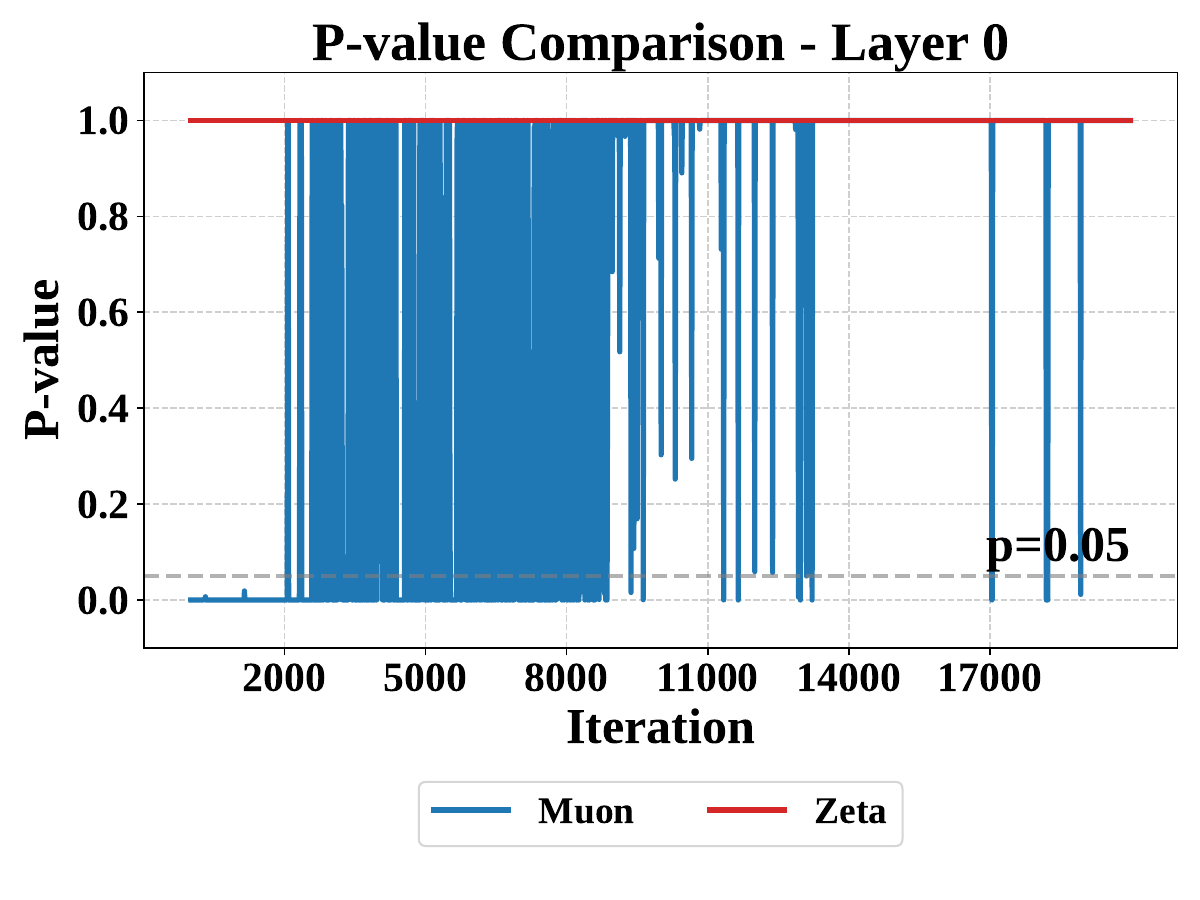}
    \includegraphics[width=0.32\textwidth]{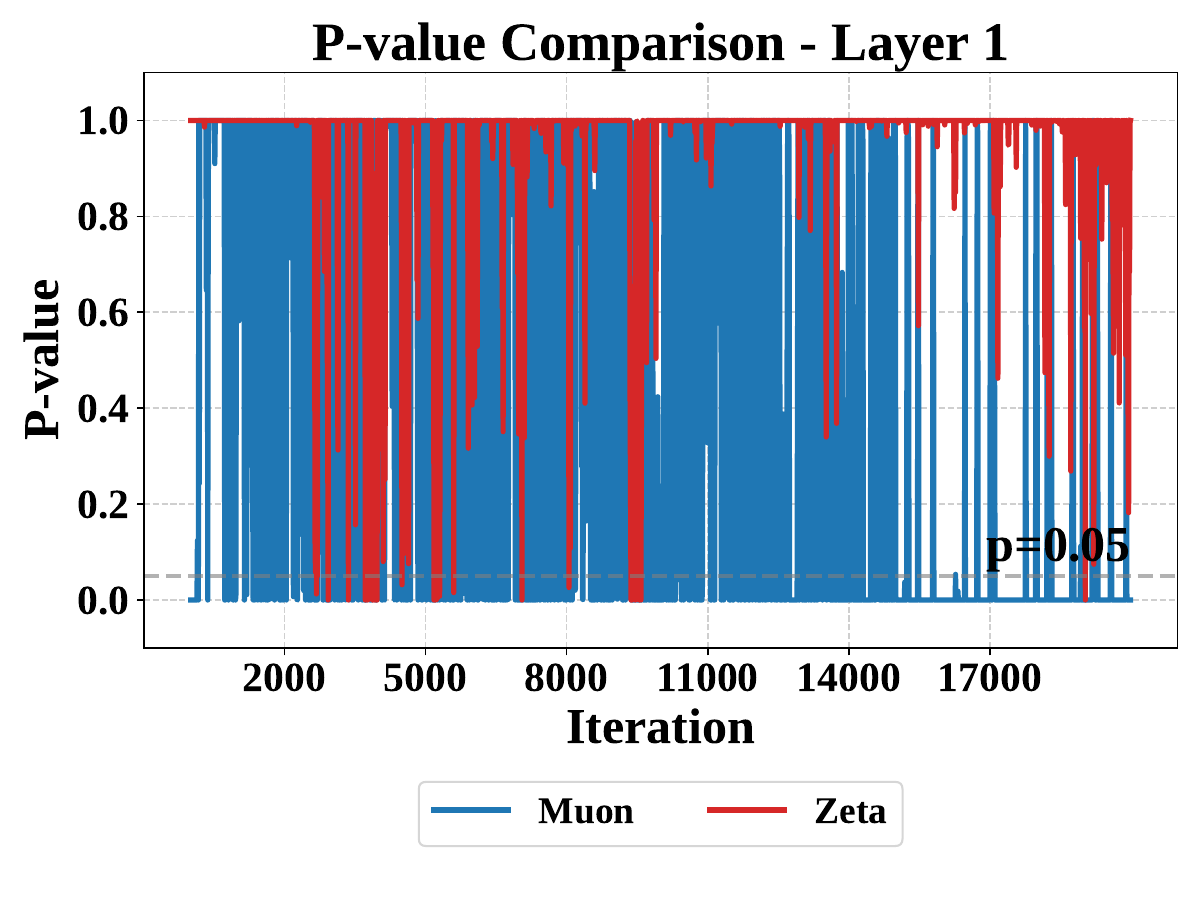}
    \includegraphics[width=0.32\textwidth]{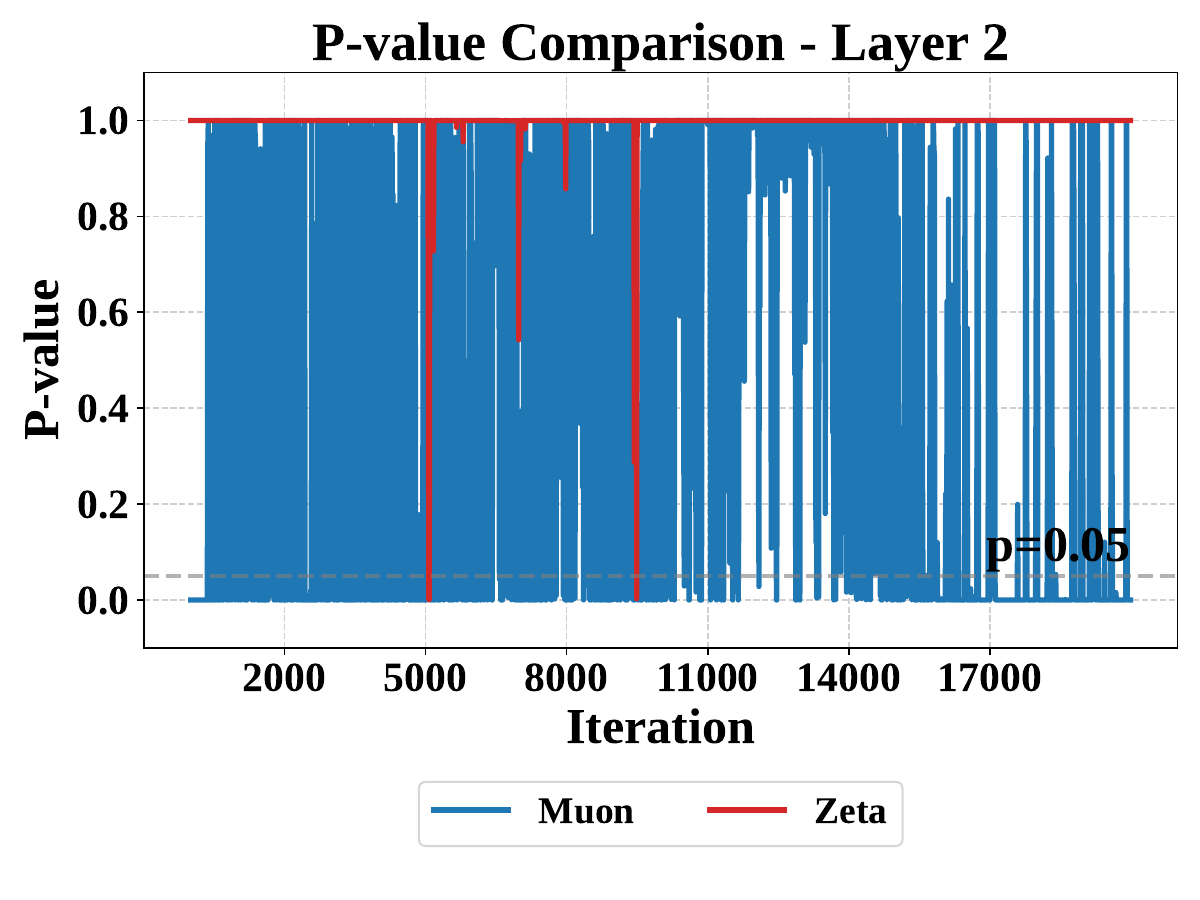}

    \includegraphics[width=0.32\textwidth]{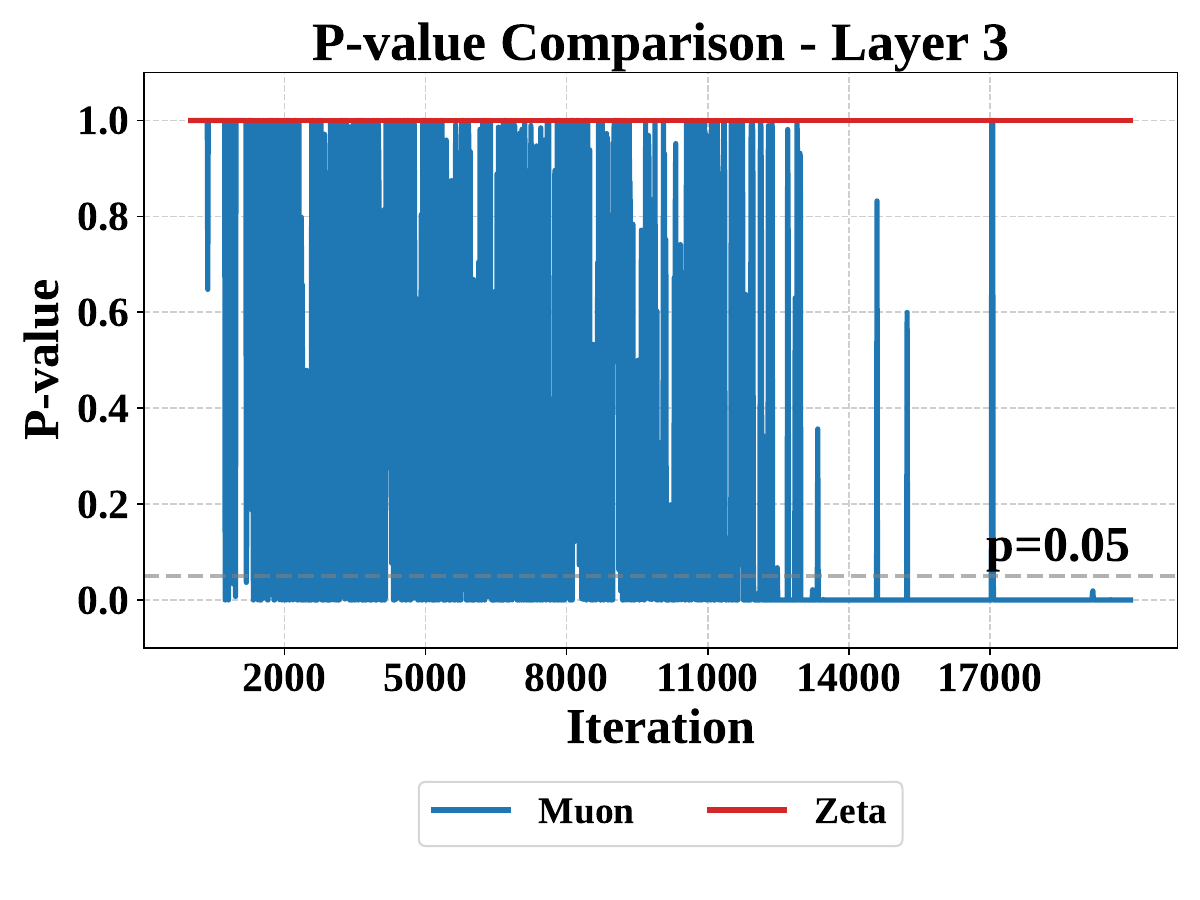}
    \includegraphics[width=0.32\textwidth]{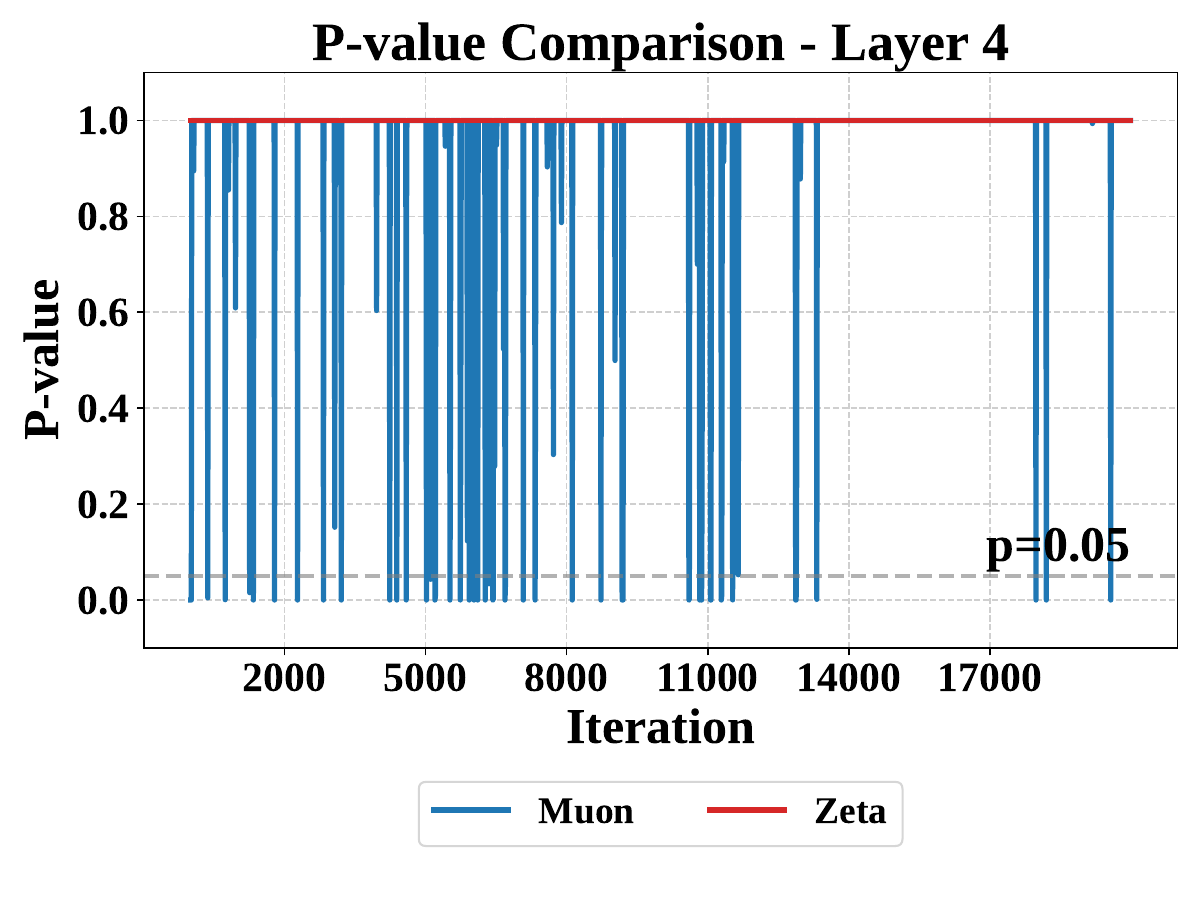}
    \includegraphics[width=0.32\textwidth]{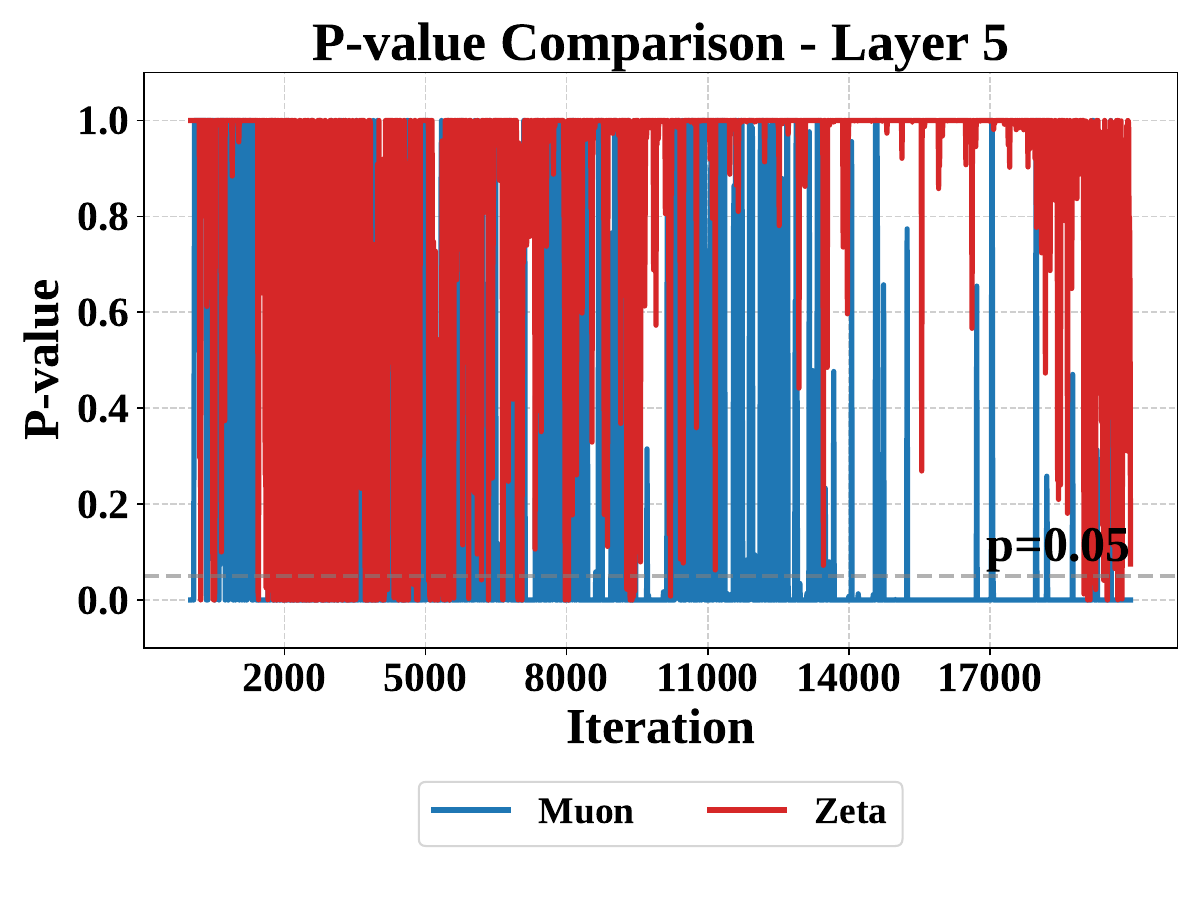}

    \includegraphics[width=0.32\textwidth]{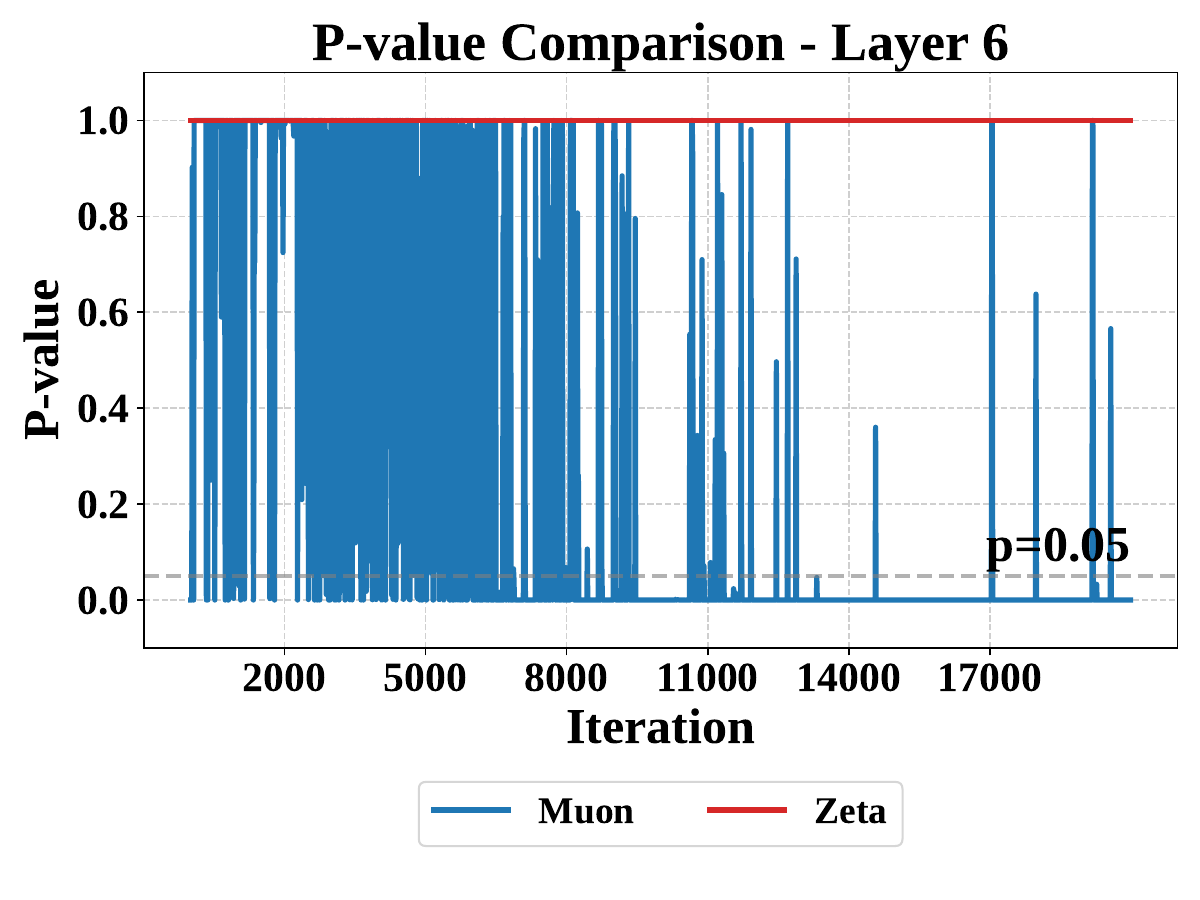}
    \includegraphics[width=0.32\textwidth]{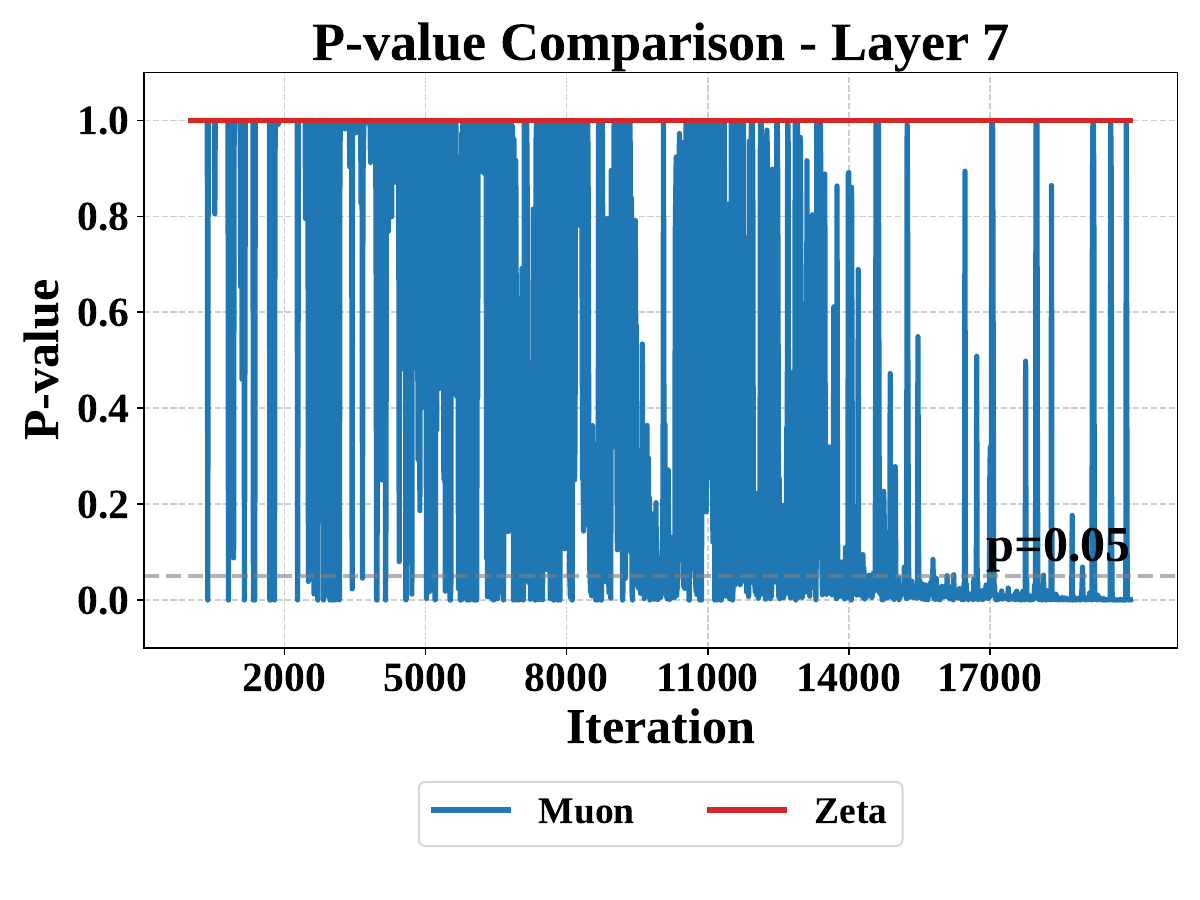}
    \includegraphics[width=0.32\textwidth]{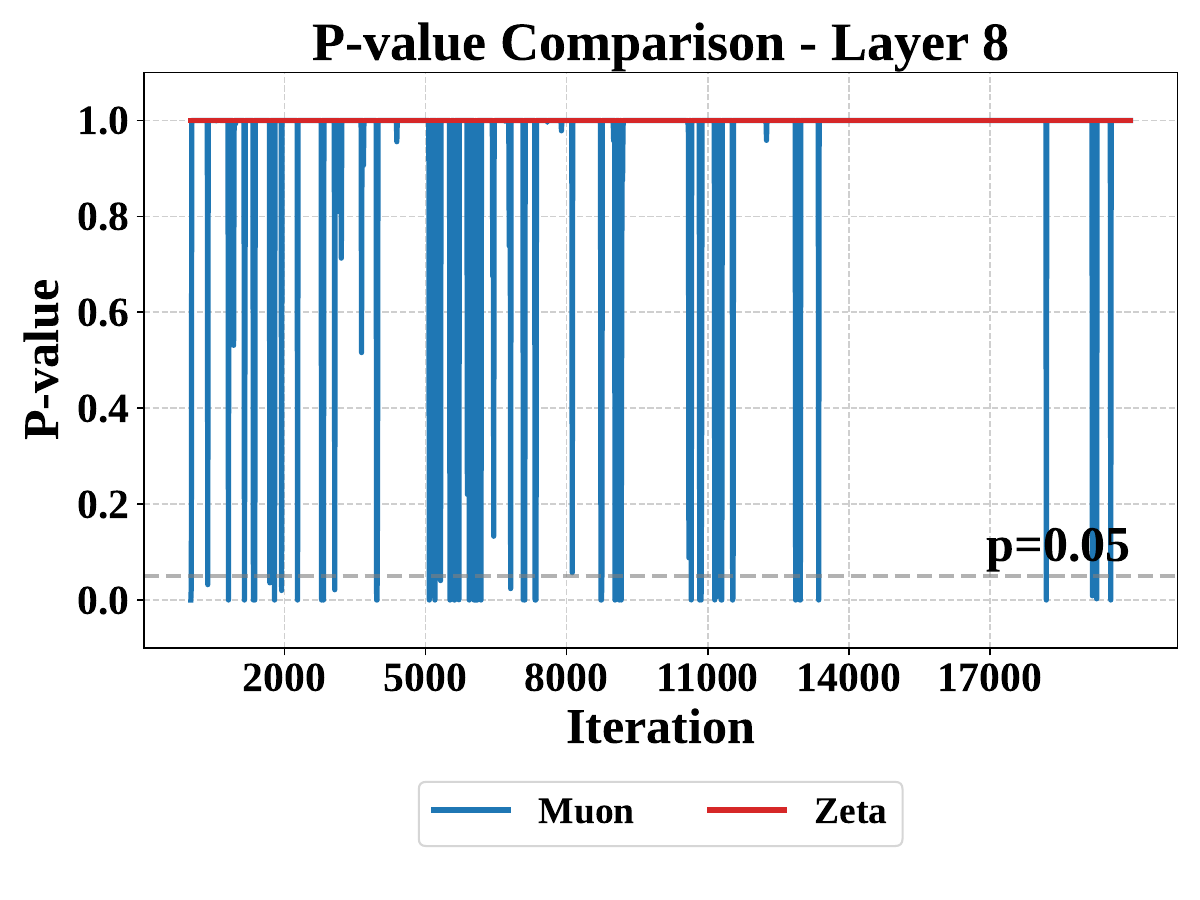}

    \caption{Layer-wise motivation-experiment results for layers 0--8. Each subplot reports the $p$-value statistics before and after coordinate whitening for the corresponding layer.}
    \label{fig:motivation_layers_part1}
\end{figure*}

\begin{figure*}[!h]
    \centering
    \includegraphics[width=0.32\textwidth]{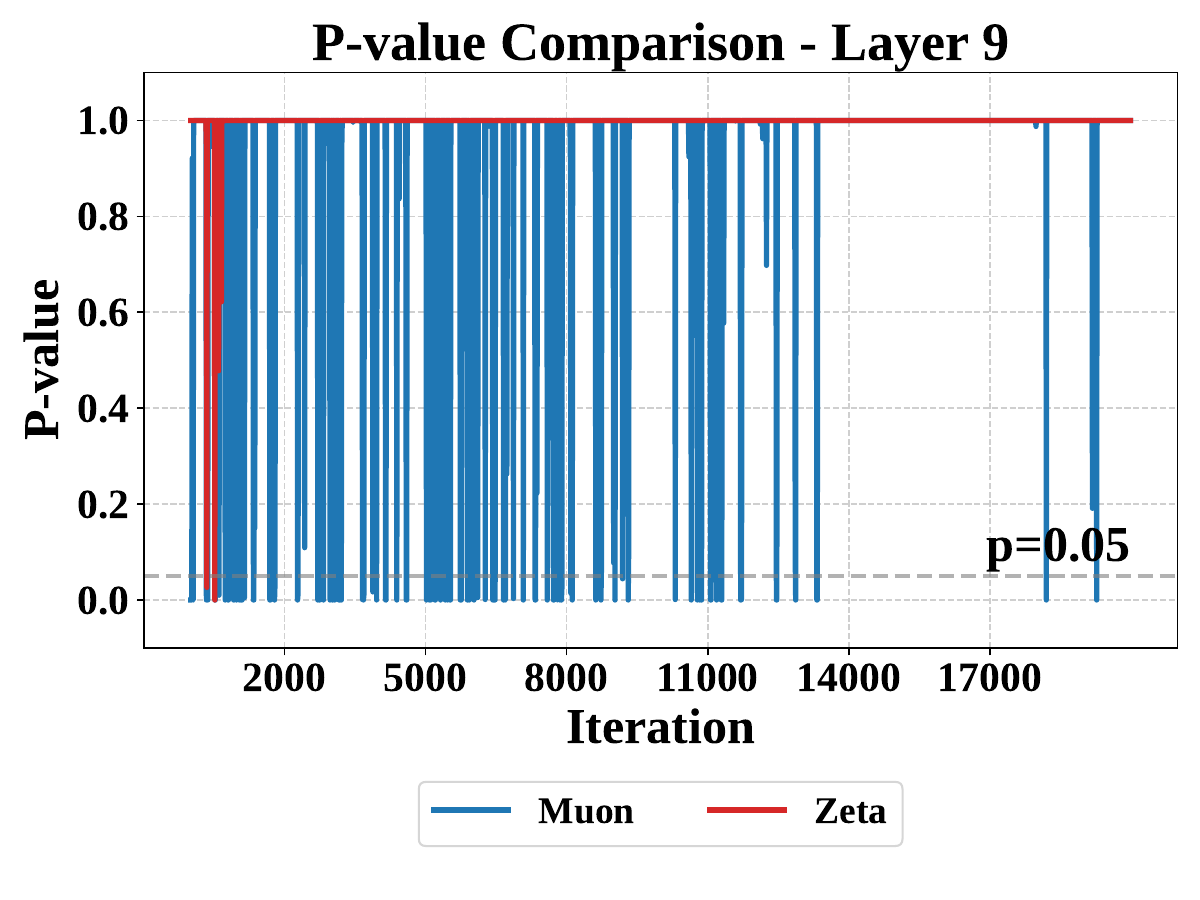}
    \includegraphics[width=0.32\textwidth]{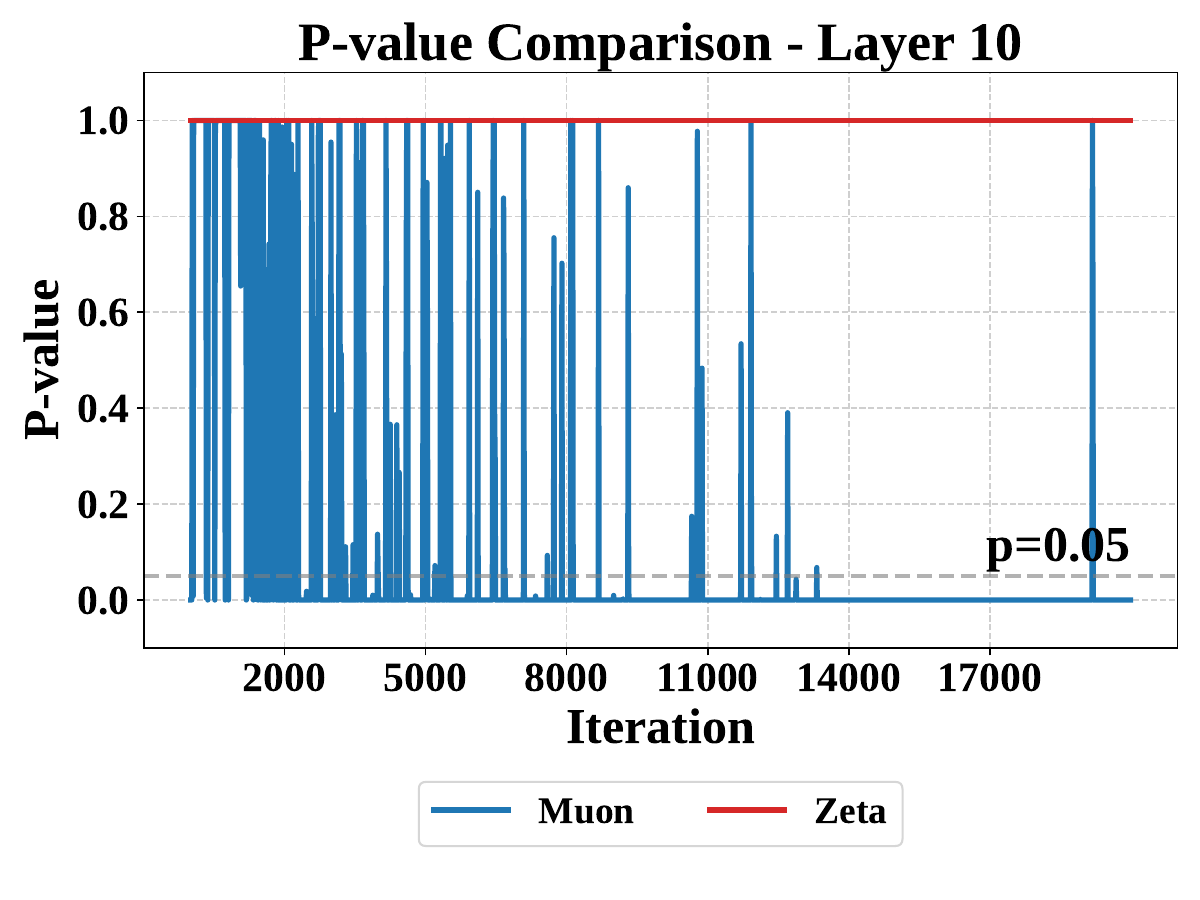}
    \includegraphics[width=0.32\textwidth]{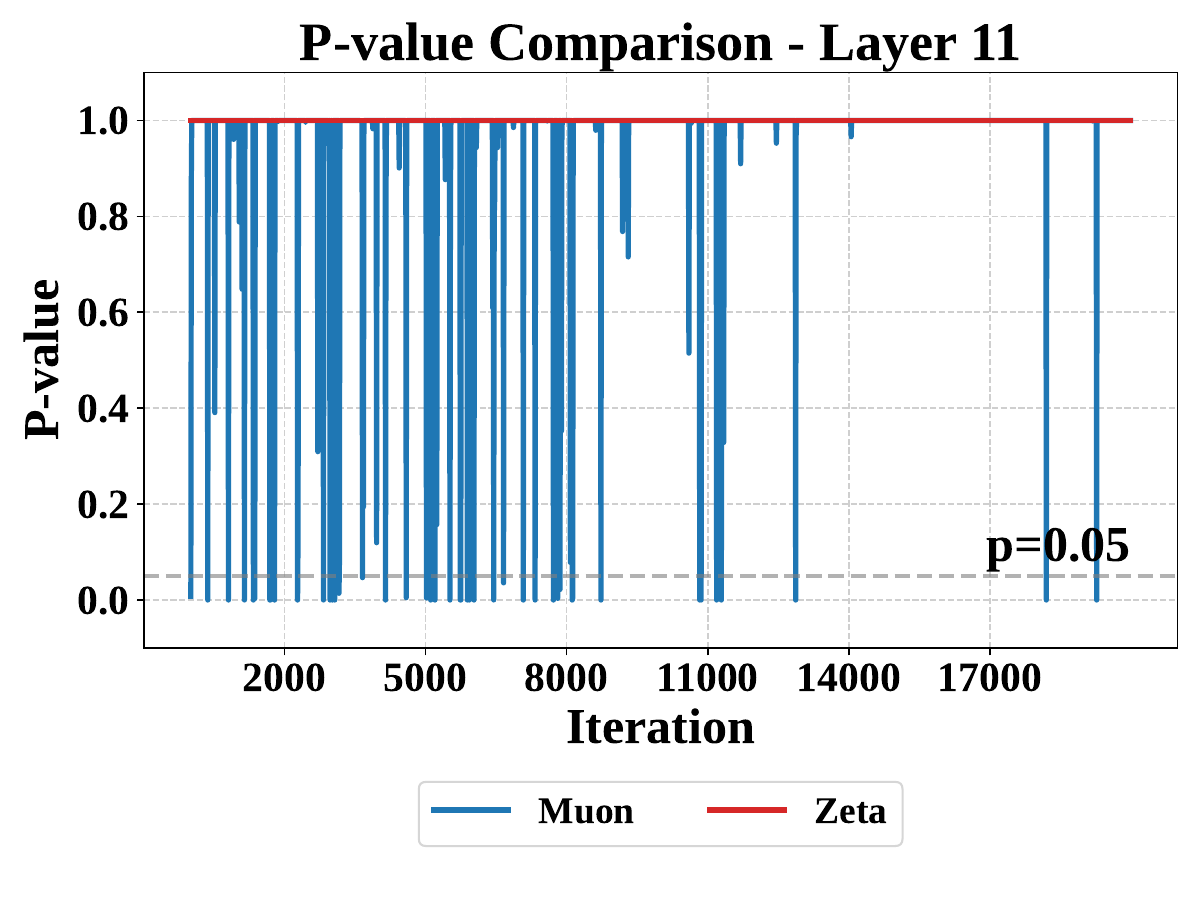}
    
    \includegraphics[width=0.32\textwidth]{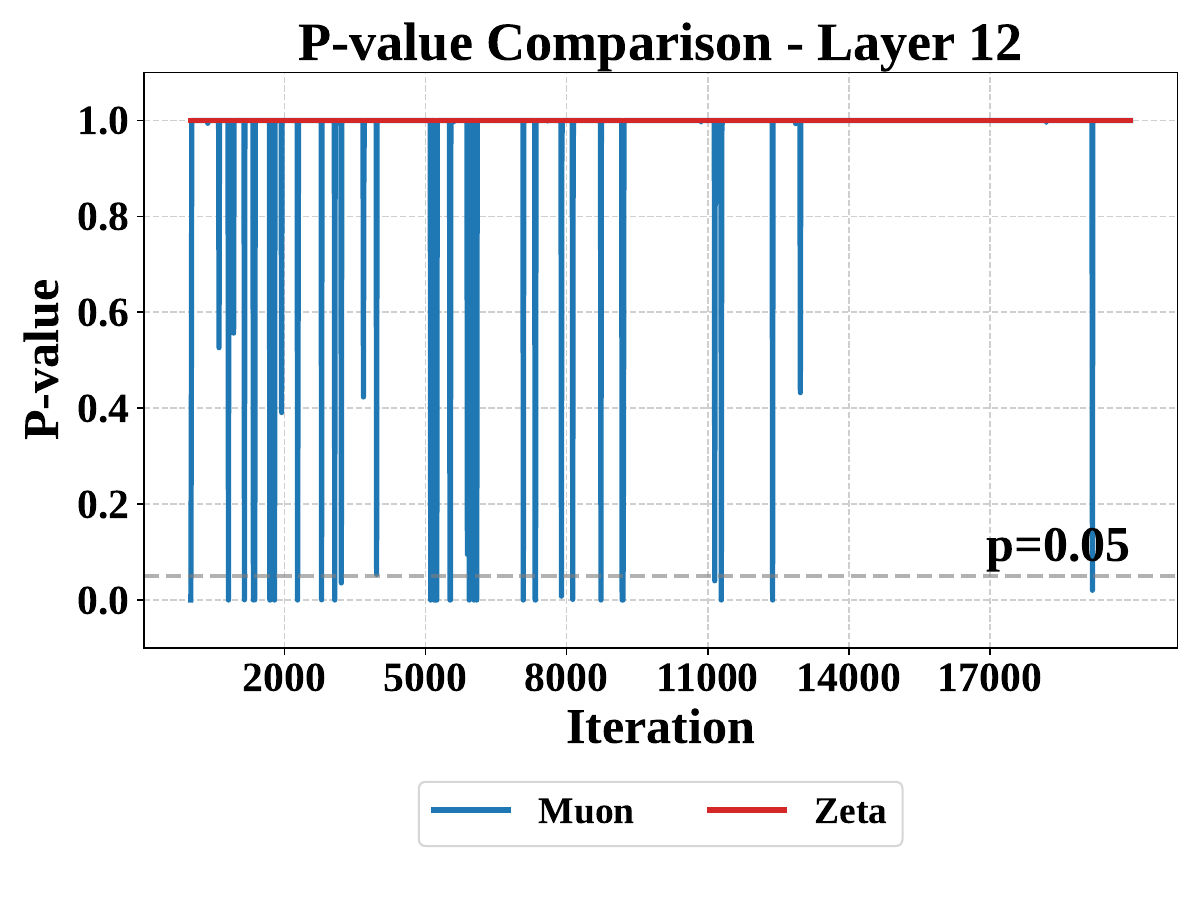}
    \includegraphics[width=0.32\textwidth]{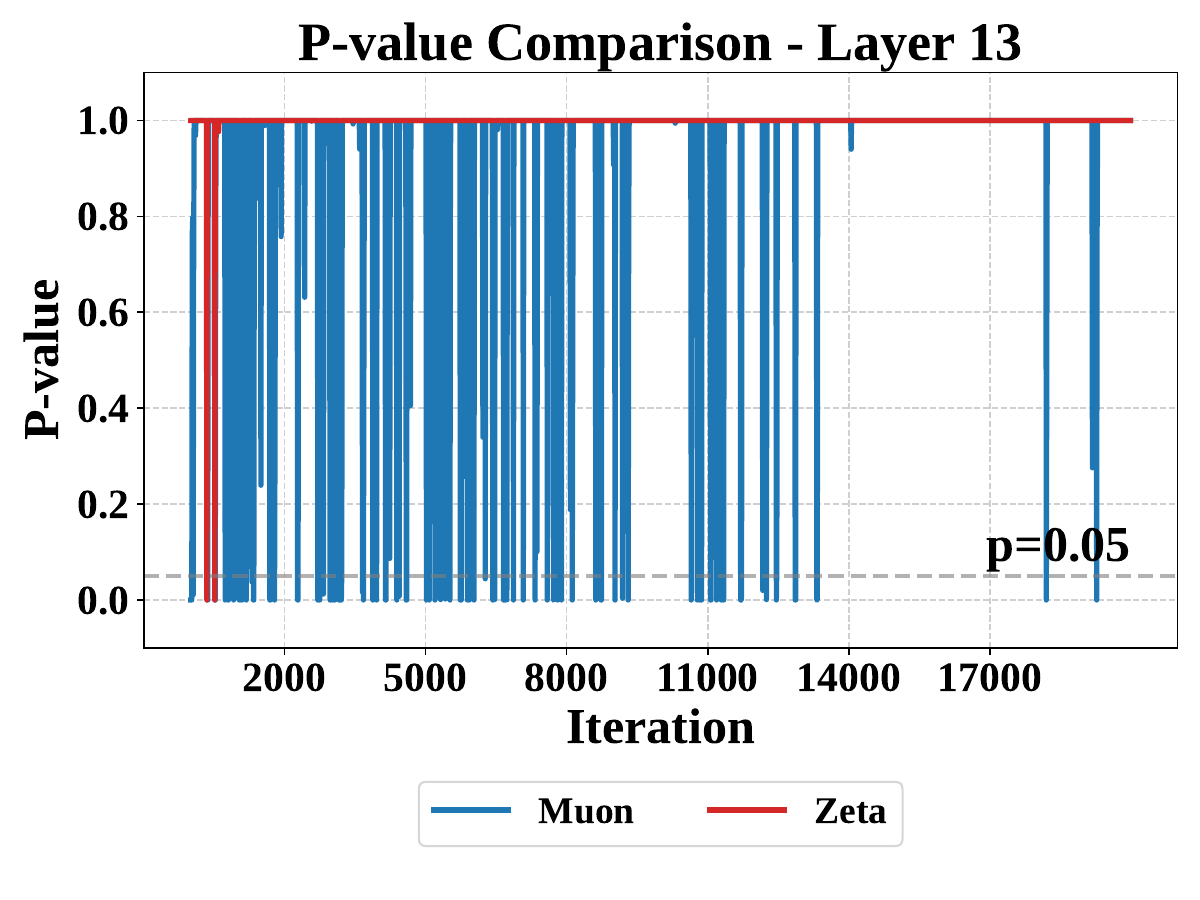}
    \includegraphics[width=0.32\textwidth]{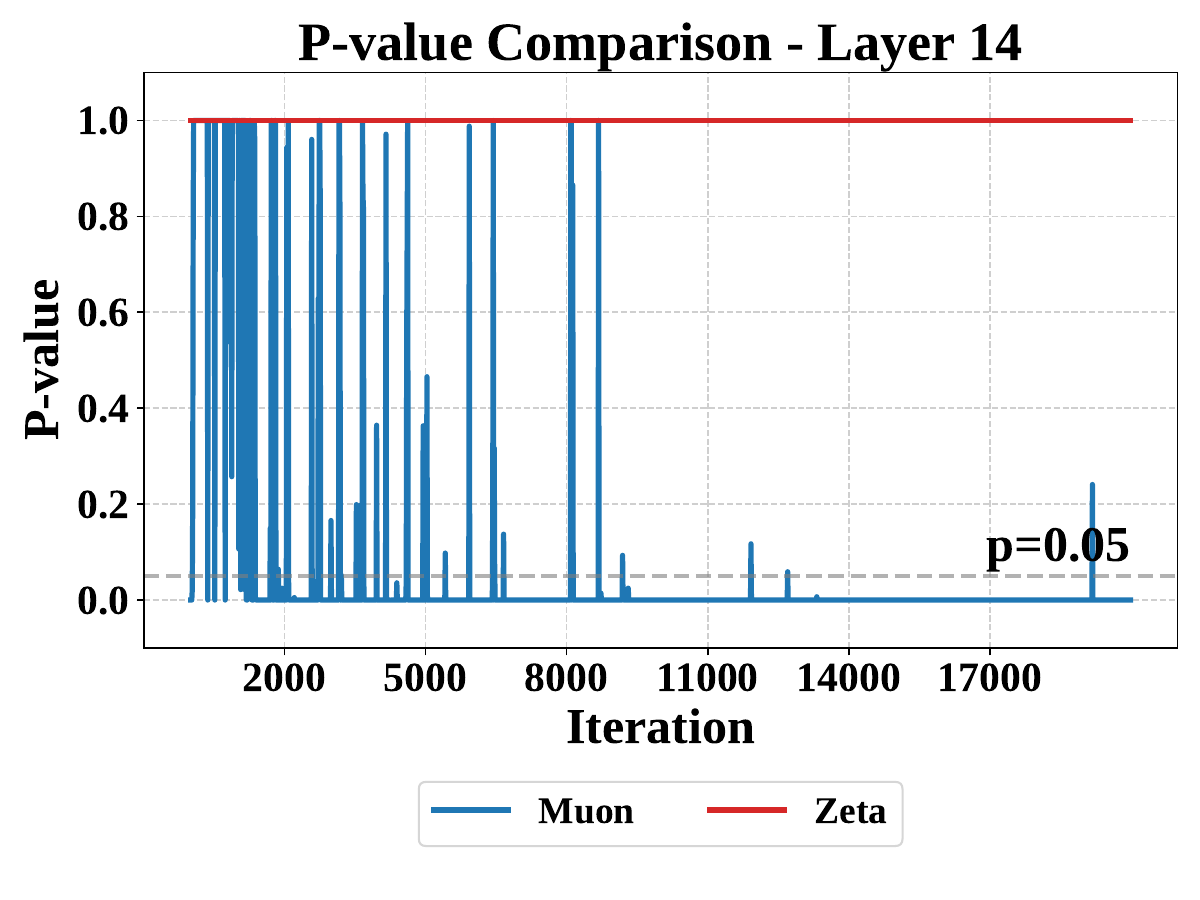}
    
    \includegraphics[width=0.32\textwidth]{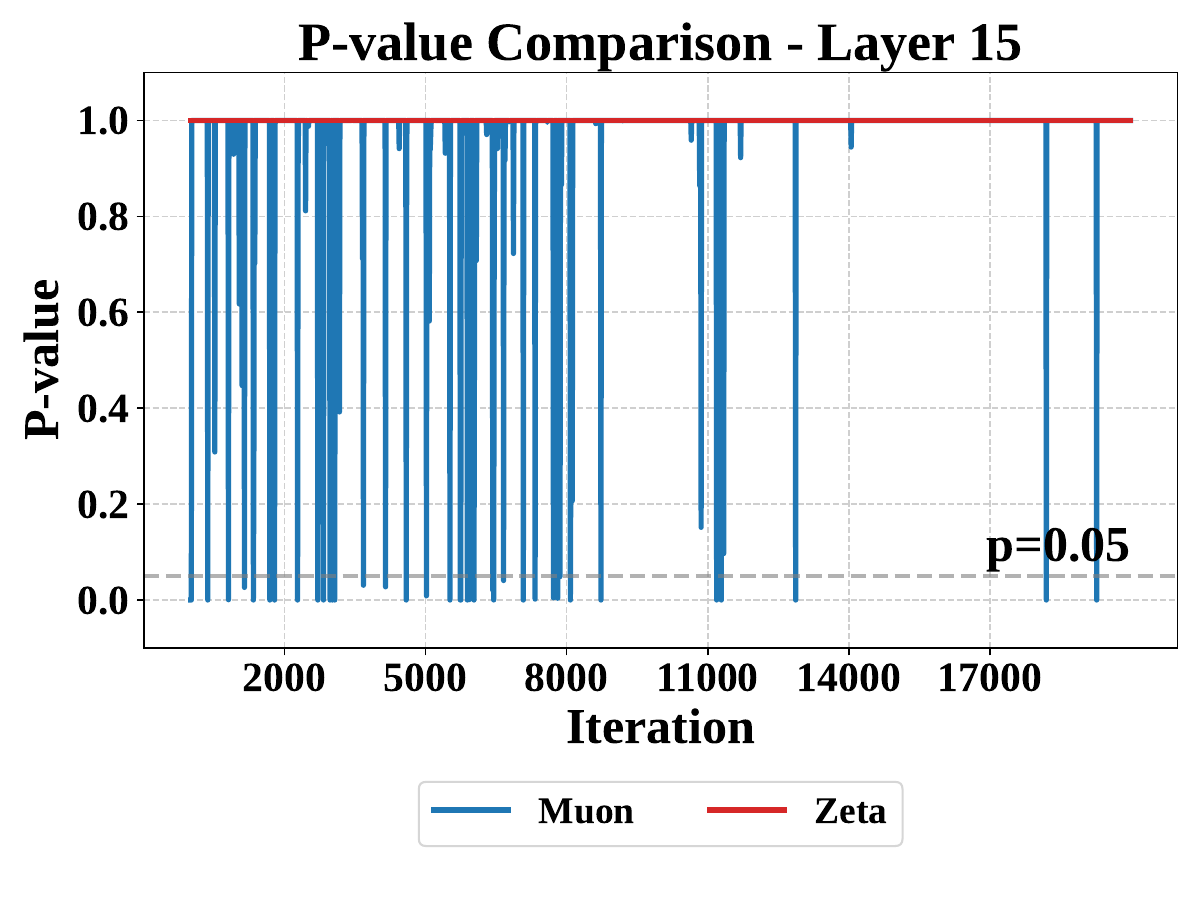}
    \includegraphics[width=0.32\textwidth]{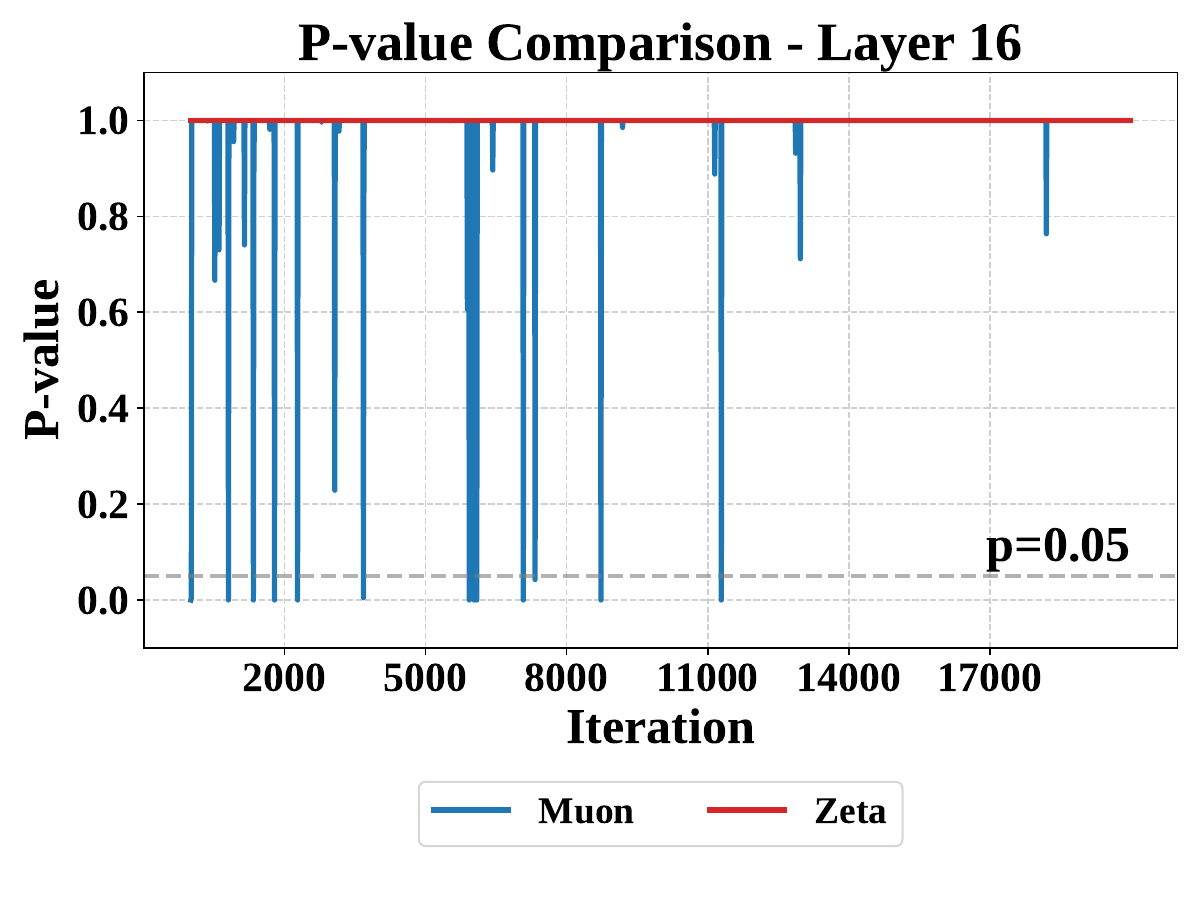}
    \includegraphics[width=0.32\textwidth]{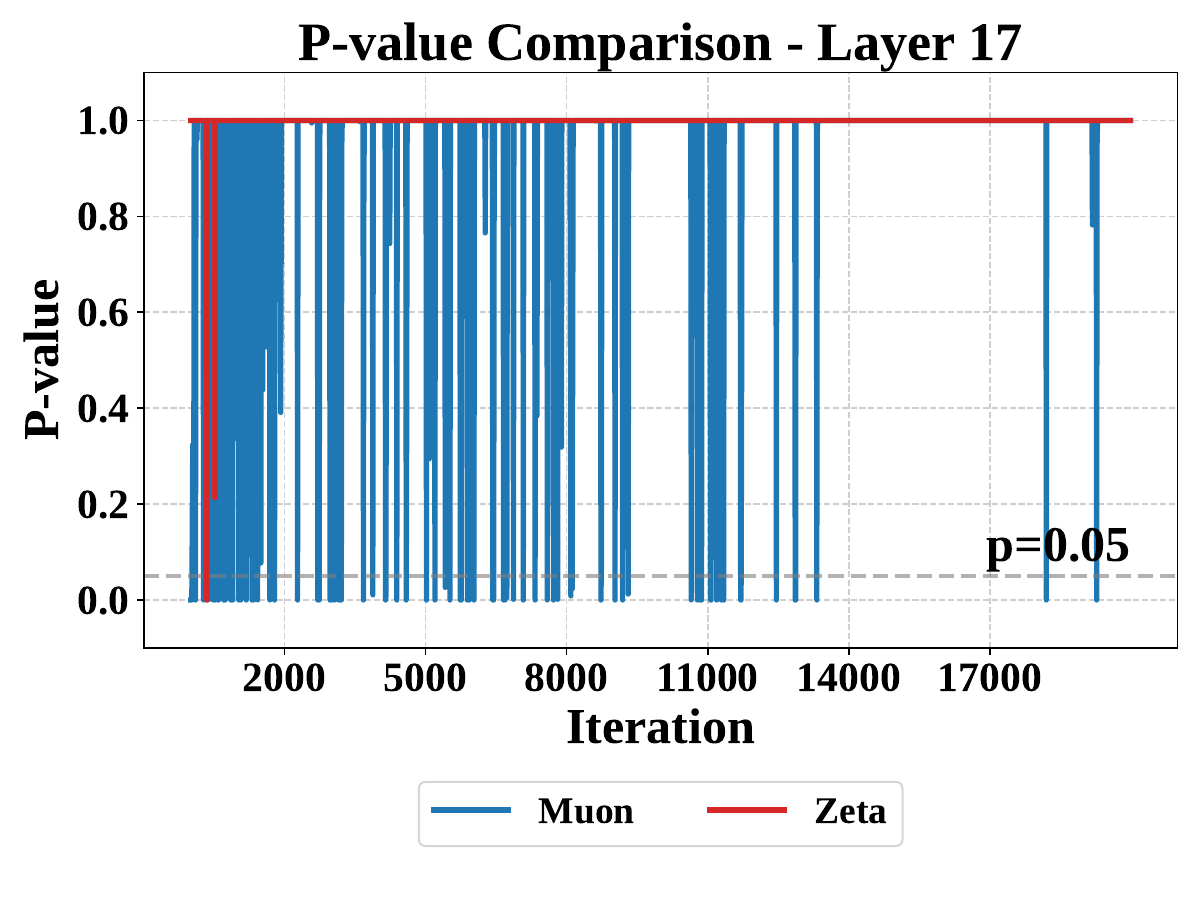}
    \caption{Layer-wise motivation-experiment results for layers 9--17. Each subplot reports the $p$-value statistics before and after coordinate whitening for the corresponding layer.}
    \label{fig:motivation_layers_part2}
\end{figure*}

\begin{figure*}[!h]
    \centering
    \includegraphics[width=0.32\textwidth]{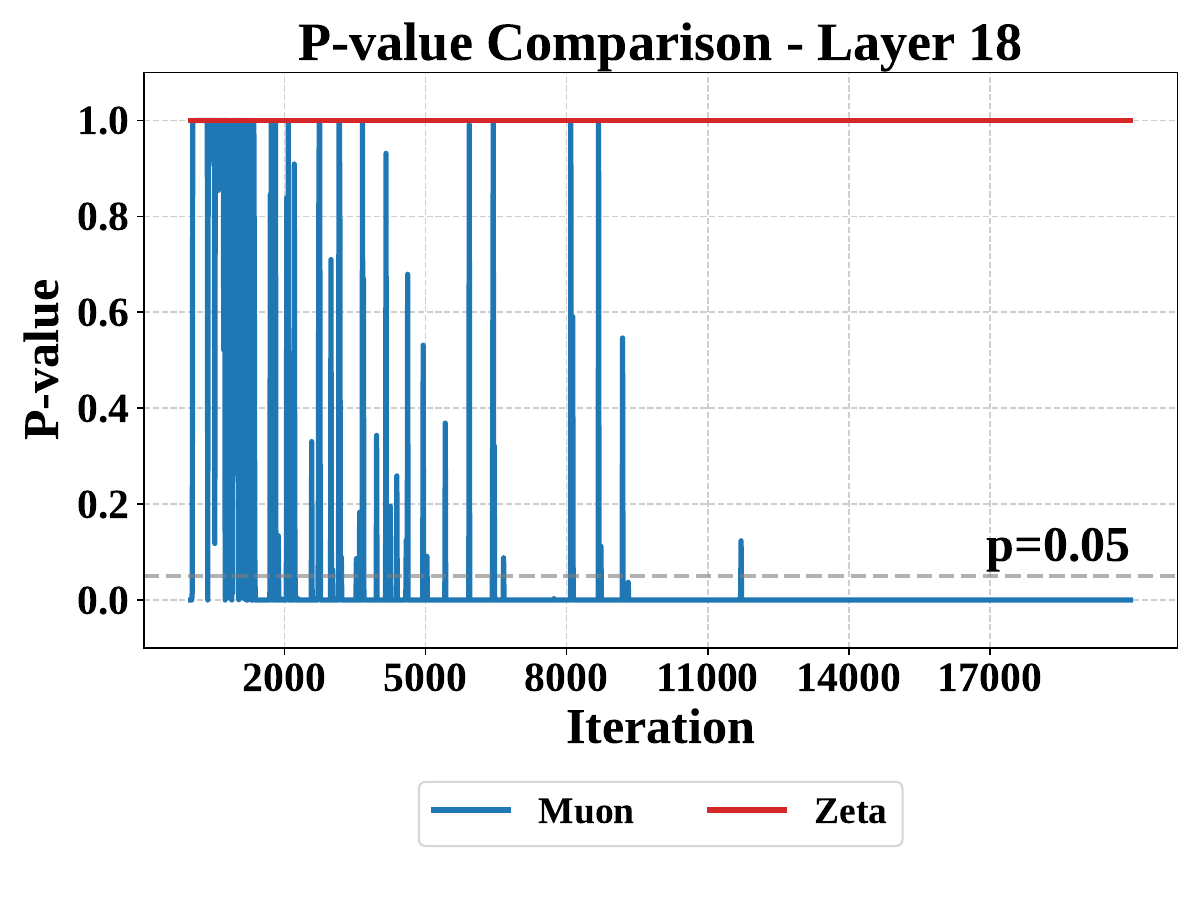}
    \includegraphics[width=0.32\textwidth]{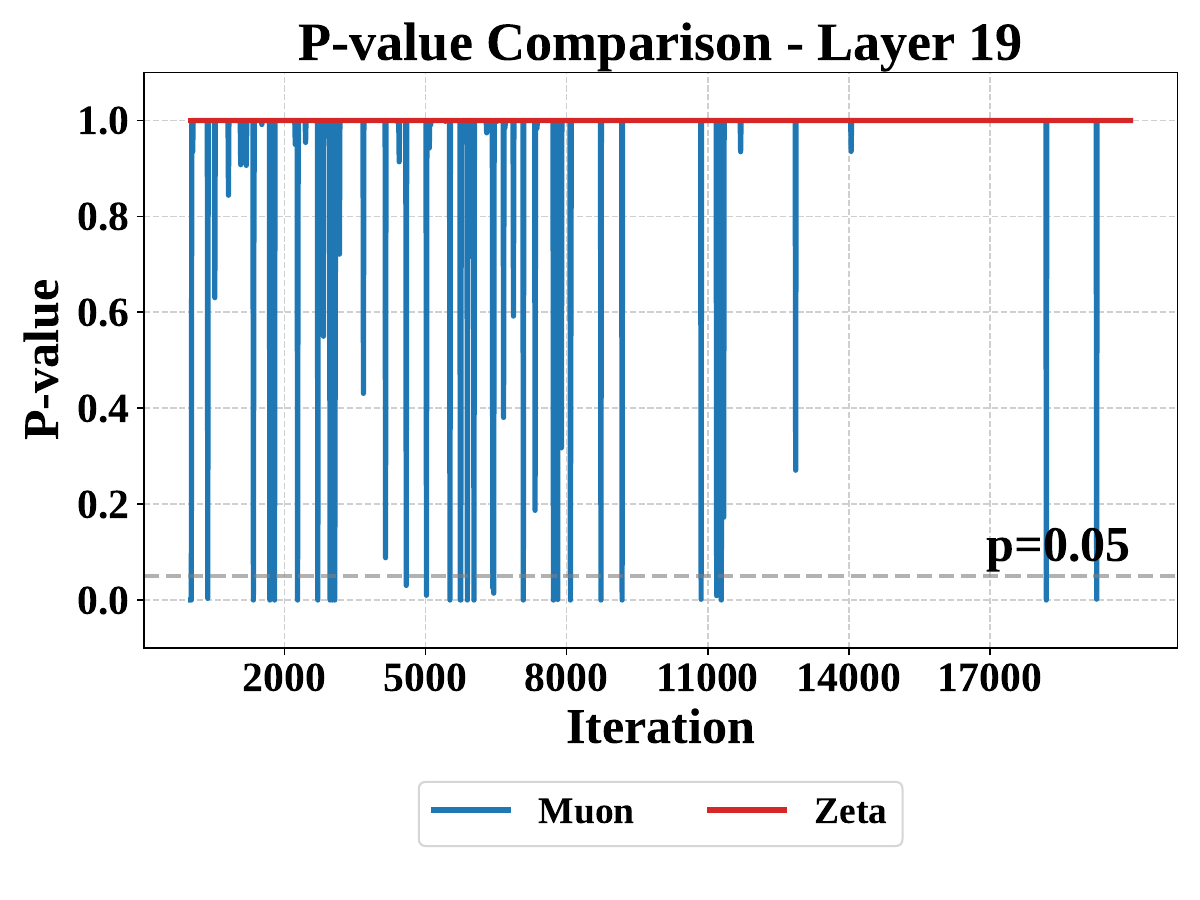}
    \includegraphics[width=0.32\textwidth]{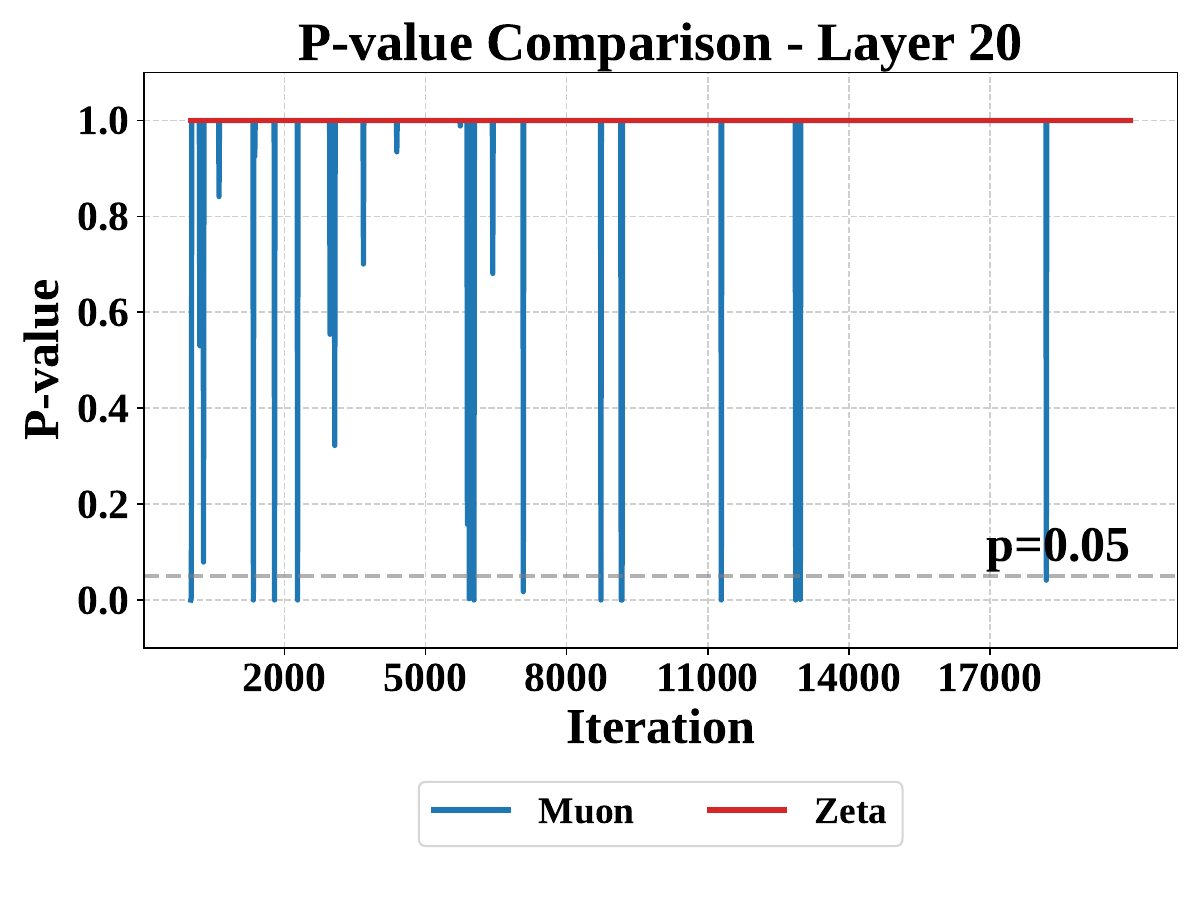}
    
    \includegraphics[width=0.32\textwidth]{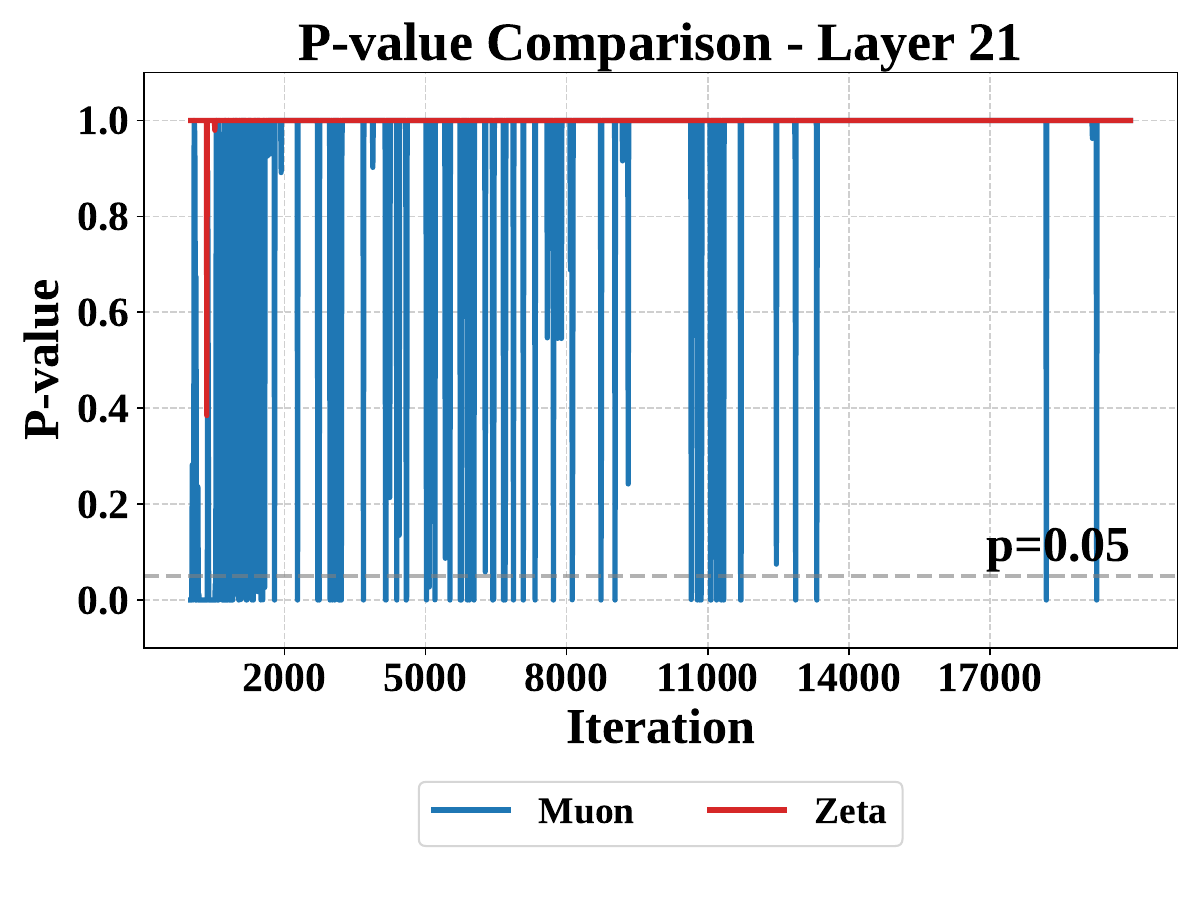}
    \includegraphics[width=0.32\textwidth]{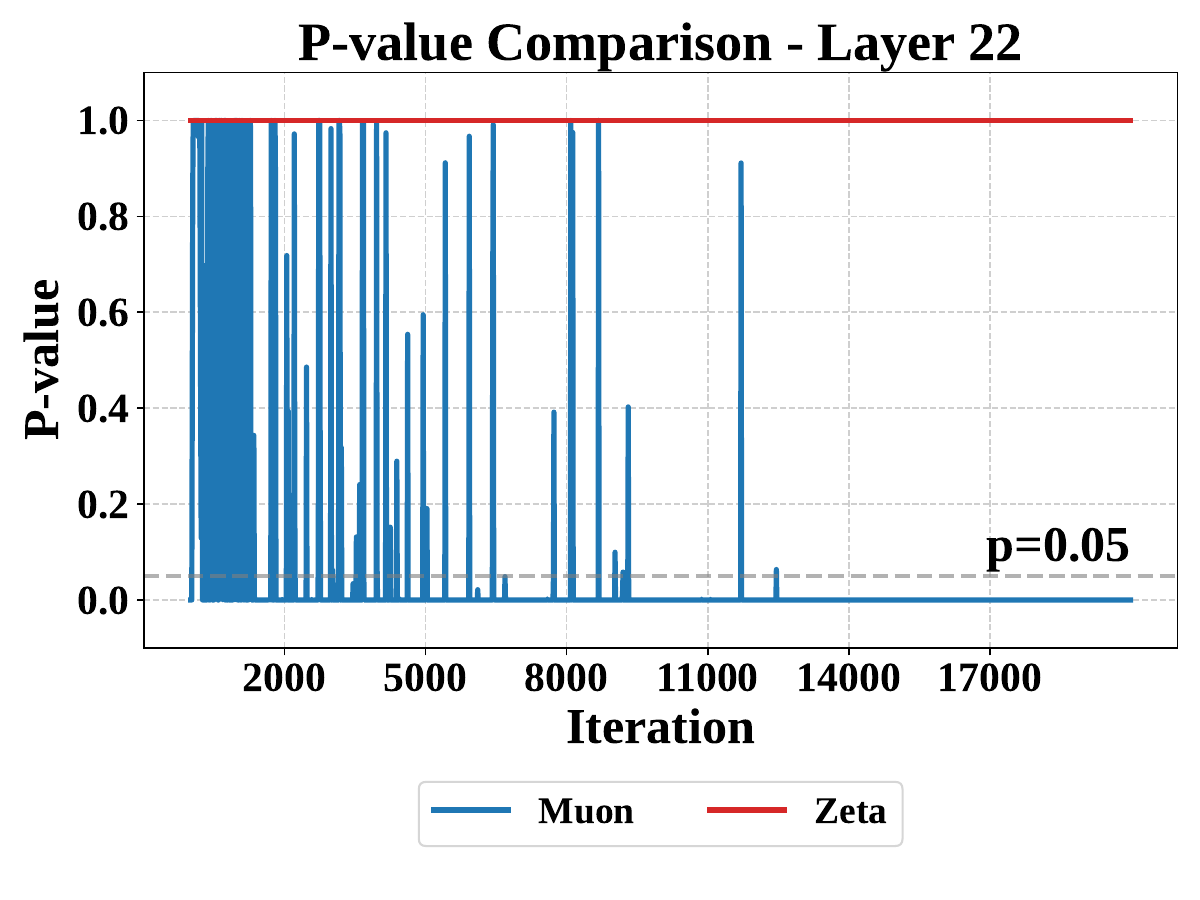}
    \includegraphics[width=0.32\textwidth]{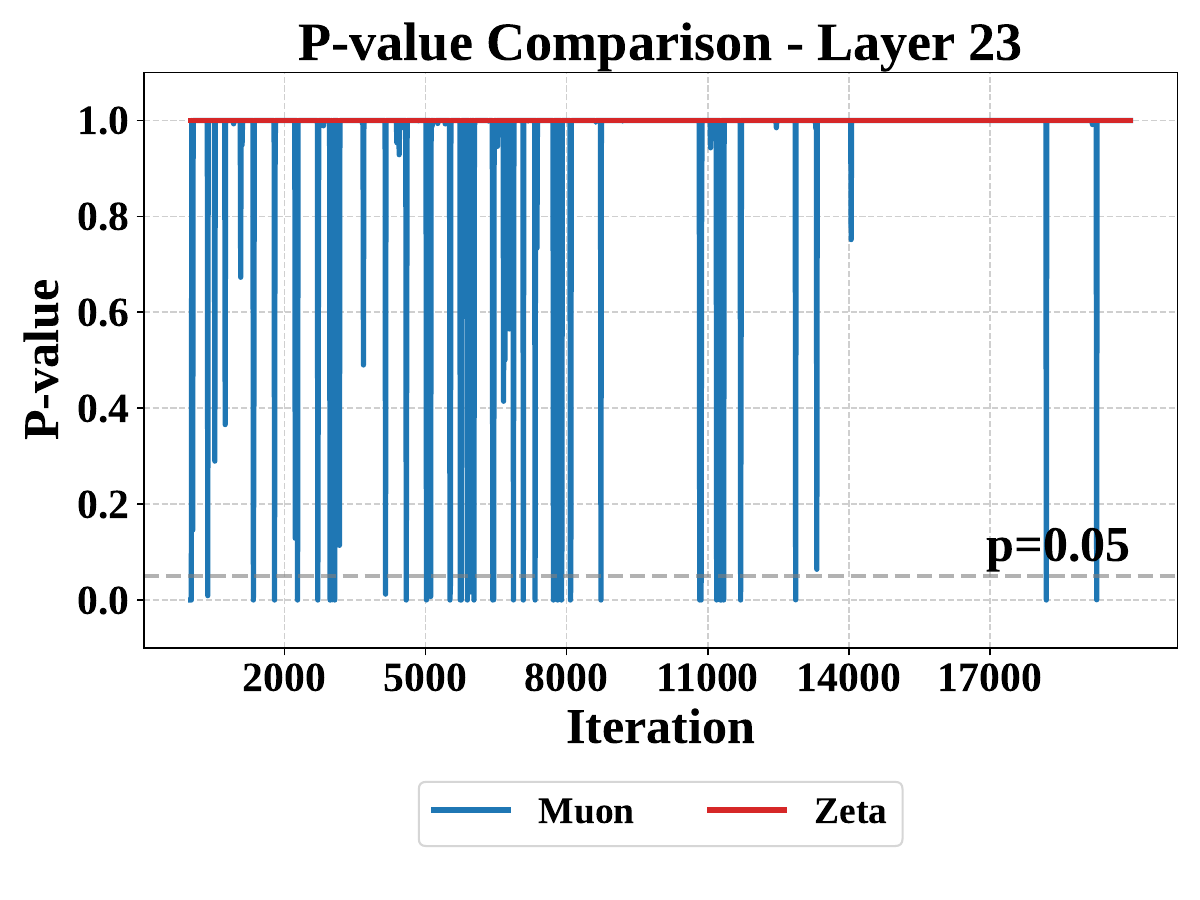}
    
    \includegraphics[width=0.32\textwidth]{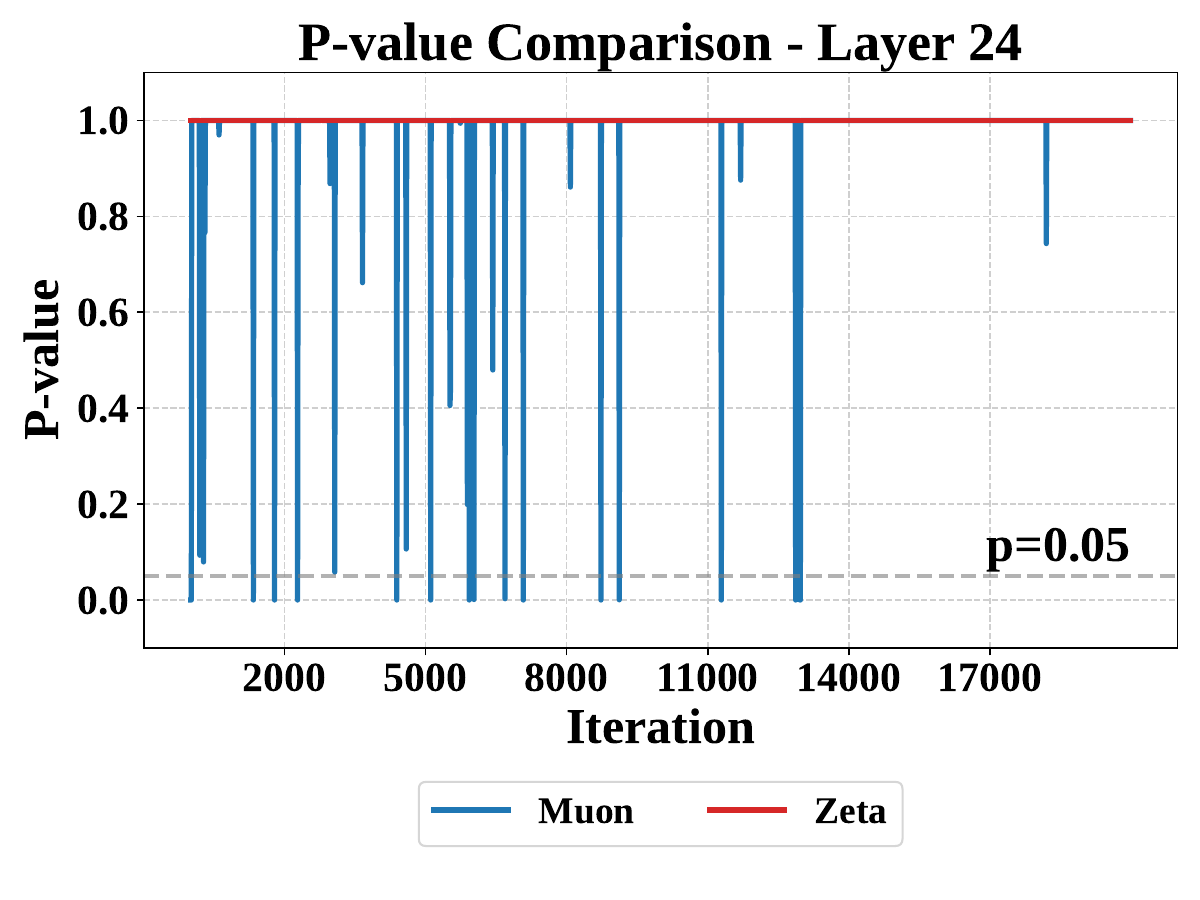}
    \includegraphics[width=0.32\textwidth]{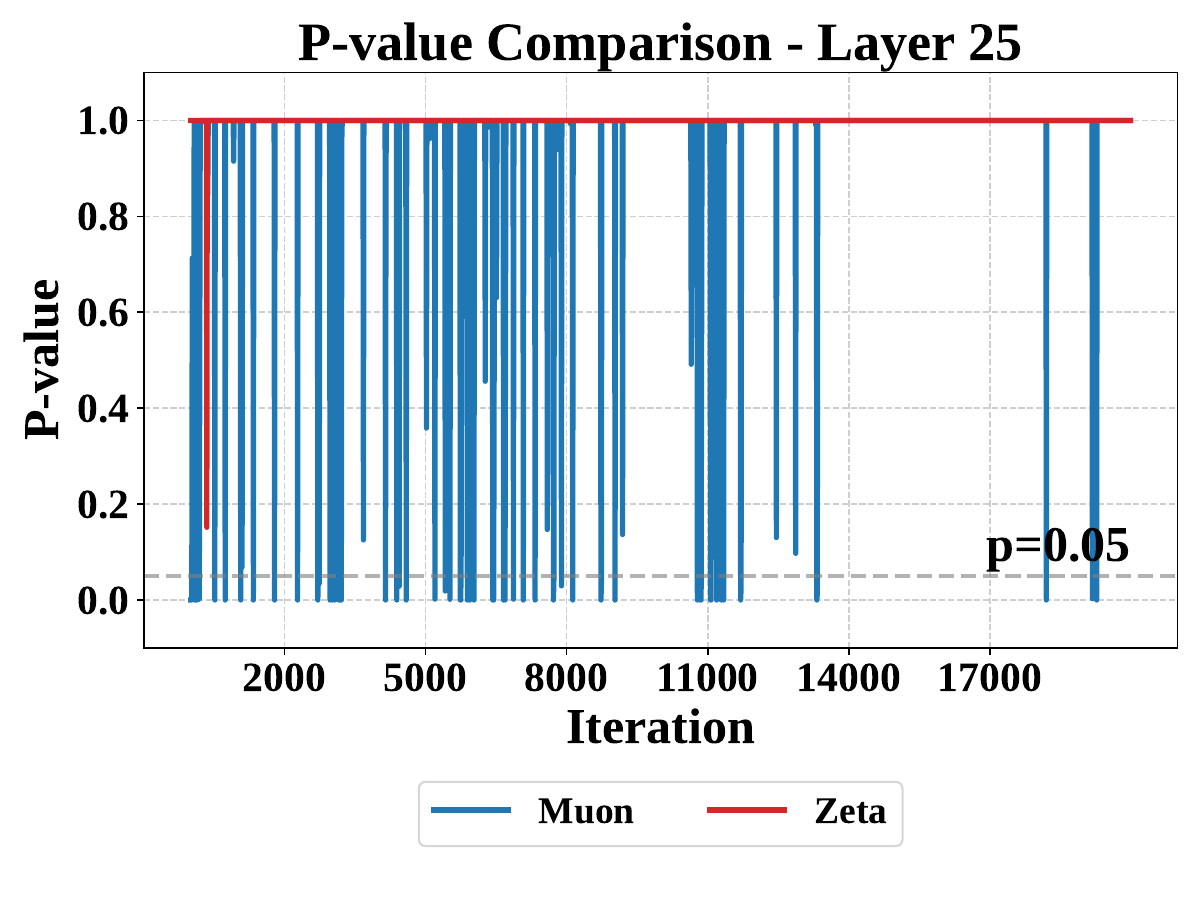}
    \includegraphics[width=0.32\textwidth]{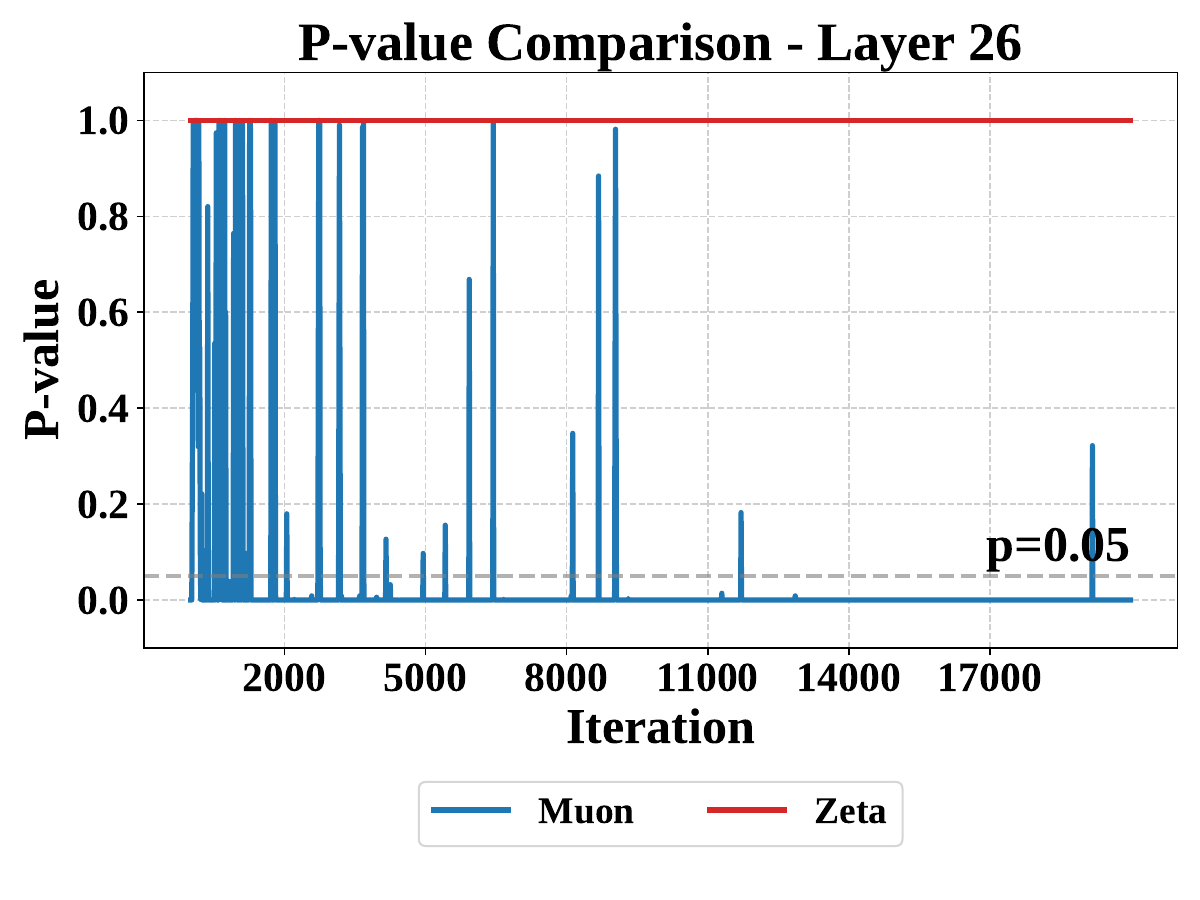}
    
    \includegraphics[width=0.32\textwidth]{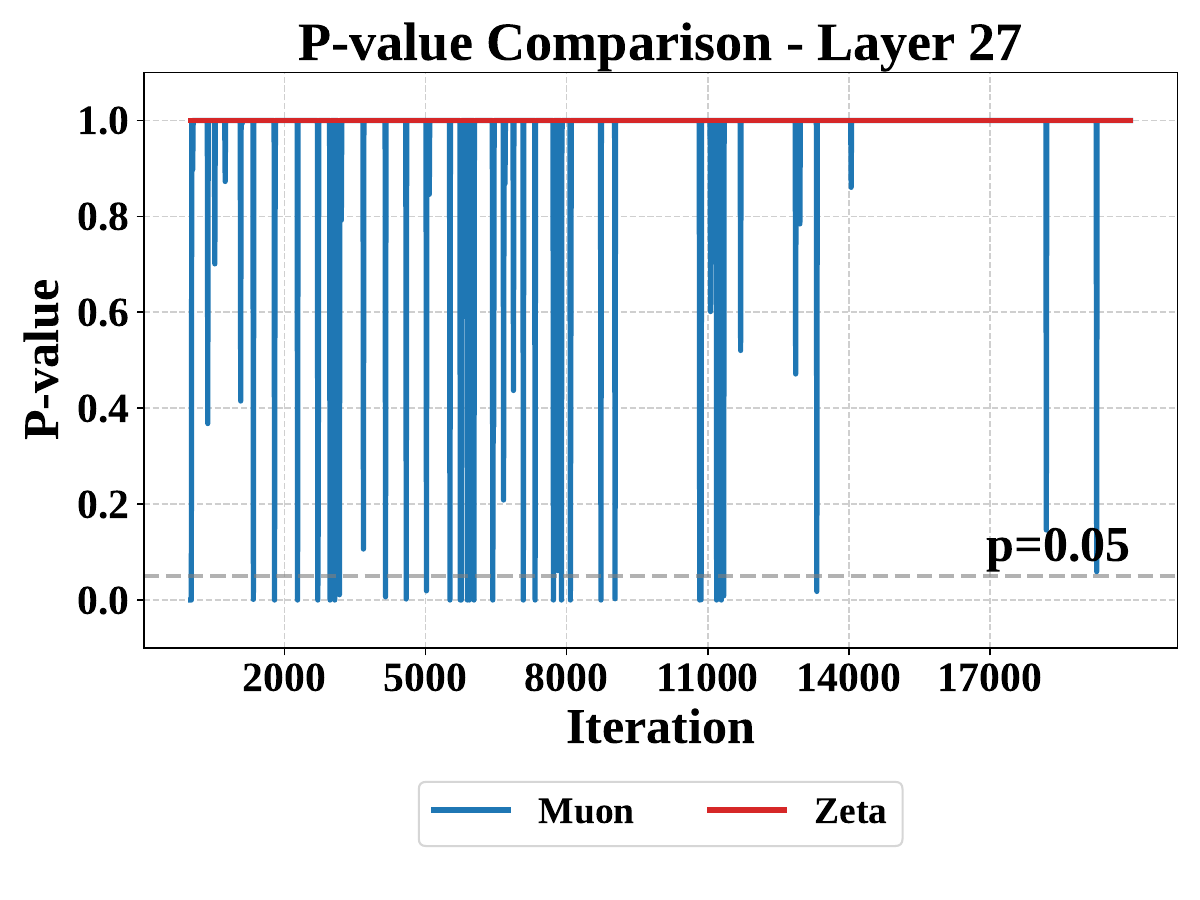}
    \includegraphics[width=0.32\textwidth]{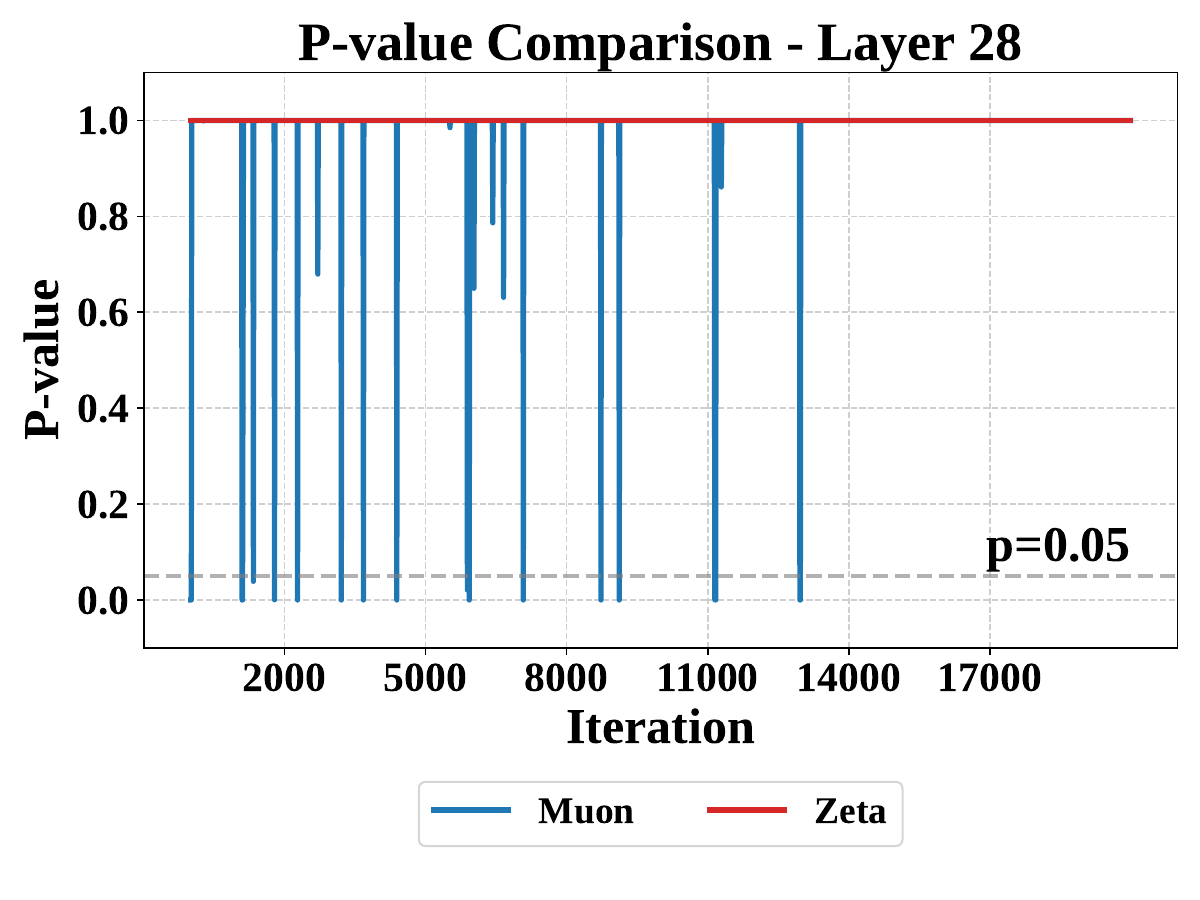}
    \caption{Layer-wise motivation-experiment results for layers 18--28. Each subplot reports the $p$-value statistics before and after coordinate whitening for the corresponding layer.}
    \label{fig:motivation_layers_part3}
\end{figure*}

\subsection{Learning-rate Sensitivity Results on Qwen3-0.6B}

In this subsection, we supplement the learning-rate tuning results for Qwen3-0.6B. Table~\ref{tab:qwen06b_lr_sensitivity} reports the final training loss under different learning rates, and Figure~\ref{fig:qwen06b_lr_sensitivity_curves} shows the corresponding training curves. Together, these results allow us to assess whether Zeta preserves its optimization advantage over a broader range of step sizes.

\begin{table}[!h]
    \centering
    \caption{Training loss on Qwen3-0.6B under different learning rates. Lower is better.}
    \label{tab:qwen06b_lr_sensitivity}
    \begin{tabular}{lcccc}
        \toprule
        Optimizer & $3\times10^{-4}$ & $9\times10^{-4}$ & $3\times10^{-3}$ & $9\times10^{-3}$ \\
        \midrule
        AdamW & 2.698 & 2.678 & 2.618 & 2.666 \\
        Muon & 2.641 & 2.598 & 2.610 & 2.644 \\
        AdaMuon & 2.652 & 2.598 & 2.580 & 2.589 \\
        Zeta & 2.626 & 2.583 & \textbf{2.564} & 2.590 \\
        \bottomrule
    \end{tabular}
\end{table}

\begin{figure}[!h]
    \centering
    \begin{minipage}{0.48\textwidth}
        \centering
        \includegraphics[width=\textwidth]{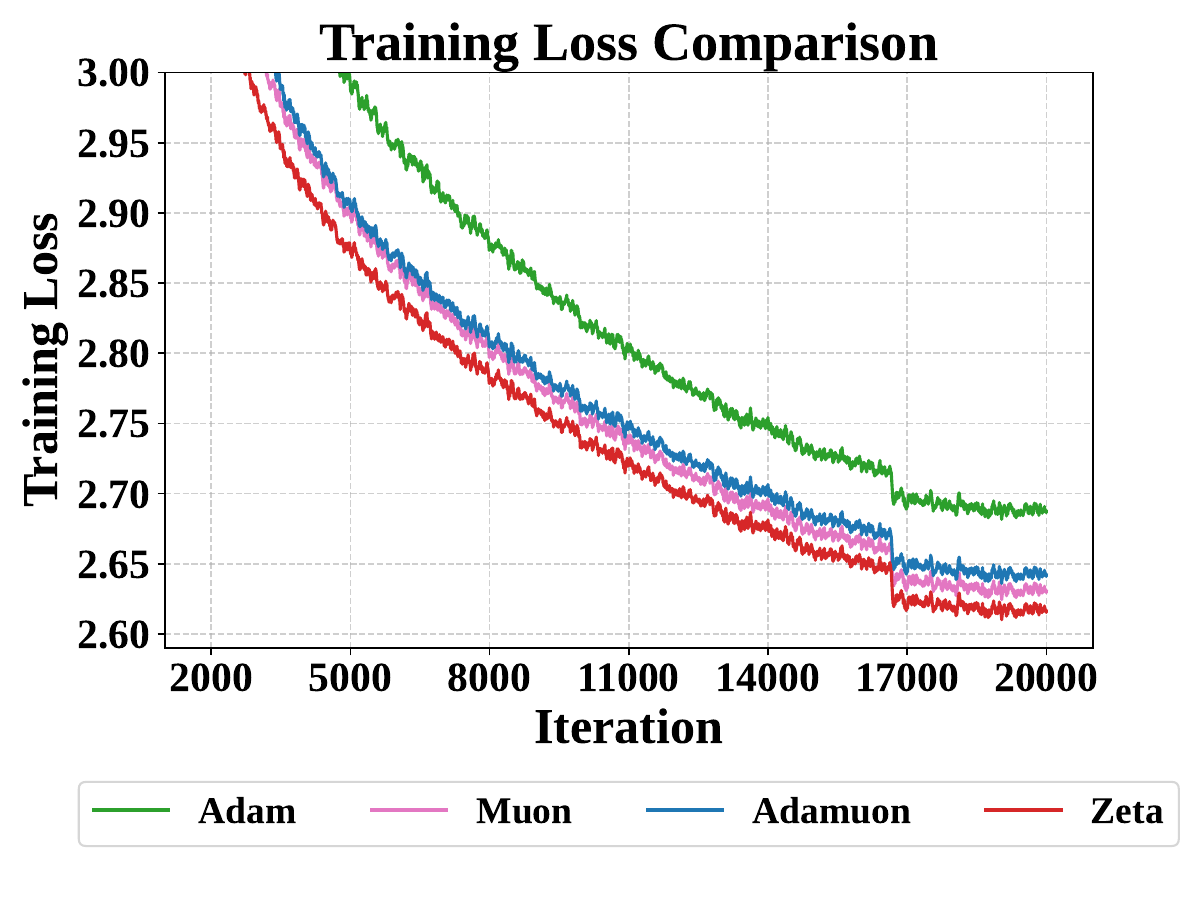}
    \end{minipage}
    \hfill
    \begin{minipage}{0.48\textwidth}
        \centering
        \includegraphics[width=\textwidth]{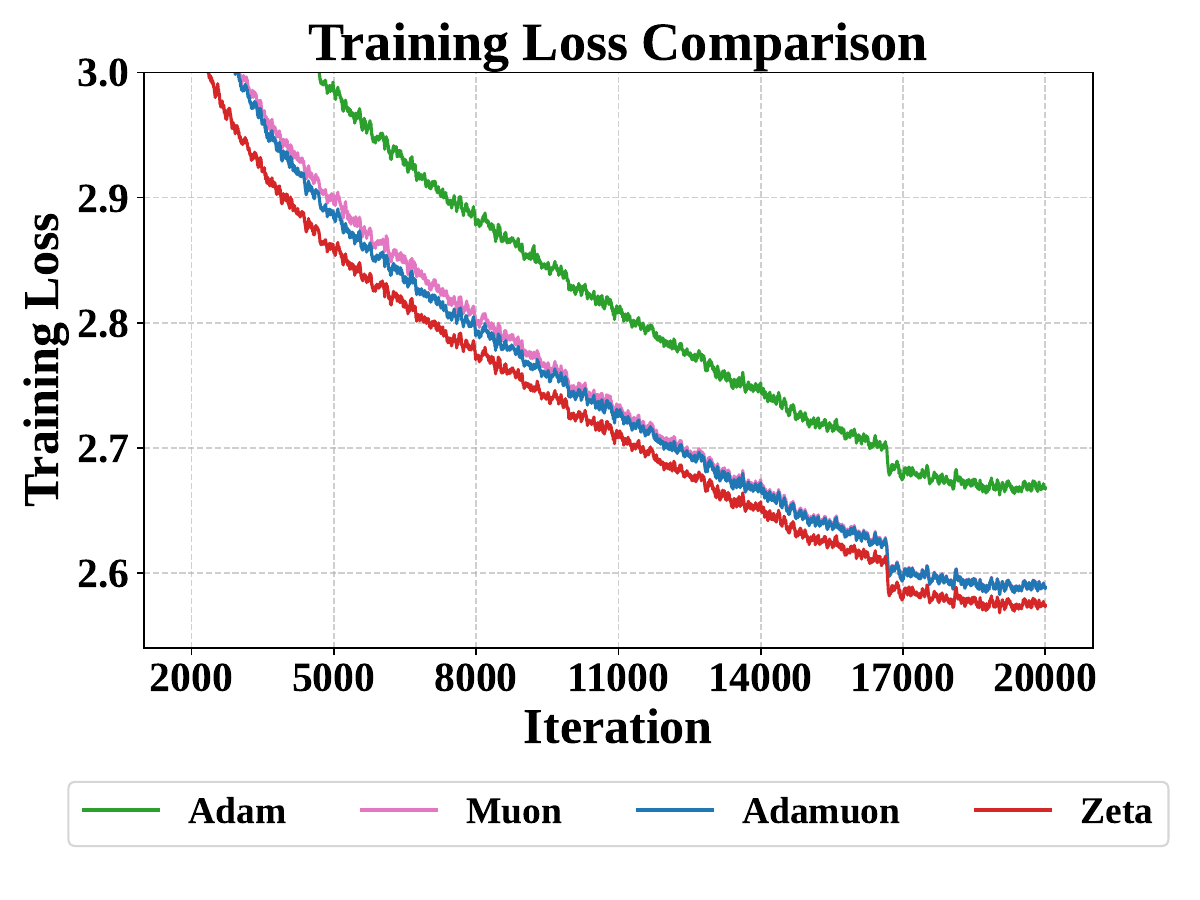}
    \end{minipage}

    \vspace{1ex}

    \begin{minipage}{0.48\textwidth}
        \centering
        \includegraphics[width=\textwidth]{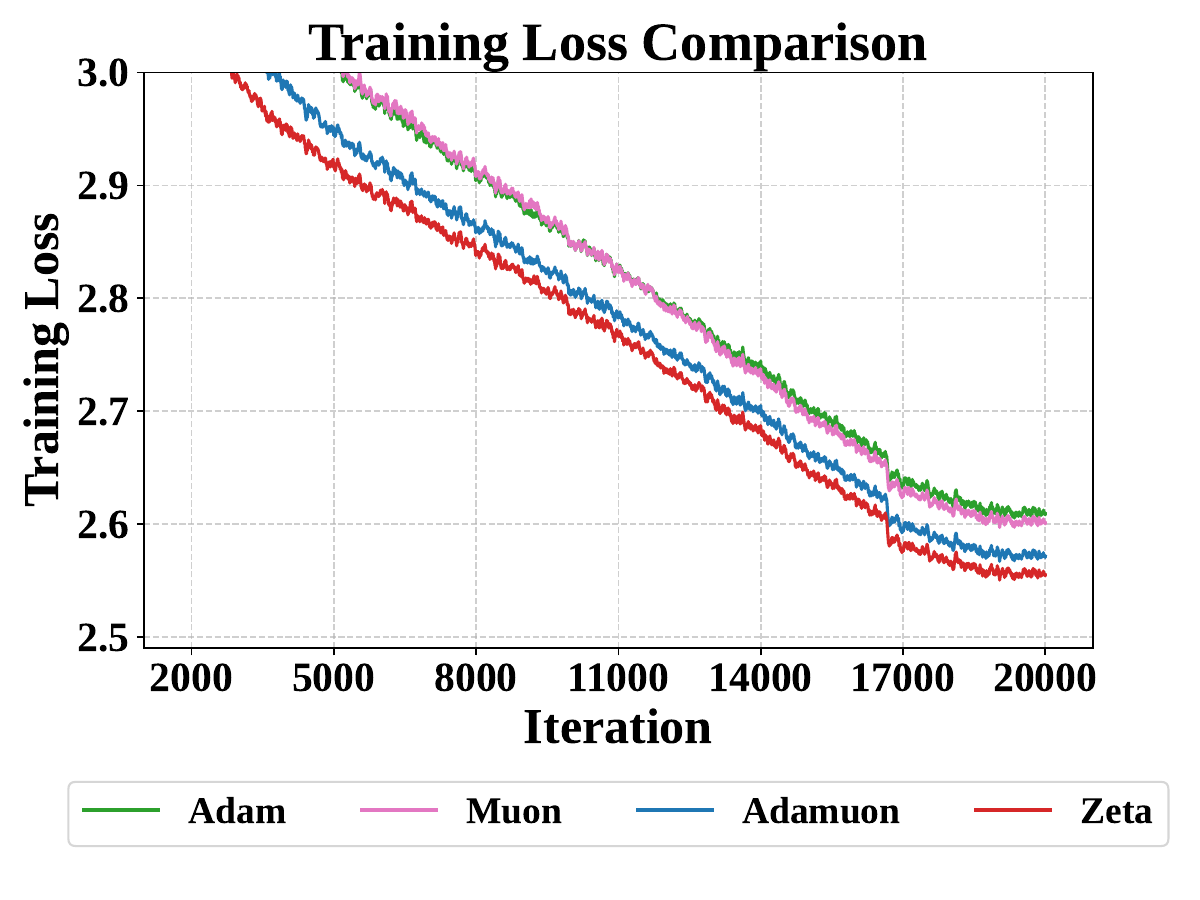}
    \end{minipage}
    \hfill
    \begin{minipage}{0.48\textwidth}
        \centering
        \includegraphics[width=\textwidth]{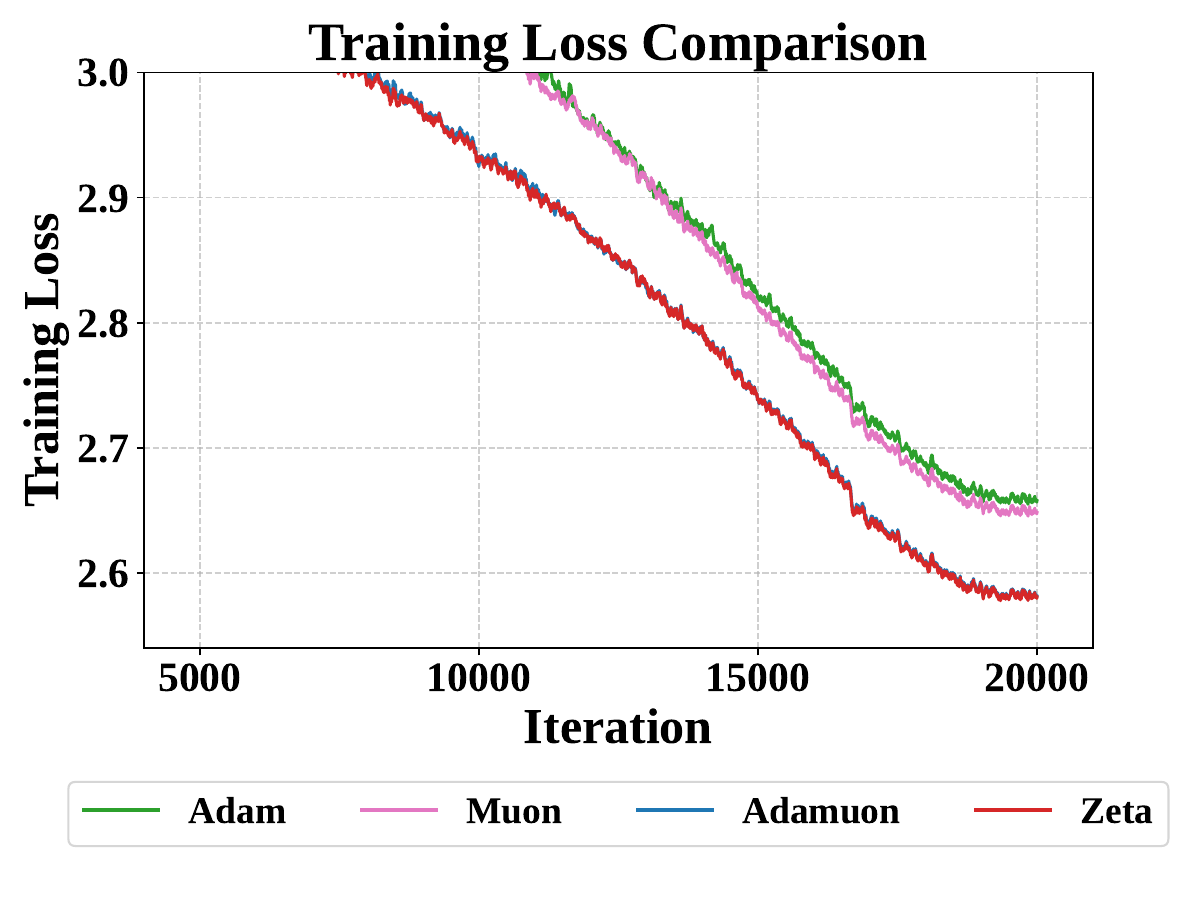}
    \end{minipage}

    \vspace{1ex}
    \caption{Training loss curves on Qwen3-0.6B under four learning rates: $3\times10^{-4}$ (top left), $9\times10^{-4}$ (top right), $3\times10^{-3}$ (bottom left), and $9\times10^{-3}$ (bottom right).}
    \label{fig:qwen06b_lr_sensitivity_curves}
\end{figure}

\subsection{Learning-rate Sensitivity Results on Vision Tasks}

In this subsection, we supplement the learning-rate tuning results on the vision tasks. Tables~\ref{tab:vision_lr_adamw}--\ref{tab:vision_lr_zeta} report the validation accuracy of AdamW, Muon, AdaMuon, and Zeta on both ViT-Tiny and ViT-Base under their separately tuned learning rates. Because the effective learning-rate ranges differ substantially across optimizers, we present one table per optimizer for a clearer comparison.

\begin{table}[!h]
    \centering
    \caption{Vision learning-rate sensitivity for AdamW. Top-1 validation accuracy (\%) on CIFAR-100.}
    \label{tab:vision_lr_adamw}
    \small
    \begin{tabular}{lcccc}
        \toprule
        \multicolumn{5}{c}{\textbf{ViT-Tiny}} \\
        \midrule
        LR & 0.009 & 0.0010 & 0.0020 &  \\
        Acc. & 56.42 & \textbf{57.02} & 54.92 &  \\
        \midrule
        \multicolumn{5}{c}{\textbf{ViT-Base}} \\
        \midrule
        LR & 0.0001 & 0.0002 & 0.0003 & 0.0004 \\
        Acc. & 57.05 & \textbf{57.98} & 52.10 & 42.58 \\
        \bottomrule
    \end{tabular}
\end{table}

\begin{table}[!h]
    \centering
    \caption{Vision learning-rate sensitivity for Muon. Top-1 validation accuracy (\%) on CIFAR-100.}
    \label{tab:vision_lr_muon}
    \small
    \begin{tabular}{lccccc}
        \toprule
        \multicolumn{6}{c}{\textbf{ViT-Tiny}} \\
        \midrule
        LR & 0.0020 & 0.0030 & 0.0040 & 0.0050 & 0.0060 \\
        Acc. & 63.65 & 63.82 & 63.83 & \textbf{64.68} & 64.15 \\
        \midrule
        \multicolumn{6}{c}{\textbf{ViT-Base}} \\
        \midrule
        LR & 0.0008 & 0.0009 & 0.0010 & 0.0020 &  \\
        Acc. & 64.10 & \textbf{64.53} & 63.69 & 62.95 &  \\
        \bottomrule
    \end{tabular}
\end{table}

\begin{table}[!h]
    \centering
    \caption{Vision learning-rate sensitivity for AdaMuon. Top-1 validation accuracy (\%) on CIFAR-100.}
    \label{tab:vision_lr_adamuon}
    \small
    \begin{tabular}{lccccccccc}
        \toprule
        \multicolumn{10}{c}{\textbf{ViT-Tiny}} \\
        \midrule
        LR & 0.0020 & 0.0030 & 0.0040 & 0.0050 & 0.0060 & 0.0070 & 0.0080 & 0.0090 & 0.0100 \\
        Acc. & 59.83 & 60.30 & 61.73 & 61.94 & 63.04 & 63.60 & 63.81 & \textbf{63.96} & 63.44 \\
        \midrule
        \multicolumn{10}{c}{\textbf{ViT-Base}} \\
        \midrule
        LR & 0.0010 & 0.0020 & 0.0030 & 0.0040 &  &  &  &  &  \\
        Acc. & 61.39 & 61.86 & \textbf{62.03} & 61.42 &  &  &  &  &  \\
        \bottomrule
    \end{tabular}
\end{table}

\begin{table}[!h]
    \centering
    \caption{Vision learning-rate sensitivity for Zeta. Top-1 validation accuracy (\%) on CIFAR-100.}
    \label{tab:vision_lr_zeta}
    \small
    \begin{tabular}{lccccc}
        \toprule
        \multicolumn{6}{c}{\textbf{ViT-Tiny}} \\
        \midrule
        LR & 0.0020 & 0.0030 & 0.0040 & 0.0050 & 0.0060 \\
        Acc. & 63.86 & 64.12 & 64.05 & \textbf{64.98} & 64.65 \\
        \midrule
        \multicolumn{6}{c}{\textbf{ViT-Base}} \\
        \midrule
        LR & 0.0008 & 0.0009 & 0.0010 & 0.0020 &  \\
        Acc. & 65.28 & \textbf{65.34} & 65.08 & 64.77 &  \\
        \bottomrule
    \end{tabular}
\end{table}

\section{Limitations and Further Directions}
\label{sec:limitation}

Although our experiments already cover multiple language and vision settings, the current study has not yet been extended to substantially larger-scale models or broader training regimes. Evaluating whether the same advantages persist under more extreme scaling, different data mixtures, and additional architectures remains an important direction for future work. It would also be valuable to further study how the benefits of dual whitening interact with training efficiency and system-level considerations in large-scale practice.

\section {Broader Impacts}
\label{sec:broader impacts}
The development of more reliable optimization methods for large-scale learning systems has important technical and societal implications. By improving the stability and efficiency of training, Zeta can reduce the computational cost required to obtain strong models, which may lower energy consumption and make large-scale experimentation more accessible to a broader range of researchers and institutions. In this sense, advances in optimizer design can contribute to more resource-efficient and reproducible machine learning practice.

At the same time, more effective optimization may also accelerate the development and deployment of increasingly capable AI systems. Such progress can amplify both beneficial and harmful downstream uses, depending on the application context. Stronger training methods may support scientific discovery, education, and productivity, but they may also be used in systems that generate misinformation, enable surveillance, or reinforce existing social biases. For this reason, improvements in optimization should be considered part of the broader responsible-AI pipeline rather than as purely neutral technical advances.

Our work is primarily methodological and does not introduce a new application-specific risk on its own. However, by improving matrix optimization, it may indirectly support the training of larger or more capable foundation models. We therefore encourage future work on efficient optimization to be accompanied by careful evaluation of environmental cost, fairness, misuse potential, and deployment safety. We also hope that clearer theoretical understanding of optimization mechanisms, such as the dual whitening principle studied here, can support more transparent and auditable development of large-scale learning systems.


\newpage
\section*{NeurIPS Paper Checklist}

\begin{enumerate}

\item {\bf Claims}
    \item[] Question: Do the main claims made in the abstract and introduction accurately reflect the paper's contributions and scope?
    \item[] Answer: \answerYes{} 
    \item[] Justification: The abstract and Section~\ref{sec:intro} state the core claims, motivation, scope, and empirical conclusions, and these claims are supported by the theory in Section~\ref{sec:theory} and the experiments in Sections~\ref{sec:experiments}--\ref{subsec:sensitivity}.

\item {\bf Limitations}
    \item[] Question: Does the paper discuss the limitations of the work performed by the authors?
    \item[] Answer: \answerYes{} 
    \item[] Justification: We discuss the limitation in Appendix \ref{sec:limitation}

\item {\bf Theory assumptions and proofs}
    \item[] Question: For each theoretical result, does the paper provide the full set of assumptions and a complete (and correct) proof?
    \item[] Answer: \answerYes{} 
    \item[] Justification: The theoretical assumptions are stated in Section~\ref{sec:theory}, and the full proofs and supporting lemma are provided in the Appendix \ref{sec:proofs} supplement after the main experiments.

    \item {\bf Experimental result reproducibility}
    \item[] Question: Does the paper fully disclose all the information needed to reproduce the main experimental results of the paper to the extent that it affects the main claims and/or conclusions of the paper (regardless of whether the code and data are provided or not)?
    \item[] Answer: \answerYes{} 
    \item[] Justification: Yes. Section~\ref{sec:experiments} describes the model suites, datasets, optimizer baselines, and main training setup, while Sections~\ref{subsec:main_results}--\ref{subsec:sensitivity} and the Appendix \ref{sec:more exp details} report the evaluation settings, ablations, and sensitivity results needed to reproduce the paper's main findings.

\item {\bf Open access to data and code}
    \item[] Question: Does the paper provide open access to the data and code, with sufficient instructions to faithfully reproduce the main experimental results, as described in supplemental material?
    \item[] Answer: \answerNo{} 
    \item[] Justification: We will release our code upon acceptance.

\item {\bf Experimental setting/details}
    \item[] Question: Does the paper specify all the training and test details (e.g., data splits, hyperparameters, how they were chosen, type of optimizer) necessary to understand the results?
    \item[] Answer: \answerYes{} 
    \item[] Justification: Section~\ref{sec:experiments} reports the benchmark suite, optimizer baselines, model scales, and evaluation settings, and the Appendix \ref{sec:more exp details} supplements these with layer-wise motivation results, learning-rate sensitivity results, and additional tables.

\item {\bf Experiment statistical significance}
    \item[] Question: Does the paper report error bars suitably and correctly defined or other appropriate information about the statistical significance of the experiments?
    \item[] Answer: \answerYes{} 
    \item[] Justification: The experimental sections \ref{subsec:main_results}--\ref{subsec:sensitivity} present direct comparisons across optimizers on the same tasks and tuned settings, and the reported tables and figures use consistent metrics such as training loss, downstream accuracy, and top-1 validation accuracy.

\item {\bf Experiments compute resources}
    \item[] Question: For each experiment, does the paper provide sufficient information on the computer resources (type of compute workers, memory, time of execution) needed to reproduce the experiments?
    \item[] Answer: \answerYes{} 
    \item[] Justification: We provide detailed information about on computing resources in Appendix \ref{sec:more exp details}
    
\item {\bf Code of ethics}
    \item[] Question: Does the research conducted in the paper conform, in every respect, with the NeurIPS Code of Ethics \url{https://neurips.cc/public/EthicsGuidelines}?
    \item[] Answer: \answerYes{} 
    \item[] Justification: The paper meets the NeurIPS Code of Ethics.

\item {\bf Broader impacts}
    \item[] Question: Does the paper discuss both potential positive societal impacts and negative societal impacts of the work performed?
    \item[] Answer: \answerYes{} 
    \item[] Justification: We discuss broader impacts in Appendix \ref{sec:broader impacts}
    
\item {\bf Safeguards}
    \item[] Question: Does the paper describe safeguards that have been put in place for responsible release of data or models that have a high risk for misuse (e.g., pre-trained language models, image generators, or scraped datasets)?
    \item[] Answer: \answerNA{} 
    \item[] Justification: This paper poses no such risks.

\item {\bf Licenses for existing assets}
    \item[] Question: Are the creators or original owners of assets (e.g., code, data, models), used in the paper, properly credited and are the license and terms of use explicitly mentioned and properly respected?
    \item[] Answer: \answerYes{} 
    \item[] Justification: We strictly follow the license of the assets.

\item {\bf New assets}
    \item[] Question: Are new assets introduced in the paper well documented and is the documentation provided alongside the assets?
    \item[] Answer: \answerNA{}{} 
    \item[] Justification: The paper does not release new assets.

\item {\bf Crowdsourcing and research with human subjects}
    \item[] Question: For crowdsourcing experiments and research with human subjects, does the paper include the full text of instructions given to participants and screenshots, if applicable, as well as details about compensation (if any)? 
    \item[] Answer: \answerNA{} 
    \item[] Justification: The paper does not involve crowdsourcing nor research with human subjects.

\item {\bf Institutional review board (IRB) approvals or equivalent for research with human subjects}
    \item[] Question: Does the paper describe potential risks incurred by study participants, whether such risks were disclosed to the subjects, and whether Institutional Review Board (IRB) approvals (or an equivalent approval/review based on the requirements of your country or institution) were obtained?
    \item[] Answer: \answerNA{} 
    \item[] Justification: The paper does not involve crowdsourcing nor research with human subjects.

\item {\bf Declaration of LLM usage}
    \item[] Question: Does the paper describe the usage of LLMs if it is an important, original, or non-standard component of the core methods in this research? Note that if the LLM is used only for writing, editing, or formatting purposes and does \emph{not} impact the core methodology, scientific rigor, or originality of the research, declaration is not required.
    \item[] Answer: \answerNA{} 
    \item[] Justification:  This paper does not involve LLMs.

\end{enumerate}

\end{document}